\documentclass[11pt]{article}

\usepackage[nodisplayskipstretch]{setspace}
\usepackage[utf8]{inputenc}
\usepackage{lipsum}

\usepackage{amssymb,amsbsy,amsfonts,amsmath,xspace,amsthm}
\usepackage{enumitem}

\usepackage{mathrsfs}
\usepackage{graphicx,algpseudocode,algorithm}

\allowdisplaybreaks

\algnewcommand\algorithmicinput{\textbf{Input:}}
\algnewcommand\algorithmicoutput{\textbf{Output:}}
\algnewcommand\INPUT{\item[\algorithmicinput]}
\algnewcommand\OUTPUT{\item[\algorithmicoutput]}

\def\@normalsize{\@setsize\normalsize{11pt}\xpt\@xpt}

\usepackage{caption}
\usepackage{setspace}
\usepackage{comment}
\usepackage[margin=1in]{geometry}
\usepackage{subcaption}
\usepackage{multirow}
\usepackage{epstopdf}
\usepackage{epsfig}
\usepackage[colorlinks,citecolor=blue]{hyperref}

\usepackage[labelfont=bf]{caption}
\usepackage{url}
\usepackage[toc,page]{appendix}
\usepackage{float}
\usepackage{natbib}
\usepackage{mathrsfs}
\usepackage{color}
\usepackage[dvipsnames]{xcolor}
\usepackage{verbatim}
\usepackage{authblk}
\usepackage[normalem]{ulem}

\usepackage{hyperref}
\hypersetup{
    colorlinks=true,
    linkcolor=blue,
    filecolor=magenta,      
    urlcolor=cyan,
    citecolor=blue
}

\newcommand*{\KeepStyleUnderBrace}[1]{
\mathop{%
\mathchoice
{\underbrace{\displaystyle#1}}%
{\underbrace{\textstyle#1}}%
{\underbrace{\scriptstyle#1}}%
{\underbrace{\scriptscriptstyle#1}}%
}\limits
}
\usepackage{mathtools}
\mathtoolsset{showonlyrefs}
\usepackage{amsmath,amssymb,amsthm,bm}
\usepackage{dsfont,listings}

 \renewcommand\footnotemark{}

\theoremstyle{plain} 
\newtheorem{thm}{Theorem}[section]
\newtheorem{lem}{Lemma}

\theoremstyle{definition}
\newtheorem{prop}{Proposition}
\newtheorem{corollary}{Corollary}
\newtheorem{assumption}{Assumption}
\newtheorem{defn}{Definition}
\newtheorem{example}{Example}
\newtheorem{rmk}{Remark}

\def\caliF{\tF_{\textup{sgn}}}
\def\caliM{\tM_{\textup{sgn}}}

\usepackage{xcolor}
\allowdisplaybreaks
\usepackage{stmaryrd} 

\def\mTheta{\bm{\Theta}}

\def\ma{\bm{a}}
\def\mb{\bm{b}}

\def\me{\bm{e}}

\def\bmu{\bm{u}}
\def\mv{\bm{v}}

\def\mA{\bm{A}}
\def\mB{\bm{B}}
\def\mC{\bm{C}}

\def\mE{\bm{E}}

\def\mI{\bm{I}}

\def\mI{\bm{I}}
\def\mM{\bm{M}}

\def\mS{\bm{S}}

\def\mU{\bm{U}}
\def\mV{\bm{V}}
\def\mW{\bm{W}}
\def\mX{\bm{X}}
\def\mY{\bm{Y}}
\def\mZ{\bm{Z}}
\def\mLambda{\bm{\Lambda}}

\def\mZ{\bm Z}
\def\mA{\bm A}
\def\mB{\bm B}
\def\mS{\bm S}
\def\mC{\bm C}
\def\mI{\bm I}

\def\mX{\bm X}
\def\mY{\bm Y}
\def\mM{\bm M}

\def\tB{\mathcal{B}}

\def\tF{\mathcal{F}}

\def\tH{\mathcal{H}}
\def\tI{\mathcal{I}}

\def\tM{\mathcal{M}}
\def\tN{\mathcal{N}}
\def\tO{\mathcal{O}}

\def\tX{\mathcal{X}}
\def\tY{\mathcal{Y}}

\def\entry#1{\llbracket #1 \rrbracket}
\def\entry#1{\llbracket #1 \rrbracket}

\newcommand{\norm}[1]{\left\lVert#1\right\rVert}

\newcommand{\Fnorm}[1]{\left\lVert#1\right\rVert_F}

\newcommand{\newnormSize}[2]{#1\lVert#2#1\rVert}

\newcommand{\onenormSize}[2]{#1\lVert#2#1\rVert_1}

\newcommand{\vnormSize}[2]{#1\lVert#2#1\rVert_2}

\newcommand{\FnormSize}[2]{#1\lVert#2#1\rVert_F} 
\newcommand{\mnormSize}[2]{#1\lVert#2#1\rVert_\infty}

\DeclareMathOperator*{\argmin}{arg\,min}

\def\@normalsize{\@setsize\normalsize{11pt}\xpt\@xpt}
\def\spacingset#1{\renewcommand{\baselinestretch}%
{#1}\small\normalsize} \spacingset{1}

\usepackage[parfill]{parskip}
\usepackage{bm}

\def\sign{\textup{sgn}}
\def\bayesS{S_{\textup{bayes}}}
\def\bayespif{f_{\textup{bayes},\pi}}

\def\srank{\textup{srank}}
\def\rank{\textup{rank}}
\def\supp{\textup{supp}}

\def\risk{\textup{Risk}_\pi}
\def\erisk{\widehat{\textup{Risk}}_\pi}

 \def\caliF{\tF_{\textup{sgn}}}
\def\caliM{\tM_{\textup{sgn}}}

\def\shift{\bar Y_\pi}

\def\riskF{\textup{Risk}_{\pi,F}}
\def\eriskF{\widehat{\textup{Risk}}_{\pi,F}}

\def\CNN{\text{\bf \small CNN}}
\def\Lasso{\text{\bf \small LogisticV}}
\def\NonparaM{\text{\bf \footnotesize ASSIST}}
\def\LogisticM{\text{\bf \small LogisticM}}  
\def\HardImpute{\text{\bf \small HardImpute}}  
\def\SoftImpute{\text{\bf \small SoftImpute}}  
\def\ALT{\text{\bf \small ALT}}

\usepackage{setspace}
\usepackage{xr}

\usepackage{authblk}
\title{Nonparametric Trace Regression in High Dimensions via \\Sign Series Representation}
\author[1]{Chanwoo Lee}
\author[2]{Lexin Li}
\author[3]{Hao Helen Zhang}
\author[1]{Miaoyan Wang$^*$\footnote{$^*$corresponding author: miaoyan.wang@wisc.edu. }}
\affil[1]{Department of Statistics, University of Wisconsin--Madison}
\affil[2]{Division of Biostatistics, University of California--Berkley}
\affil[3]{Department of Mathematics, University of Arizona}
\date{}

\begin{document}
\maketitle

\begin{abstract}
Learning of matrix-valued data has recently surged in a range of scientific and business applications. Trace regression is a widely used method to model effects of matrix predictors and has shown great success in matrix learning. However, nearly all existing trace regression solutions rely on two assumptions: (i) a known functional form of the conditional mean, and (ii) a global low-rank structure in the entire range of the regression function, both of which may be violated in practice. In this article, we relax these assumptions by developing a general framework for nonparametric trace regression models via structured sign series representations of high dimensional functions. The new model embraces both linear and nonlinear trace effects, and enjoys rank invariance to order-preserving transformations of the response. In the context of matrix completion, our framework leads to a substantially richer model based on what we coin as the ``sign rank'' of a matrix. We show that the sign series can be statistically characterized by weighted classification tasks. Based on this connection, we propose a learning reduction approach to learn the regression model via a series of classifiers, and develop a parallelable computation algorithm to implement sign series aggregations. We establish the excess risk bounds, estimation error rates, and sample complexities. Our proposal provides a broad nonparametric paradigm to many important matrix learning problems, including matrix regression, matrix completion, multi-task learning, and compressed sensing. We demonstrate the advantages of our method through simulations and two applications, one on brain connectivity study and the other on high-rank image completion. 
\end{abstract}

\section{Introduction}
\label{sec:intro}

Matrix-valued data are rising ubiquitously in modern data science applications, for instance, brain neuroimaging analysis, integrative genomics, and sensor network localization. Trace regression is one of the most commonly used approaches for modeling matrix data \citep{fan2019generalized,hamidi2019low}. The model characterizes the relationship between a scalar response $Y$ and a high dimensional matrix predictor $\mX \in \tX \subset \mathbb{R}^{d_1\times d_2}$ as 
\begin{equation}\label{eq:linear}
Y=\langle \mX, \mB \rangle+ \varepsilon,\ \text{with } \mB \in \mathbb{R}^{d_1\times d_2} \text{ and rank}(\mB)\leq r,
\end{equation}
where $\varepsilon$ is a zero-mean sub-Gaussian noise, and $r\in\mathbb{N}_{+}$ is the matrix rank typically assumed fixed and much smaller than $\min(d_1,d_2)$. The function $\mX\mapsto \langle \mX,\mB\rangle=\text{tr}(\mX\mB^T)$ is called the trace effect, where $\text{tr}(\cdot)$ denotes the matrix trace. Over the last decade, the low-rank trace regression \eqref{eq:linear} has been studied intensively in numerous contexts, including matrix predictor regression, matrix completion, multi-task learning, and compressed sensing.
\begin{itemize}
\item{\bf Matrix predictor regression.} Linear trace regression~\eqref{eq:linear} was first proposed to model a matrix-valued predictor \citep{zhou2014regularized, wang2014network}, and was later generalized to model an exponential family response with a known link function \citep{wang2017generalized, fan2019generalized}. 
\smallskip

\item {\bf Matrix completion.} In addition to the usual regression setting, another application of trace regression \eqref{eq:linear} is matrix completion, where the goal is to fill in the missing entries of a partially observed matrix \citep{Cai2016}. Suppose the predictor space $\tX$ consists of basis matrices $\ma_i\mb^T_j$ in $\mathbb{R}^{d_1\times d_2}$, with $\ma_i\in\mathbb{R}^{d_1}$ (respectively, $\mb_j\in\mathbb{R}^{d_2}$) being the basis vector with 1 at the $i$-th (respectively, $j$th) position and 0 elsewhere. Let $\mathbb{P}_{\mX}$ be a uniform distribution over $\tX$. Then model \eqref{eq:linear} reduces to a matrix completion problem, $Y_{ij}=\langle \ma_i \mb^T_j,\mB \rangle +\varepsilon_{ij}= B_{ij}+\varepsilon_{ij}$, where $Y_{ij}, B_{ij}\in\mathbb{R}$ denotes the $(i,j)$-th entry of the data matrix $\mY$ and the signal matrix $\mB$, respectively, for $(i,j) \in \Omega\subset \{1,\ldots,d_1\}\times\{1,\ldots,d_2\}$ in the observed index set. Moreover, the model becomes a matrix denosing problem~\citep{Ma2016} when the observation set is complete, i.e, $\Omega=\{1,\ldots,d_1\}\times\{1,\ldots,d_2\}$. 
\smallskip

\item {\bf Multi-task learning.} Another application of trace regression is multi-task learning, where the goal is to predict one task response by leveraging the structural similarities among multiple tasks. Here the predictor space $\tX$ consists of only matrices that have a single non-zero row. The multi-task problem collects $n$ observations from $d_1$ different supervised learning tasks. Each task is modeled as a linear regression with an unknown $d_2$-dimensional parameter $\mb_i, i=1,\ldots,d_1$, and the collection of $\mb_i$ forms the rows of $\mB$. The model exploits similarities among multiple tasks to predict the response of the $i$-th task \citep{caruana1997multitask,fan2019generalized}. 
\smallskip

\item {\bf Compressed sensing.} Compressed sensing is also a special application of trace regression, where the goal is to recover the structured matrix $\mB$ from multiple linear combinations of the entry observations. The space $\tX$ is the family of measurement matrices given the sampling schemes. For example, Gaussian ensembles use random matrices $\mX$ with i.i.d.\ entries from a standard normal distribution \citep{candes2011tight}, while factorized ensembles use rank-1 matrices $\mX=\bmu\mv^T$ for two random vectors $\bmu\in\mathbb{R}^{d_1}, \mv\in\mathbb{R}^{d_2}$ \citep{recht2010guaranteed}.
\end{itemize}

\noindent
In this article, we propose and study a nonparametric extension of the trace regression model \eqref{eq:linear}, which encompasses all above matrix learning problems. Particularly, we illustrate our method with two common problems, i.e., matrix predictor regression and matrix completion.

\subsection{Inadequacy of low-rank trace regression}
\label{sec:limit}

The existing trace regression model \eqref{eq:linear} and its variants rely on two key assumptions: the relationship between $\mathbb{E}(Y|\mX)$ and the trace effect is known a priori through some link function, and the matrix effect is encoded by a global low-rank matrix $\mB$ in the entire function range. However, despite the popularity of trace regressions, these assumptions are stringent and may often be violated in practice. Next, we use two examples to illustrate the limitations of the classical low-rank trace regression. We present the pitfall in the context of matrix completion, and similar phenomena also occur in general matrix predictor regression. 

In the first example, we show the sensitivity of low-rank matrix models to order-preserving transformations. Let $\mB=\mU^T\mV \in \mathbb{R}^{d\times d}$ be a rank-5 matrix, where $\mU, \mV \in\mathbb{R}^{d\times 5}$ consists of i.i.d. standard normal entries and $d=50$. Now suppose a monotonic transformation $g(b)=(1+\exp(-cb))^{-1}$ is applied to $\mB$ entry-wise, and we let $g(\mB)$ be the signal matrix prior to measurements. A small $c$ implies an approximate linearity $b\mapsto -cb$, whereas a large $c$ implies a high nonlinearity $b\mapsto \{0,1\}$. Fig~\ref{fig:limit}(a) shows that the numerical rank of $g(\mB)$ increases rapidly with $c$, rendering the classical low-rank model ineffective. In genomic signal processing and other applications, the matrix of interest often undergoes unknown transformation prior to measurements. The sensitivity makes low-rank models less desirable as the global low-rank structure fails to be preserved through monotonic transformations.

In the second example, we show the failure of the classical low-rank model in representing a structured but high-rank effect. We again consider the matrix completion for simplicity, but this time, from a full-rank signal matrix $\mB\in\mathbb{R}^{d\times d}$, where the $(i,j)$-th entry is $\log(1+\max(i,j)/d)$ and $d=10$. Fig~\ref{fig:limit}(b) shows that $\mB$ is clearly structured, but is of full-rank that $\text{rank}(\mB)=d$. The classical low-rank model is again ineffective in this case. 

\begin{figure}
\begin{center}
\includegraphics[width=.8\textwidth]{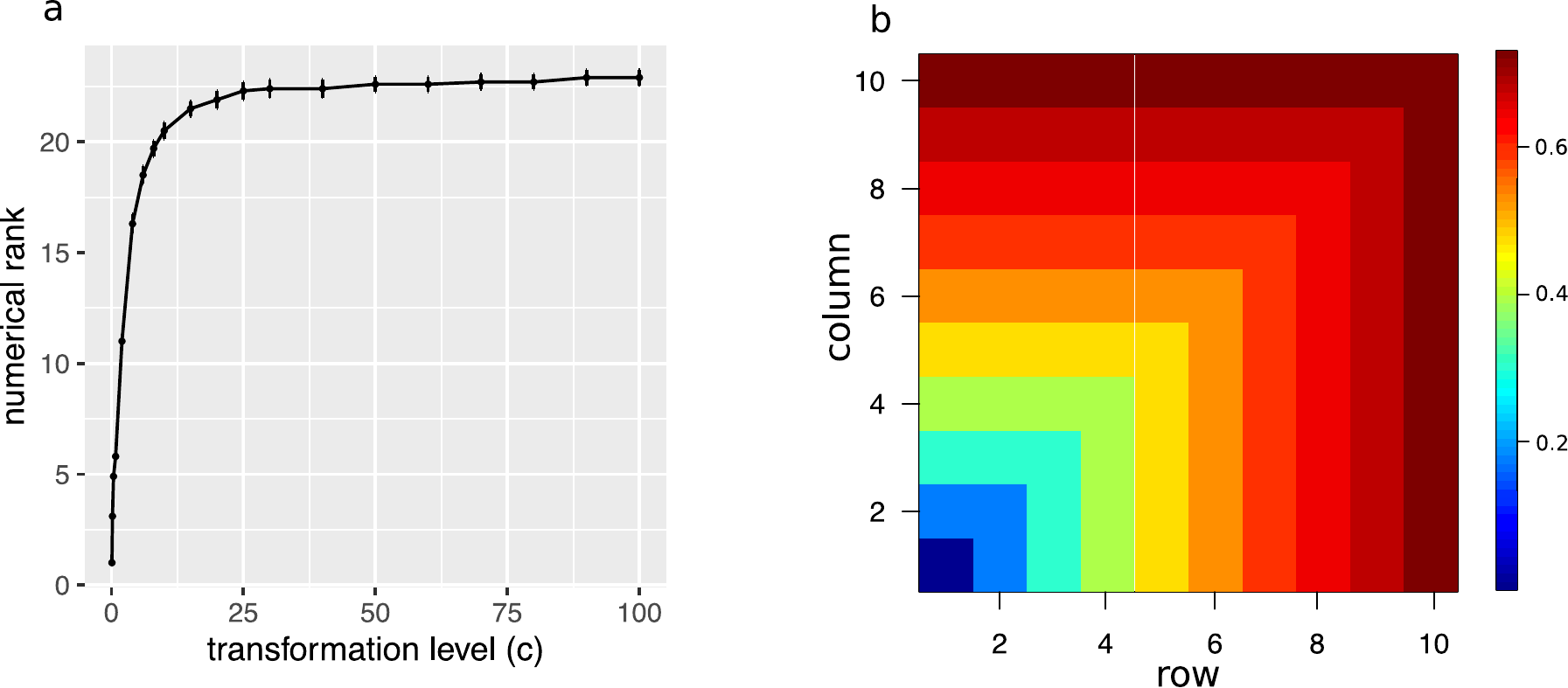}
\caption{Two examples of high-rank matrix trace models. (a) The numerical rank of the matrix $g(\mB)$ versus $c$ in the transformation, where the numerical rank is defined by $\rank(g(\mB))=\min\{\rank(\mC )\colon \FnormSize{}{\mC-g(\mB)} \leq 0.01 \FnormSize{}{g(\mB)} \}$. The error bar represents standard errors from 10 realizations of $\mB$. (b) Heatmap of a full-rank matrix $\mB\in\mathbb{R}^{d\times d}$ with the $(i,j)$-th entry equal to $\log(1+\max(i,j))$. In (a), $d=50$, and in (b), $d=10$.} 
\label{fig:limit}
\label{penG}
\end{center}
\end{figure}

These examples reveal the inadequacy of the conventional low-rank trace model~\eqref{eq:linear} in capturing important yet complex matrix effects. This has motivated us to develop a flexible class of nonparametric trace regression for modeling and estimating nonlinear, local, and possibly high-rank effects for high dimensional matrices. We later revisit these two examples in Section~\ref{sec:idea}, and show how those limitations can be overcome using a richer class of matrix models based on a new concept what we coin as the matrix ``sign rank''.

\subsection{Our proposal and contributions}

In this article, we first propose a new notion of low-rank sign representable function, then develop a flexible class of nonparametric trace regression models based on this representation, as well as relevant theory and computational algorithms. Our proposal makes useful contributions on multiple fronts. 

First, the proposed work fills a crucial gap between a global parametric model and a local nonparametric model in the literature of matrix modeling. We develop a new nonparametric regression paradigm -- structured sign representations -- to address the challenges previously difficult or infeasible in trace regressions, especially in the high dimensional regime where $d_1d_2\gg n$. Existing literature on matrix regressions almost exclusively focuses on low-rank trace effects in the global scale. However, such a premise often fails, where the rank of global effects may grow with the matrix dimension. By contrast, our proposed model enjoys rank invariance under monotonic transformations, and permits both low-rank and high-rank effects through aggregations of sign representation functions. We show that the low-rank sign functions not only preserve all information for conventional low-rank models, but also provide powerful tools for extracting nonlinear, high-rank trace effects and estimating them accurately. Our framework is flexible and applicable to high-rank matrix learning problems, and it greatly expands the horizon of conventional low-rank matrix models.

Second, we show that the sign function series can be statistically characterized by classification tasks with carefully specified weights. This characterization converts a complex and hard regression problem, \emph{``what is the value of the nonparametric regression function?''} to a series of simpler and easier classification problems, \emph{``does the regression function fall below a threshold?''} Correspondingly, we develop a learning reduction approach to estimate the regression function via a series of classifiers, by leveraging classification solutions from existing state-of-art computational algorithms. Theoretically, we establish the excess risk bounds, estimation error rates, and sample complexities. Particularly, our error bound reveals the well-controlled complexity from sign estimation to regression, where 
\vspace{-0.01in}
\begin{align*}
\text{sign function error }& \lesssim \KeepStyleUnderBrace{t_n^{\alpha/( 2+\alpha)}}_{\text{classification error}},\\
\text{regression error } & \lesssim  \KeepStyleUnderBrace{t_n^{\alpha/(2+\alpha)}\log H}_{\text{estimation error inherited from classification}}+\KeepStyleUnderBrace{\textstyle{1\over H}}_{\text{reduction bias}}+\KeepStyleUnderBrace{t_nH\log H}_{\text{reduction variance}},
\end{align*}
in which $\alpha\geq 0$ quantifies the smoothness of the nonparametric regression function, $H\in\mathbb{N}_{+}$ is a resolution parameter that specifies the total number ($2H+1$) of sign functions to aggregate in our algorithm, $t_n=t_n(d,n)\to 0$ quantifies the convergence rate depending on the specific model, and $d=d_1=d_2$ for simplicity. In particular, we establish $t_n\asymp n^{-1}\log d$ under a two-way sparse non-parametric trace regression model (see Section~\ref{sec:sparse}), and $t_n \asymp n^{-1}d$ under a low sign rank non-parametric matrix completion model (see Section~\ref{sec:matrixcompletion}). These results imply that a low sample complexity with respect to the matrix dimension. Note that the sign function estimation reaches a faster $\tO(n^{-1})$ rate compared to the $\tO(n^{-1/2})$ regression rate when $\alpha= \infty$, which confirms our premise that sign estimation is easier than regression. To our knowledge, these statistical guarantees are among the first for the learning reduction approach in the context of nonparametric matrix regression. 

Lastly, we develop an alternating direction method of multipliers (ADMM) algorithm for optimization with a family of large-margin loss functions. From the computational and learning perspectives, the proposed method can be characterized as the {\bf \small  A}ggregation of {\bf \small  S}tructured {\bf \small  SI}gn {\bf \small  S}eries for {\bf \small T}race regression ({\bf \footnotesize ASSIST}). We show that the \NonparaM\ algorithm leverages recent advances in large-margin solvers as well as non-convex optimization for low-rank, two-way sparse matrix learning. As demonstrated in our simulations and real data applications, the \NonparaM\ method contributes a new matrix modeling tool of easy interpretability and accurate prediction.

\subsection{Related work}

Nonparametric learning for matrix data is much more challenging than standard multivariate data. Naively turning a matrix into a vector followed by a classical vector based nonparametric method can destroy rich structural information encoded in the matrix data. Moreover, most nonparametric methods rely on some notion of smoothness in a local neighborhood of the predictors. In the context of matrix regressions, however, the predictor space is huge, rendering the ``local smoothness'' assumption less practical, which is partially why the topic is barely explored by data with a limited sample size.  

Our work is related to but also clearly distinctive from several lines of existing research. The first line is the classical trace regression \citep{fan2019generalized,hamidi2019low}. The key difference is that the existing solutions all adopt a parametric model with a global low-rank structure. By contrast, our method is nonparametric and embraces nonlinear, local, and possibly high-rank effects for high dimensional matrices. 
 
The second line is the recent development of nonparametric methods with matrix-valued or tensor-valued data. In imaging analysis,  convolution neural networks (CNNs) have been widely adopted as a nonparametric tool for prediction given matrix-valued images \citep{goodfellow2016deep}. In contrast, our proposal studies not only prediction, but also estimation and interpretability, with the theoretical guarantees. We also numerically compare our method with CNNs. \cite{hao2019sparse} proposed a sparse additive model with tensor predictors by extending the usual spline basis functions. \cite{zhou2020broadcasted} studied tensor predictors and proposed a broadcasting operation to introduce nonlinearity to individual tensor entries. Our nonparametric solution has broader implications than those approaches in estimating local low-rank effects. Our sign series representation of function bridges the gap between regression and classification in high dimensions, and naturally lends the problem to a learning reduction type solution. Moreover, although a matrix can be viewed as a two-dimensional tensor, the problem of nonparametric learning for matrix data itself is more parsimonious and deserves a full investigation. We leave the counterpart problem for nonparametric tensor regression as future research. 

The third line is function sign estimation, which is in turn related to classification, or more generally, the level set estimation. The latter problem has a long history in statistics \citep{tsybakov1997nonparametric} and computational mathematics \citep{gibou2018review}. Particularly, \cite{wang2008probability} proposed a conditional probability estimation method based on support vector machines (SVMs), but their results were restricted to a fixed number of features and vector predictors only. \cite{singh2009adaptive} proposed a tree based method for multiple sets extraction, but their goal was level set estimation instead of function estimation. None of these methods address the regression problem or high dimensional matrix predictors. By contrast, we bridge the problems of level set estimation and nonparametric regression using low-rank sign series representations. Instead of constructing a point-wise function in the domain space, the sign representation partitions the domain space based on the function range. The benefit bears the analogy of Lebesgue versus Riemann integrals in functional analysis, in the sense that the neighborhood is determined by the range space instead of the domain space. The former approach is especially appealing for matrix regressions, where the range space is determined by a simple scalar response, whereas the domain space is huge and high dimensional.

\subsection{Notation and organization}

We adopt the following notation throughout this article. Let $\tX \subset \mathbb{R}^{d_1\times d_2}$ denote the feature space equipped by some measure $\mathbb{P}_{\mX}$. For a function $f\colon \tX \to \mathbb{R}$, let $\sign f$ denote its sign function, i.e., $\sign f(\mX)=1$ if $f(\mX)>0$ and $\sign f(\mX)=-1$ otherwise. Let $\onenormSize{}{f}$ denote its $L_1$ norm, where we define $\onenormSize{}{f}=\mathbb{E}|f(\mX)|$ with the expectation taken with respect to $\mX \sim \mathbb{P}_{\mX}$. For a set $A \subset \tX$, let $\sign (\mX\in A)$ denote the sign function induced by $A$, i.e., a function taking value $1$ on the event $\{\mX\in A\}$ and $-1$ otherwise. Let $[n] = \{1,\ldots,n\}$, and $|\cdot|$ denote the cardinality. Let $\newnormSize{}{\cdot}_p$ denote the vector $p$-norm for $p\geq 0$. For a matrix $\mB \in \mathbb{R}^{d_1\times d_2}$, let $\mB_i$ denote its $i$-th row and $B_{ij}$ its $(i,j)$-th entry. Let $\newnormSize{}{\mB}_{p,q}$ denote the matrix $(p,q)$-norm such that $\newnormSize{}{\mB}_{p,q}=\newnormSize{}{\mb}_q$, where $\mb=(\newnormSize{}{\mB_1}_p,\ldots,\newnormSize{}{\mB_{d_1}}_p)^T\in\mathbb{R}^{d_1}$ consists of the $p$-norms for each row of $\mB$. In particular, let $\newnormSize{}{\mB}_{1,0}=|\{i\in [d_1]\colon \mB_i\neq 0\}|$ denote the number of non-zero rows in $\mB$. Let $\FnormSize{}{\mB}=\sqrt{\langle \mB, \mB \rangle}$ denote the matrix Frobenius norm, and $\mnormSize{}{\mB}=\max_{(i,j)}|B_{ij}|$ the matrix maximum norm. Denote $a_n\asymp b_n$ if $c_1\leq \lim_{n\to \infty} a_n/b_n\leq c_2$ for some constants $c_1,c_2>0$, and denote $a_n\lesssim b_n$ if $\lim_{n\to\infty} a_n/b_n\leq c$ for some constant $c\geq 0$. Let $\tO(\cdot)$ denote the big-O notation, $\tilde \tO(\cdot)$ the variant that hides the logarithmic factors, and $\mathds{1}(\cdot)$ the indicator function. Whenever applicable, the basic arithmetic operators are applied to a matrix in an element-wise manner. 

The rest of the article is organized as follows. Section \ref{sec:idea} presents the low-rank sign representable functions and our nonparametric trace regression model. Section \ref{sec:bridge} develops the learning reduction approach through weighted classifications, and establishes the corresponding statistical guarantees. Section \ref{sec:examples} specializes the general theory to two concrete learning problems, the low-rank sparse matrix predictor regression and the high-rank matrix completion. Section \ref{sec:estimation} studies the large-margin based estimation and develops an optimization algorithm. Section \ref{sec:simulation} presents the simulations, and Section \ref{sec:realdata} two real data applications. Section \ref{sec:discussion} concludes with a discussion. All technical proofs and additional results are relegated to the Supplementary Appendix.

\section{Nonparametric trace regression model}
\label{sec:idea}

In this section, we present our nonparametric trace regression model. Let $\mX\in\tX\subset \mathbb{R}^{d_1\times d_2}$ denote the matrix predictor, $Y\in\mathbb{R}$ the scalar response, and $\mathbb{P}_{\mX,Y}$ the joint probability distribution. We consider the model,
\begin{equation}\label{eq:model}
Y=f(\mX)+\varepsilon,
\end{equation}
where $f\colon\tX\mapsto \mathbb{R}$ is an unknown regression function of interest, and $\varepsilon$ is a mean-zero noise. For a cleaner exposition, we assume the noise is bounded and the range of $Y$ is in $[-1,1]$; the extension to a sub-Gaussian noise is provided in Section~\ref{sec:sub-Gaussian} of the Appendix. In addition, we allow a heterogeneous noise such that $\varepsilon$ may depend on $\mX$. Model \eqref{eq:model} therefore incorporates both continuous and binary-valued responses. For instance, we allow the binary regression problem where $Y$ is a $\{0,1\}$-label from a Bernoulli distribution, in which case, the noise variance depends on the mean, and $f$ represents the conditional probability, $f(\mX)=\mathbb{P}(Y=1|\mX)$. Our goal is to estimate the regression function $f(\mX)=\mathbb{E}(Y|\mX)$ based on $n$ i.i.d.\ training samples $(\mX_i,Y_i)_{i=1,\ldots,n}$. 

We next introduce the notion of low-rank sign representable function, which is essential to bridge the usual global low-rank trace models to nonparametric local low-rank trace models. 

\begin{defn}[Rank-$r$ sign representable function] \label{def:caliF}
A function $f\colon \tX\mapsto[-1,1]$ is called $(r,\pi)$-sign representable, for a given level $\pi\in[-1,1]$ and a rank $r \in \mathbb{N}_{+}$, if the function $(f-\pi)$ has the same sign as a rank-$r$ trace function; that is,
\begin{equation} \label{eq:sign}
\sign(f(\mX)-\pi) = \sign(\langle \mX, \mB \rangle+b),\quad \text{for all }\mX\in\tX,
\end{equation}
where $\mB=\mB(\pi)$ is a rank-$r$ matrix, and $b=b(\pi)$ is the intercept. A function $f$ is called globally rank-$r$ sign representable, if $f$ is $(r,\pi)$-sign representable for all $\pi\in[-1,1]$. Let $\caliF(r)$ denote the rank-$r$ sign representable function family, and let $\Phi(r)=\{\phi\colon \mX\mapsto \langle \mX, \mB \rangle+b\ \big|\ \text{rank}(\mB)\leq r, (\mB,b)\in\mathbb{R}^{d_1\times d_2}\times\mathbb{R}\}$ denote the rank-$r$ trace function family.
\end{defn}

Next, we show that \eqref{eq:model} and \eqref{eq:sign} together form a very general family of models that incorporate most existing matrix regression models, including the low-rank trace regression, single index models, and high-rank matrix completion model. 

\begin{example}[Generalized trace regression] The linear and generalized trace regression \citep{zhou2014regularized, wang2017generalized, fan2019generalized} imposes that $f(\mX)=g(\langle \mX, \mB \rangle)$ with a known link function $g$ and a rank-$r$ coefficient matrix $\mB$. By definition, $\sign(f(\mX)-\pi)=\sign(\langle \mX, \mB \rangle -g^{-1}(\pi))$ holds for every $\pi$ in the function range. Therefore, our model includes the generalized trace regression, i.e, $f \in \caliF(r)$. In particular, the usual trace model corresponds to the identity link $g$. More generally, any monotonic $g$ is allowed as the link function, e.g., the logistic function $g(z)=(1+\exp(-z))^{-1}$, the arctangent function $g(z)={1/\pi}\arctan(z)+{1/2}$, the rectified linear unit (ReLU) function $g(z)=\max(0,z)$, and any inverse cumulative distribution function. 
\end{example}

\begin{example}[Single index regression model] 
The monotonic matrix predictor single index model \citep{balabdaoui2019least,ganti2017learning} assumes a similar form of regression function $f(\mX)=g^{}(\langle \mX, \mB\rangle)$ with a low-rank $\mB$ and a monotonic $g$, but the form of $g$ is unknown. By definition, our model family $\caliF(r)$ incorporates the single index model and does not require to know $g$ a priori. 
\end{example}

\begin{example}[Multivariate normal mixture]
The prospective model from matrix linear discriminant analysis \citep{hu2020matrix} considers a binary response $Y=\{0,1\}$, and assumes the matrix $\mX|Y$ follows a Gaussian mixture distribution, $\mX|\{Y=i\} = \mB_0 + \mB\times i + \mE_i$, $i=0,1$, where $\mB_0$ is an arbitrary baseline matrix, $\mB$ is a rank-$r$ matrix, and $(\mE_i)_{i=0,1}$ are two mutually independent noise matrices with i.i.d.\ standard normal entries. Our model incorporates this model, by noting that $f(\mX)=\mathbb{E}(Y|\mX)=\text{logistic}(\langle \mB, \mX \rangle+b)$ for some $b\in\mathbb{R}$, and thus $f\in \caliF(r)$. 
\end{example}

Definition \ref{def:caliF} leads to another notion, the matrix sign rank, which is important for applying our proposed model for matrix completion as a special nonparametric trace regression. Specifically, for a given matrix $\mTheta\in\mathbb{R}^{d_1\times d_2}$, define its sign rank as: 
\begin{equation*}
\srank(\mTheta)=\min\big\{ \rank(\mTheta')\colon \sign(\mTheta')=\sign(\mTheta),\  \mTheta'\in\mathbb{R}^{d_1\times d_2} \big\}.
\end{equation*}
This concept is important in areas such as combinatorics \citep{cohn2013fast} and quantum mechanics \citep{de2003nondeterministic}, and, to our knowledge, we are the first to exploit this notion for nonparametric learning. To better understand its relation to the proposed nonparametric trace regression, we consider model~\eqref{eq:sign} with the predictor space $\tX = \{\ma_i\mb_j^T\colon (i,j)\in[d_1]\times[d_2]\}$, and $\ma_i\in\mathbb{R}^{d_1}, \mb_j\in\mathbb{R}^{d_2}$ are the basis vectors. For matrix completion, a function $f$ over $\tX$ is equivalently represented by a $d_1$-by-$d_2$ signal matrix $\mTheta=\entry{f(\me_i\me_j^T)}$. Our proposed function family $\caliF(r)$ essentially defines a new family of structured matrices with a low sign rank, as shown in the next proposition. 

\begin{prop}[Sign-representable function over basis matrices]\label{prop:signbasis} Consider the predictor space $\tX=\{\ma_i\mb_j^T \colon (i,j)\in[d_1]\times[d_2]\}$. We represent a bounded function $f\colon \tX\to [-1,1]$ by its function values organized as a matrix $\mTheta=\entry{f(\ma_i\mb_j^T)} \in [-1,1]^{d_1\times d_2}$, for basis vectors $\ma_i\in\mathbb{R}^{d_1}, \mb_j\in\mathbb{R}^{d_2}$. If $f$ is rank-$r$ sign representable, then $\max_{\pi\in[-1,1]}\srank(\mTheta-\pi)\leq r+1$ (the constant 1 is due to the intercept in~\eqref{eq:sign}). Conversely, if $\max_{\pi\in[-1,1]}\srank(\mTheta-\pi)\leq r$, then $\mTheta$ defines a rank-$r$ sign representable function $f$. 
\end{prop}

Define the sign-$r$ representable family for the signal matrix in matrix completion.  
\begin{align*}
\caliM(r)=\{\mTheta\colon \max_{\pi\in[-1,1]}\srank(\mTheta-\pi)\leq r, \ \mnormSize{}{\mTheta}\leq 1\}.
\end{align*}
The family $\caliM(r)$ is a special case of the function family $\caliF(r)$ in Definition \ref{def:caliF} with $b=0$ and the predictor space $\tX=\{\ma_i\mb_j^T\colon (i,j)\in[d_1]\times[d_2]\}$. We next further compare the sign rank with the matrix rank in this setting. 

\begin{prop}[Sign-rank vs. matrix rank]\label{prop:signrank} Consider the setting in Proposition \ref{prop:signbasis}. Then,
\begin{enumerate}[label=(\alph*)]
\item $\max_{\pi\in[-1,1]}\srank(\mTheta-\pi)\leq \rank(\mTheta)+1$.
\item If $\mTheta \in \caliM(r)$, then $g(\mTheta)/\mnormSize{}{g(\mTheta)}\in\caliM(r+1)$ for any strictly monotonic function $g\colon \mathbb{R}\to\mathbb{R}$. Here $g(\mTheta)$ denotes the matrix by applying $g(\cdot)$ to $\mTheta$ entry-wise. 
\item For every dimension $d$, there exists a $d$-by-$d$ matrix $\mTheta\in\caliM(2)$ such that $\rank(\mTheta)= d$.  
\end{enumerate}
\end{prop}

\noindent
Proposition \ref{prop:signrank} highlights the advantages of using the sign rank in the high dimensional matrix analysis. The first property implies that classical low-rank matrix model is a special case of our low sign rank model. The second property shows that, compared to the matrix rank, the sign rank remains nearly invariant under monotonic transformations, since $\srank(g(\mTheta)) \leq 1+\srank(\mTheta)$ for all monotonic functions $g$. The last property shows that the sign rank can be dramatically smaller than the conventional matrix rank. Therefore, our model $\caliM(r)$ is strictly richer than the usual low-rank model. 

A key advantage about the sign rank concept is that the low sign rank assumption is more relaxed and hence more realistic than the classical low matrix rank assumption. We next revisit the high-rank matrix model in Fig~\ref{fig:limit}(a) to show that $\mB$ is of a high matrix rank but a low sign rank. Meanwhile, we provide some additional examples of low sign rank matrices in Section~\ref{sec:signrank} of the Appendix, including matrices with repeating patterns \citep{chan2014consistent}, banded matrices, and the identity matrix.

\begin{example}[Single index model based matrix completion]
For the model in Fig~\ref{fig:limit}(a), $g(\mB)$ is a low sign rank matrix because $\srank(g(\mB)-\pi)\leq 1+\rank(\mB)=6$ for all $\pi$ in the function range. However, $g(\mB)$ itself is often high-rank as shown in Fig~\ref{fig:limit}(a).
\end{example}

\begin{example}[High-rank matrix completion model]\label{ex:high-rank}
For the model in Fig~\ref{fig:limit}(b), the matrix $\mB=\entry{\log(1+\max(i,j)/d)}$ is full-rank. Remarkably, this high-rank matrix belongs to our sign representable function with rank 2, i.e., $\mB\in \caliM(2)$. This is because $\srank(\mB-\pi)=\srank(\bar \mB)$, where $\bar \mB=\entry{\sign(\max(i,j)-e^\pi+1)}$ is a block matrix with rank at most 2. More generally, matrices of the type $\mB=\entry{g(\max(i,j)/d)}$ belong to $\caliM(2r)$, where $g(\cdot)$ is a polynomial of degree $r$. See Section~\ref{sec:signrank} of the Appendix.
\end{example}

Our proposed nonparametric matrix regression model $\caliF(r)$ therefore implies a new matrix completion model in $\caliM(r)$. In next sections, we establish the general theory for $\caliF(r)$ first, then specialize the results to the high-rank completion problems in Section~\ref{sec:matrixcompletion}.

\section{From classification to regression: a learning reduction approach}
\label{sec:bridge}

In this section, we present a learning reduction approach to estimate $f$ from the model as specified in \eqref{eq:model} and \eqref{eq:sign}. Our main crux is to provably convert
the regression estimation problem into a series of sign function estimation problems, which are in turn solved by weighted classifications. 

More specifically, we dichotomize the response $Y_i$ into a series of binary observations, $\sign(Y_i-\pi)$, for $\pi\in\tH=\{-1,\ldots,-{1/H}, 0, {1/H}, \ldots,1\}$, where $H\in\mathbb{N}_{+}$ is a resolution parameter that controls the total number of sign functions to estimate. Then, for each $\pi$, we estimate the sign function $\sign(f-\pi)$ by performing a classification task, 
\begin{equation}\label{eq:proposal}
\hat \phi_\pi =\argmin_{\phi\in\Phi(r)}{1\over 2n}\sum_{i=1}^n\text{weighted-classification}(\sign(Y_i-\pi),\ \sign \phi(\mX_i)),
\end{equation}
where $\Phi(r)$ is the collection of rank-$r$ trace functions, and the weighted classification$(\cdot,\cdot)$ denotes a classification objective function with a response-specific weight to each sample point. The weight in the objective function is crucial in our method, and we will detail the form in next section. Our final regression function estimate takes the form, 
\begin{equation}\label{eq:stepfunction}
\hat f= {1\over 2H+1}\sum_{\pi \in \tH} \sign \hat \phi_\pi.
\end{equation}

\begin{figure}[t]
\includegraphics[width=\textwidth]{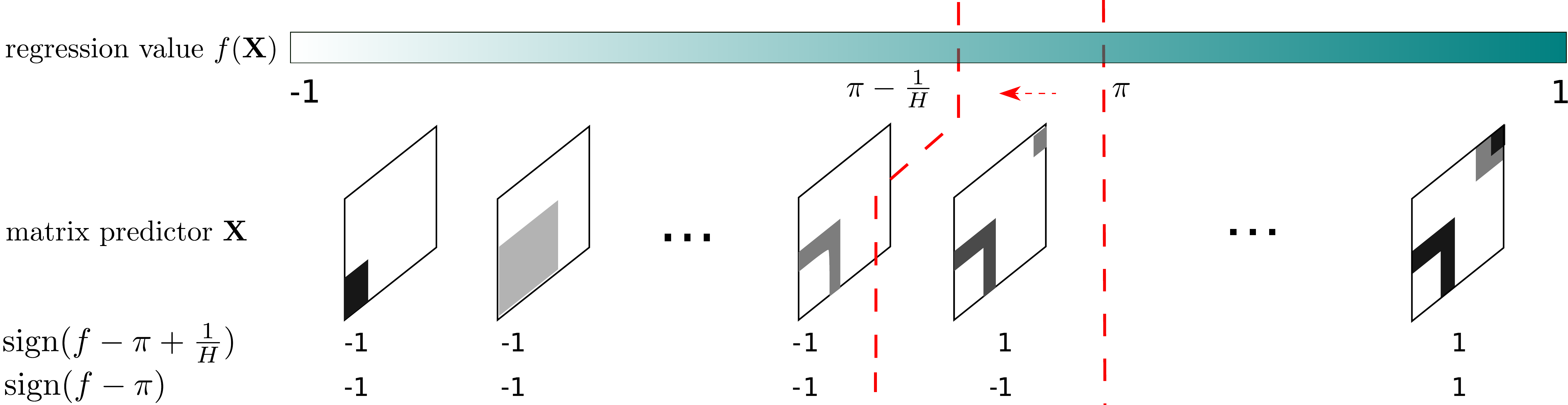}
\caption{Nonparametric matrix regression via sign function series estimation. We use a series of weighted classifications to estimate the sign functions, then obtain the regression function estimate via sign aggregations. Here, $\mX\in\tX$ denotes matrix-valued predictor, $f\colon \tX\to \mathbb{R}$ denotes regression function, and $\sign(f-\pi)\in\{-1,1\}$ is the sign function, where $\pi\in\{-1,\ldots,-1/H,0,1/H,\ldots, 1\}$ is the series of levels to aggregate in our algorithm.}
\label{fig:method}
\end{figure}

We comment that the $(2H+1)$ estimation tasks of the sign functions are fully separable, leading naturally to a parallel type computation. Moreover, the sign functions bridge the problems of level set estimation and Bayes classification, as we will detail in Section~\ref{sec:identifiability}. Fig~\ref{fig:method} illustrates our main idea graphically. We refer to our method as the {\bf \small A}ggegration of {\bf \small S}tructured {\bf \small SI}gn {\bf \small S}eries for {\bf \small T}race regression, and abbreviate it as \NonparaM.

Next, we describe the specific form of weighted classification, the uniqueness of the classification optimizer, as well as the accuracy guarantee of the estimator.

\subsection{Statistical characterization of sign functions via weighted classification}

For a given level $\pi\in[-1,1]$, define the $\pi$-shifted response $\bar Y_{\pi,i} =Y_i-\pi$ for $i\in[n]$. We propose a weighted classification objective function in~\eqref{eq:proposal} using 
\begin{equation}\label{eq:loss}
L(\phi;(\mX_i,\bar Y_{\pi,i})_{i\in[n]})={1\over 2n}\sum_{i=1}^n\KeepStyleUnderBrace{|\bar Y_{\pi,i}|}_{\text{response-specific weight}}\times\KeepStyleUnderBrace{|\sign \bar Y_{\pi,i} - \sign \phi(\mX_i)|}_{\text{classification loss}},
\end{equation}
where $\phi\in \Phi(r)$ is the trace function to be optimized, and $|\bar Y_{\pi, i}|$ serves as the weight. Such a response-specific weight incorporates the magnitude information of the response into classification, in that the response values that are far away from the target level are penalized more heavily in the objective \eqref{eq:loss}. In the special case of a binary response $Y_i\in\{-1,1\}$ and target level $\pi=0$, the objective \eqref{eq:loss} reduces to the usual classification loss. 

Next, define the weighted classification risk, 
\begin{equation}\label{eq:constrained}
\risk(\phi)=\mathbb{E}L(\phi; (\mX_i,\bar Y_{\pi,i})_{i\in[n]}),
\end{equation}
where the expectation is taken with respect to the joint distribution of $(\mX_i,Y_i)$ i.i.d.\ from $\mathbb{P}_{\mX,Y}$. The next theorem quantifies the global optimum of \eqref{eq:constrained}. 

\begin{thm}[Global optimum of weighted classification risk]\label{thm:oracle} For any given level $\pi\in[-1,1]$, under the model specified in \eqref{eq:model} and \eqref{eq:sign}, for all functions $\bar f$ that have the same sign as $\sign(f-\pi)$, it holds that $\risk(\bar f) = \inf\{\risk(\phi)\colon \phi\in \Phi(r)\}$. 
\end{thm}

\noindent
Theorem~\ref{thm:oracle} suggests a practical procedure to estimate $\sign(f-\pi)$ through weighted classifications. The result ensures that the sign function $\sign(f-\pi)$ minimizes the weighted classification risk. The inverse, however, may not hold true, due to possible multiple global optimizers of $\risk(\cdot)$. A simple example is a constant regression $f(\mX)=\mathbb{E}(Y|\mX) = c$, in which case, every function $\phi\in \Phi(r)$ minimizes $\risk(\cdot)$ at the level $\pi=c$. The next section resolves this issue by characterizing the uniqueness of the risk optimizer.

\subsection{Identifiability}\label{sec:identifiability}

To establish the statistical guarantee of the minimizer of $\risk(\cdot)$, we first address its uniqueness, up to some sign equivalence. It turns out the local behavior of the regression function $f$ around $\pi$ plays a key role to establish the identifiability of sign function series from weighted classifications.

We introduce some additional notation. We call $\bayesS(\pi)=\{\mX\in\tX\colon f(\mX)\geq \pi\}$ the Bayes set at level $\pi$, and $\partial \bayesS(\pi)=\{\mX\in \tX\colon f(\mX)=\pi\}$ the level set boundary. Note that there is a one-to-one correspondence between the sign function $\sign(f-\pi)$ and the Bayes set $\bayesS(\pi)$. We choose to present the results in terms of $\bayesS(\pi)$ for easier comparison with the existing classification literature \citep{tsybakov2004optimal,singh2009adaptive}. We call a level $\pi\in[0,1]$ a mass point if the level set boundary $\partial \bayesS(\pi)$ has a non-zero measure under $\mathbb{P}_{\mX}$. Let $\tN=\{\pi\in[-1,1] \colon \mathbb{P}_{\mX}\left[f(\mX)=\pi\right]\neq 0\}$ denote the collection of all mass points in $f$. Assume there exists a constant $c>0$, independent of the feature space dimension, such that $|\tN|\leq c<\infty$. We introduce a notion of smoothness for the cumulative distribution function (CDF) of $f(\mX)$ under measure $ \mathbb{P}_{\mX}$. 

\begin{defn} [$\alpha$-smoothness] \label{ass:decboundary} 
Suppose $\mathbb{P}_{\mX}$ is a continuous distribution, and denote the CDF $G(\pi)=\mathbb{P}_{\mX}[f(\mX)\leq \pi]$. A function $f$ is called $(\alpha,\pi)$-locally smooth, for a given $\pi \notin \tN$, if there exist constants $C=C(\pi)>0$ and $\alpha=\alpha(\pi)\geq 0$, such that
\begin{equation}\label{eq:mass}
\sup_{0\leq t<\rho(\pi, \tN)}{G(\pi+t)-G(\pi-t)\over t^{\alpha}}\leq C,
\end{equation}
where $\rho(\pi,\tN) = \min_{\pi'\in \tN} |\pi-\pi'|$ denotes the distance from $\pi$ to the nearest point in $\tN$. We make the convention that $\rho(\pi,\tN)=2$ (which equals the range of $\pi\in[-1,1]$) when $\tN$ is empty, and $\alpha=\infty$ when the numerator in \eqref{eq:mass} is zero. The largest possible $\alpha=\alpha(\pi)$ in \eqref{eq:mass} is called the smoothness index at level $\pi$. The function $f$ is called $\alpha$-globally smooth, if \eqref{eq:mass} holds with a global constant $C$ for all $\pi\in[-1,1]$ except for a finite number of levels.
\end{defn}

\noindent
Fig~\ref{fig:CDF} shows three examples of the CDF with various levels of smoothness. A small value of $\alpha<1$ indicates the infinite density at level $\pi$, or equivalently, when $G(\pi)$ jumps at $\pi$. A large value of $\alpha>1$ corresponds to the case of no point mass around $\pi$, or equivalently, when $G(\pi)$ remains flat. An intermediate case is $\alpha=1$ when $G(\pi)$ has a finite non-zero sub-derivative in the vicinity of $\pi$. The global smoothness index is the minimal $\alpha$ over all $\pi$'s; meanwhile, we allow exceptions for a finite number of levels. 

\begin{figure}[t]
\includegraphics[width=.95\textwidth]{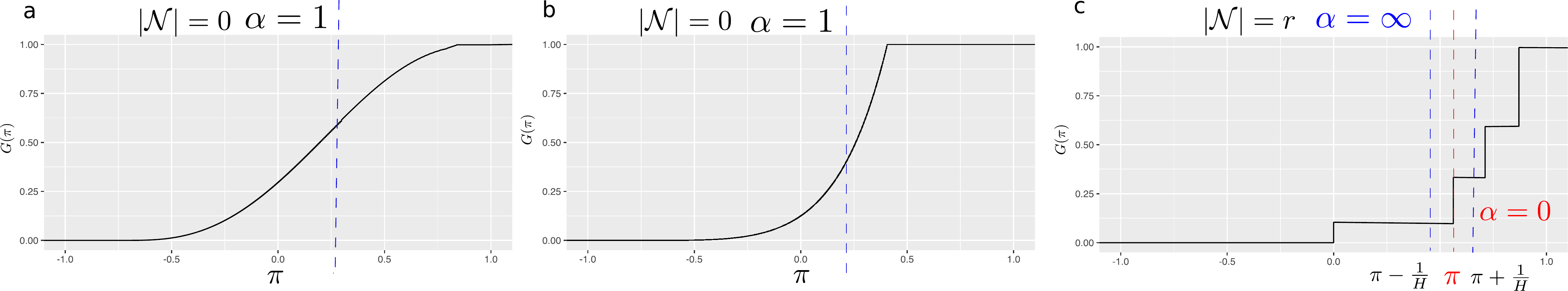}
\caption{Three examples of CDF, $G(\pi)=\mathbb{P}_{\mX}(f(\mX)\leq \pi)$, with local smoothness index $\alpha$ at $\pi$ depicted in dashed line. (a) and (b). Function $G(\pi)$ $\alpha=1$ because the $G(\pi)$ has finite sub-derivatives in the range of $\pi$; (c). Function $G(\pi)$ with $\alpha=\infty$ at most $\pi$ (in blue), except for a total number of $|\tN|=r$ jump points (in red). Here $|\tN|$ denotes the number of jump points.}
\label{fig:CDF}
\end{figure}

Next, we show that the $\alpha$-smoothness with $\alpha\neq 0$ implies the uniqueness of $\bayesS(\pi)$ for the optimizer of $\risk(\cdot)$. For two sets $S_1, S_2\in \tX$, define the probabilistic set difference, 
\begin{align*} 
d_{\Delta}(S_1,S_2) = \mathbb{P}_{\mX}(S_1\Delta S_2)=\mathbb{P}_{\mX}\{\mX\colon \mX\in S_1\setminus S_2 \text{ or }S_2\setminus S_1\},
\end{align*}
and the risk difference,
\begin{align*} 
d_\pi(S_1,S_2) = \risk(\sign(S_1))-\risk(\sign(S_2)).
\end{align*}

\begin{thm}[Identifiability]~\label{thm:identifiability} Suppose $f$ is $\alpha$-globally smooth over $\tX$. Then,
\begin{align}\label{eq:identity}
d_{\Delta}(S,\bayesS(\pi)) \lesssim \left[d_\pi(S,\bayesS(\pi))\right]^{\alpha\over 1+\alpha}+{1\over\rho(\pi, \tN)} d_\pi(S,\bayesS(\pi)),
\end{align}
for all sets $S\in\tX$ and all levels $\pi\in[-1,1]$ except for a finite number of levels.
\end{thm}

\noindent
We make two remarks. First, the bound~\eqref{eq:identity} controls the worst-case perturbation of the classifiers under the measure $\mathbb{P}_{\mX}$ with respect to the weighted classification risks. When $\alpha \neq 0$, the inequality \eqref{eq:identity} immediately implies the uniqueness, up to a measure-zero set in $\mathbb{P}_{\mX}$, of $\bayesS(\pi)$ in minimizing $\risk(\cdot)$. Second, our identifiability improves the earlier results for a single level set estimation to multiple level set estimations. Existing work \citep{singh2009adaptive,xu2020class} considered only a finite number of $\pi$'s, and provided only the first term in the bound \eqref{eq:identity}. In contrast, our bound quantifies the full dependence on the level $\pi$, and establishes the recovery condition of $\bayesS(\pi)$ uniformly over all possible $\pi$'s. It turns out both terms in the bound \eqref{eq:identity} are crucial for our regression function estimation. The first term contributes to the classification error, and the second term contributes to the variance in sign series aggregations.

\subsection{Regression risk bound}

In this section, we provide the statistical accuracy guarantee for the learning reduction based estimators~\eqref{eq:proposal} and~\eqref{eq:stepfunction}. Our theory consists of three main ingredients. We first leverage the $\alpha$-smoothness to provide a sharp rate for $\hat \phi_\pi$'s classification risk faster than the usual root-$n$ convergence. The improvement stems from the fact that, under the given assumptions, the variance of the excess classification loss is bounded in terms of its expectation. Because the variance decreases as we approach the optimal $\sign(f-\pi)$, the risk of $\hat \phi_\pi$ converges more quickly to the optimal risk than the simple uniform converge results would suggest.  The second step is to convert the risk error into the probability set error by Theorem~\ref{thm:identifiability}. The last step is to aggregate the set error into the final nonparametric function estimation. A careful error analysis reveals the joint contribution from both sign aggregations and variance-bias trade-off. 

The next result establishes the estimation accuracy for sign function estimator \eqref{eq:proposal}. 

\begin{thm}[Sign function estimation]\label{thm:main} Suppose the regression function $f\in\caliF(r)$ is $\alpha$-globally smooth over $\tX$, and let $d_{\max}=\max(d_1,d_2)$. Then, for all $\pi\in[-1,1]$ except for a finite number of levels, with high probability at least $1-\exp(-rd_{\max})$ over the training data $(\mX_i,Y_i)_{i\in[n]}$, we have, 
\begin{equation}\label{eq:riskbound}
\onenormSize{}{\sign \hat \phi_\pi- \sign(f-\pi)} \lesssim \left({rd_{\max} \over n}\right)^{\alpha\over 2+\alpha}+{1\over \rho^2(\pi, \tN)}\left({rd_{\max}\over n}\right),
\end{equation}
where the $L_1$ norm is taken with respect to the measure $\mX\sim\mathbb{P}_{\mX}$. 
\end{thm}

\noindent
Theorem~\ref{thm:main} quantifies the statistical convergence of the sign function estimation. For a fixed $\pi$, the second term in \eqref{eq:riskbound} is absorbed into the first term, leading to the rate $O(n^{-\alpha/(2+\alpha)})$. We find that the sign estimation reaches a fast rate $1/n$ when $\alpha =\infty$, and reaches a slow rate $1/\sqrt{n}$ when the point mass concentrates with $\alpha=0$. This is consistent with our intuition, because best rate $\alpha = \infty$ corresponds to a clear separation with no point mass at the Bayes set boundary $\partial \bayesS(\pi)$, whereas the worst rate $\alpha = 0$ corresponds to a heavy mass around $\partial \bayesS(\pi)$. Furthermore, the sign function estimation achieves consistency in the high dimensional region as long as $n \gg d_{\max}\to \infty$ and $\alpha\neq 0$. Combining the sign representability of the regression function and the uniform sign estimation accuracy, we obtain our main theoretical result on the nonparametric trace regression. 

\begin{thm}[Regression function estimation]\label{thm:regression} 
Suppose the same conditions in Theorem~\ref{thm:main} hold. With high probability at least $1-\exp(-rd_{\max})$ over the training data $(\mX_i,Y_i)_{i\in[n]}$, we have 
\begin{align}\label{eq:bound}
\onenormSize{}{\hat f-f} \lesssim \KeepStyleUnderBrace{\left({rd_{\max}\log H \over n}\right)^{\alpha \over 2+\alpha}}_{\text{estimation error from sign functions}}+\KeepStyleUnderBrace{1\over H}_{\text{reduction bias}}+\KeepStyleUnderBrace{\left({rd_{\max}\over n}\right)H \log H}_{\text{reduction variance}},
\end{align}
for any resolution parameter $H \in \mathbb{N}_{+}$. In particular, setting $H\asymp \left( {n\over rd_{\max}} \right)^{1/2}$ gives 
\begin{equation}\label{eq:final}
\onenormSize{}{\hat f-f} \lesssim \left({rd_{\max} \log n\over n}\right)^{\min\left({\alpha\over 2+\alpha}, {1\over 2}\right)},
\end{equation}
where the $L_1$ norm is taken with respect to the measure $\mX\sim\mathbb{P}_{\mX}$
\end{thm}

\noindent
Theorem~\ref{thm:regression} establishes the convergence rate of the proposed learning reduction estimator for the nonparametric trace regression. We make three remarks. First, the bound~\eqref{eq:bound} reveals three sources of errors: the estimation error from sign functions, the bias due to sign series representations, and the variance thereof. Recall that $H$ determines the number of sign functions in sign series representations. It controls the bias-variance tradeoff here. Second, the regression is robust to a few off-target classifications, as long as the majorities are accurate. This can also be seen in Fig~\ref{fig:CDF}(a) where the classification is nonidentifiable at some mass point (red line). Nevertheless, the regression estimation is still possible because the nearby classifications provide the sign signal (blue lines). This fact shows the benefit of sign aggregations, and also explains the trade-off in choosing $H$. Intuitively, a larger value of $H$ increases the approximation accuracy, but meanwhile renders the classification harder near the mass points. Third, the final regression error is generally no better than the sign error, as we compare the bounds in \eqref{eq:final} with \eqref{eq:riskbound}. This confirms our premise that classification is easier than regression. On the other hand, our sign representation approach allows us to disentangle the complexity and achieve the theoretical guarantee from classification to regression.

\section{Two applications of nonparametric matrix learning}
\label{sec:examples}

In this section, we apply the general theory in Theorem~\ref{thm:regression} to two specific nonparametric matrix learning problems, the low-rank sparse matrix predictor regression, and the high-rank matrix completion.

\subsection{Low-rank sparse matrix predictor regression}
\label{sec:sparse}

The first problem we consider is matrix predictor regression. In addition to the low sign rank structure, we also introduce a two-way sparsity structure. That is, we impose that some rows and columns of $\mB$ are zeros, where $\mB$ is as defined in \eqref{eq:sign}. We comment that sparsity is a commonly used structure in matrix data modeling \citep{zhou2014regularized}, and it is scientifically plausible in numerous applications \citep{Zhang2015}.

Specifically, we extend the notation $\Phi(r)$ and $\caliF(r)$ introduced in Definition~\ref{def:caliF} to incorporate the sparsity. Let $\Phi(r,s_1,s_2)$ denote the collection of trace functions, 
\begin{align*} 
\Phi(r,s_1,s_2)=\{\phi\colon \mX\mapsto \langle \mX, \mB \rangle +b \ \big| \text{rank}(\mB)\leq r,  \text{supp}(\mB)\leq (s_1,s_2), (\mB,b)\in\mathbb{R}^{d_1\times d_2}\times \mathbb{R}\},
\end{align*}
where $\text{supp}(\mB)$ denotes the support of $\mB$, with the sparsity parameters, $s_1=\newnormSize{}{\mB}_{1,0}=|\{i\in[d_1]\colon \mB_i\neq \mathbf{0}\}|$, and $s_2=\newnormSize{}{\mB^T}_{1,0}=|\{j\in[d_2]\colon \mB^T_j\neq \mathbf{0}\}|$, denoting the number of non-zero rows and non-zero columns of $\mB$, respectively. Similarly, let $\caliF(r,s_1,s_2)$ denote a family of rank-$r$, support-$(s_1,s_2)$ sign representable functions based on \eqref{eq:sign}. We have the following result. 

\begin{thm}[Nonparametric low-rank two-way sparse regression]\label{thm:sparse}
Consider the same setup as in Theorem~\ref{thm:regression}, except that we replace $\caliF(r)$ and $\Phi(r)$ with $\caliF(r,s_1,s_2)$ and $\Phi(r,s_1,s_2)$, respectively. Set $H\asymp {\left(n\over r(s_1+s_2)\log d_{\max}\right)}^{1/2}$ in~\eqref{eq:stepfunction}. With high probability at least $1-d_{\max}^{-r(s_1+s_2)}$ over the training data $(\mX_i, Y_i)_{i\in[n]}$, the estimate \eqref{eq:stepfunction} is bounded by
\begin{equation}\label{eq:final2}
\onenormSize{}{\hat f- f} \lesssim \left({r(s_1+s_2) \log d_{\max}\log n \over n}\right)^{\min\left({\alpha\over 2+\alpha}, {1\over 2}\right)}.
\end{equation}
\end{thm}

\noindent
We make two remarks. First, the bound \eqref{eq:final2} suggests that the estimator remains consistent in the high dimensional regime as $d_{\max}$ and $n\to \infty$, as long as $d_{\max}$ grows sub-exponentially in the sample size $n$. Such a sample complexity shows the pronounced advantage of the low-rank two-way sparse structural model, by comparing \eqref{eq:final2} and \eqref{eq:final}. Second, the two-way sparsity structure facilitates the interpretability, which we further demonstrate through numerical examples in Section~\ref{sec:comparison}.

\subsection{High-rank matrix completion}\label{sec:matrixcompletion}

The second problem we consider is matrix completion. Let $\mY\in\mathbb{R}^{d_1\times d_2}$ be a data matrix generated from the model,
\begin{equation}\label{eq:modelcompletion}
\mY=\mTheta+\mE,
\end{equation}
where $\mTheta\in\caliM(r)$ denotes an unknown signal matrix, and $\mE$ is an error matrix consisting of zero-mean, independent but not necessarily identically distributed entries. For simplicity, we assume $d_1=d_2=d$. Model \eqref{eq:modelcompletion} can be viewed as a special case of model \eqref{sec:idea}, where the predictor space consists of the basis matrices in $\mathbb{R}^{d\times d}$, and the data matrix $\mY=\entry{Y_{ij}}$ collects the scalar response $Y_{ij} \in \mathbb{R}$. In this case, the problem of regression estimation becomes the estimation of $\mTheta$. What is observed is an incomplete data matrix $\mY_\Omega$ from \eqref{eq:modelcompletion}, where $\Omega \subset [d]^2$ represents the index set of the observed entries. We allow both uniform and non-uniform sampling schemes for $\Omega$. Let $\Pi=\{p_\omega\}$ be an arbitrarily predefined probability distribution over the full index set with $\sum_{\omega\in[d]^2}p_\omega=1$. Assume the entries $\omega$ in $\Omega$ are i.i.d.\ draws with replacement from the full index set following the distribution $\Pi$. Denote the sampling rule as $\omega \sim \Pi$, and $\mY(\omega)$ the matrix entry indexed by $\omega$.  

Now applying our learning reduction approach to the matrix completion problem~\eqref{eq:modelcompletion} yields the signal matrix estimate
\begin{equation}\label{eq:est}
\hat \mTheta = {1\over 2H+1}\sum_{\pi \in \tH}\sign(\hat \mZ_\pi),
\end{equation}
where, for every $\pi\in\{-1,\ldots,-1/H,0,1/H,\ldots,1\}$, the matrix $\hat \mZ_\pi$ is the solution to the weighted classification
\begin{equation*}
\hat \mZ_\pi = \argmin_{\mZ\colon \rank(\mZ)\leq r}\sum_{\omega\in\Omega} \KeepStyleUnderBrace{|\mY(\omega)-\pi|}_{\text{weight}}\KeepStyleUnderBrace{|\sign(\mY(\omega)-\pi)-\sign(\mZ(\omega))|}_{\text{classification loss}}.
\end{equation*}

\noindent
To assess the accuracy of the estimate $\hat \mTheta=\hat \mTheta_{d\times d}$ in the high dimensional regime $d\to \infty$, we need to put the model in the nonparametric context of Definition~\ref{ass:decboundary}. We next extend the notion of $\alpha$-smoothness to a discrete feature space as follows. Let $\Delta s = 1/d^2$ denote a small tolerance, where $d^2$ represents the number of elements in the feature space. We quantify the distribution of the entries in matrix $\mTheta$ using a pseudo density, i.e., histogram with bin width $2\Delta s$. Specifically, let $G(\pi)=\mathbb{P}_{\omega\sim \Pi}[\mTheta(\omega)\leq \pi]$ denote the CDF of $\mTheta(\omega)$ under $\omega\sim \Pi$. We partition $[-1,1]=\tN \cup \tN^c$, where $\tN$ consists of levels whose pseudo density based on $2\Delta s$-bin is asymptotically unbounded; i.e,
\[
\tN=\left\{\pi\in[-1,1] \colon {G(\pi+{\Delta s})-G(\pi-{\Delta s})\over \Delta s} \geq c_1 \right\},\ \text{for some universal constant }c_1>0,
\]
and $\tN^c$ otherwise. Let $|\tN|_{\text{cover}}$ be the covering number of $\tN$ with $2\Delta s$-bin's; i.e, $|\tN|_{\text{cover}} =\text{Leb}(\tN)/2\Delta s$, where $\text{Leb}(\cdot)$ denotes the Lebesgue measure. The following assumption is a discrete analogy of Definition~\ref{ass:decboundary}.

\begin{defn}[$\alpha$-smoothness for discrete distribution] Let $\Pi$ be the sampling distribution over $[d^2]$. We say the signal matrix $\mTheta(\omega)$ is $\alpha$-globally smooth under $\omega\sim \Pi$, if there exist constants $c_2,c_3>0$, such that $|\tN|_{\text{cover}}\leq c_2$, and for all $\pi \in\tN^c$, 
\begin{equation*} \label{eq:smooth}
\sup_{\Delta s \leq t<\rho(\pi, \tN)}{G(\pi+{t})-G(\pi-{t})\over t^\alpha} \leq c_3, \;\; \text{ with } \rho(\pi,\tN)=\min_{\pi'\in \tN}|\pi-\pi'|+\Delta s 
\end{equation*}
and $\rho(\pi,\tN)$ denotes the adjusted distance from $\pi$ to the nearest point in $\tN$. 
\end{defn}
We assess the estimation error of~\eqref{eq:est} using the mean absolute error (MAE), $\text{MAE}(\hat \mTheta, \mTheta) = \mathbb{E}|\hat \mTheta(\omega)-\mTheta(\omega)|$, where the expectation is with respect to a future observation $\mTheta(\omega)$ from the distribution $G$.  We have the following result. 

\begin{thm}[Nonparametric matrix completion]\label{thm:estimation} 
Consider the matrix model~\eqref{eq:modelcompletion} with $\alpha$-smooth signal matrix $\mTheta\in\caliM(r)$. Set $H \asymp \left( |\Omega|\over dr\right)^{1/2}$. With high probability at least $1-\exp(-dr)$ over $\mY_\Omega$, the estimate \eqref{eq:est} satisfies that
\begin{equation}\label{eq:real}
\textup{MAE}(\hat \mTheta, \mTheta)\lesssim \left(dr \log|\Omega| \over |\Omega|\right)^{\min({\alpha \over 2+\alpha}, {1\over 2})}.
\end{equation}
\end{thm}

\noindent
We remark that our estimation accuracy \eqref{eq:real} applies to both low-rank and high-rank signal matrices. Moreover, the estimation rate depends on the sign complexity $\mTheta\in\caliM(r)$, where $r$ can be much smaller than the usual matrix rank as shown in Proposition \ref{prop:signrank}. In fact, our theorem can also be relaxed for a growing $|\tN|_{\text{cover}}$ as a function of $d$, with a slight modification on the setup; see Appendix~\ref{sec:unbounded} for such an extension. We next illustrate Theorem \ref{thm:estimation} with two matrix completion examples and compare with the existing literature.  

\begin{example}[Stochastic block model based matrix completion]
The stochastic block model \citep{chi2020provable} assumes a checkerboard structure under marginal row and column permutations. The signal matrix belongs to our sign representable family $\mTheta \in \caliM(r)$, where $r$ is the number of blocks. Besides, the block matrix is $\infty$-globally smooth, because $\tN$ consists of finitely many $2\Delta s$-bin's covering the block means. Our signal estimate achieves the rate $\tilde \tO(d^{-1/2})$ when $\alpha=\infty$ with no missingness. This rate agrees with the minimax root-mean-square error (RMSE) rate for stochastic block models with a fixed number of blocks \citep{gao2016optimal}. 
\end{example}

\begin{example}[Single index model based matrix completion]
The single index model based completion \citep{ganti2015matrix} admits a signal matrix $\mTheta=g(\mB)$, where $g$ is an unknown monotonic function, and $\mB$ is an unknown low-rank matrix. Note that $\mTheta$ itself is often of a high matrix rank as shown in Fig~\ref{fig:limit}(a). Suppose the CDF of $\mTheta(\omega)$ has a bounded pseudo density with $\alpha=1$. Applying Theorem~\ref{thm:estimation} yields the estimation error rate $\tilde \tO(d^{-1/3})$, which is faster compared to the RMSE rate $\tilde \tO(d^{-1/4})$ obtained earlier \citep{ganti2015matrix}. 
\end{example}

Finally, we obtain the sample complexity of the nonparametric matrix completion, summarized in the next corollary.

\begin{corollary}[Sample complexity for nonparametric completion] \label{thm:sample-complexity}
Suppose the same conditions of Theorem~\ref{thm:estimation} hold. When $\alpha\neq 0$, with high probability at least $1-\exp(-dr)$ over $\tY_\Omega$, 
\begin{equation*}
\textup{MAE}(\hat \mTheta, \mTheta)\to 0, \quad \text{as}\quad {|\Omega|\over {d} r \log|\Omega|}\to \infty.
\end{equation*}
\end{corollary}

\noindent
Corollary \ref{thm:sample-complexity} improves the earlier work~\citep{yuan2016tensor, pmlr-v119-lee20i} by allowing both low-rank and high-rank signals. Moreover, the sample size requirement depends only on the sign complexity $\tilde \tO(dr)$, but not the nonparametric complexity $\alpha$. We also note that $\tilde \tO(dr)$ roughly matches the degree of freedom of the signals, suggesting the optimality of our sample requirements.

\section{Large-margin implementation and ADMM algorithm}
\label{sec:estimation}

In Section \ref{sec:bridge}, we have established the methodology and theory for the nonparametric matrix trace regression under the 0-1 loss, since this is the canonical loss for classification. However, this loss may be difficult to optimize in practice. In this section, we extend it with a continuous large-margin loss, and present the corresponding optimization algorithm. We consider two loss functions: the hinge loss $F(z) = (1-z)_+$ for support vector machines, and the psi-loss $F(z)=2\min(1,(1-z)_+)$ with $z_{+}=\max(z,0)$~\citep{shen2003psi}. These two losses are most commonly used in classification, and both satisfy the linear excess risk bound; see Section~\ref{sec:large-margin}. We focus on the nonparametric low-rank sparse matrix regression problem. With some straightforward modification, the solution applies to matrix completion and other matrix learning problems as well.

\subsection{Large-margin learning} 

Specifically, we generalize the 0-1 loss minimization \eqref{eq:loss} to the following continuous large-margin loss minimization problem, 
\begin{align}\label{eq:large-margin}
\hat \phi_{\pi, F} = \argmin_{\phi \in\Phi(r,s_1,s_2)}\left\{ {1\over n}\sum_{i=1}^n |Y_i-\pi|F(\phi(\mX_i)\sign(Y_i-\pi))+ \lambda \FnormSize{}{\phi}^2\right\},
\end{align}
where $F(z)\colon \mathbb{R}\mapsto \mathbb{R}_{\geq 0}$ is a continuous function of the margin $z=y\phi(\mX)$, $\lambda>0$ is the penalty parameter, and $\FnormSize{}{\phi}$ is the penalty function. We set $\FnormSize{}{\phi}=\FnormSize{}{\mB}$, with $\mB$ being the coefficient matrix associated with $\phi\in\Phi(r,s_1,s_2)$. The use of large-margin loss in \eqref{eq:large-margin} allows us to leverage efficient large-margin optimization algorithms, while maintaining desirable statistical properties under mild conditions. The benefit of ridge penalization has been studied \citep{shen2003psi}. We obtain the corresponding regression function estimate as
\begin{equation}\label{eq:stepfunction-large-margin}
\hat f_{F} = {1\over 2H+1}\sum_{\pi \in \tH} \sign \hat \phi_{\pi, F}.
\end{equation}

\subsection{ADMM optimization}

We next present an algorithm to solve \eqref{eq:large-margin} for a given $\pi\in\tH$. We first note that the estimation problem \eqref{eq:large-margin} is equivalent to the optimization,
\begin{equation}\label{eq:sampleoptim}
\min_{\substack{(\mB,b)\colon  \rank(\mB)\leq r, \text{supp}(\mB)\leq (s_1,s_2)}}{1\over n}\sum_{i=1}^n|\bar Y_{\pi, i}| F\big( [\langle \mX_i,\mB \rangle+b] \sign \bar Y_{\pi, i}\big) + \lambda\FnormSize{}{\mB}^2,
\end{equation}
where we recall $\bar Y_{\pi, i}=Y_i-\pi$ is the $\pi$-shifted response. The loss function $F$ can be convex, e.g., hinge loss, or non-convex, e.g., psi-loss. Meanwhile, the optimization \eqref{eq:sampleoptim} has a non-convex feasible region because of the low-rank and sparsity constraints. 

We propose an alternating direction method of multipliers (ADMM) algorithm to solve \eqref{eq:sampleoptim}. We introduce a dual variable and an additional feasibility constraint to perform coordinate descent in the augmented Lagrangian function. The augmented objective of \eqref{eq:sampleoptim} is
\begin{equation*} \label{eq:ADMM}
L(\mB,b, \mS,\mLambda,\rho) = {1\over n}\sum_{i=1}^n|\bar Y_{\pi, i}|F\big([\langle \mX_i,\mB \rangle+b]\sign \bar Y_{\pi, i}\big)  + \lambda\FnormSize{}{\mB}^2+\rho\FnormSize{}{\mB-\mS}^2+\langle \mLambda, \mB-\mS\rangle,
\end{equation*}
where $\mB\in \mathbb{R}^{d_1\times d_2}$ is the unconstrained primal variable, $\mS\in\mathbb{R}^{d_1\times d_2}$ is the constrained dual variable satisfying $\rank(\mS)\leq r$ and $\text{supp}(\mS)\leq (s_1,s_2)$, $\mLambda\in\mathbb{R}^{d_1\times d_2}$ is the Lagrangian multiplier, and $\rho>0$ is the step size parameter. Note that in $L(\mB,b, \mS,\mLambda,\rho)$, the non-convexity has moved from the first two terms in $\mB$ to the last two simpler terms in $\mS$. This separability simplifies the optimization for a wide range of loss functions and constraints. 

We next minimize $L(\mB,b, \mS,\mLambda,\rho)$ via coordinate descent, by iteratively updating one variable at a time while holding others fixed. Each update reduces to a simpler problem and can be efficiently solved by standard algorithms. 

Specifically, given variables $(\mS,\mLambda,\rho)$ and $\bar \mS = (2\rho\mS-\mLambda) / [2(\rho+\lambda)]$, the objective with respect to $(\mB,b)$ is 
\begin{equation*} \label{eq:primal}
L(\mB,b|\mS,\mLambda,\rho)={1\over n}\sum_{i=1}^n |\bar Y_{\pi, i}| F\big ([\langle \mX_i,\mB \rangle+b]\sign \bar Y_{\pi, i}\big) +(\lambda+\rho)\FnormSize{}{\mB-\bar \mS}^2.
\end{equation*}
Optimization with~\eqref{eq:primal} is a standard vector based classification problem with a ridge penalty and an offset $\bar \mS$. There are a number of state-of-art algorithms for weighted SVM~\citep{wang2008probability} and psi-learning~\citep{shen2003psi}, which are readily available to solve this problem.  

Next, given $(\mB,b,\mLambda,\rho)$, and $\bar \mB=(2\rho\mB+\mLambda) / (2\rho)$, the objective with respect to $\mS$ is
\begin{equation}\label{eq:dual}
L(\mS|\mB,b,\mLambda,\rho)=\FnormSize{}{\mS-\bar \mB}^2,\quad \text{subject to} ~~ \rank(\mS)\leq r \text{ and }\text{supp}(\mS)\leq (s_1,s_2).
\end{equation}
This is equivalent to the best sparse low-rank approximation, in the least-square sense, to the matrix $\mB$. Compared to the original objective~\eqref{eq:sampleoptim}, the least-square objective is easier to handle. A number of learning algorithms have been designed to solve this problem, e.g., sparse PCA, sparse SVD, and projection pursuit \citep{Ma2013}. We adopt the recently developed double projection method, which has a competitive performance in the high dimensional regime \citep{Ma2016}. 

Finally, the Lagrangian multiplier $\mLambda$ is updated by $\mLambda\leftarrow\mLambda+2\rho(\mB-\mS)$. Following some common practice in matrix non-convex optimization \citep{Ma2016}, we run the optimization from multiple initializations to locate a final estimate with the lowest objective value. We summarize the above optimization procedure in Algorithm \ref{alg:weighted}.

\subsection{Hyperparameter tuning}

We briefly describe the hyperparameters in Algorithm~\ref{alg:weighted} and discuss their choices in practice. There are two sets of hyperparameters, one set for model specification, and the other for algorithmic stability. The model hyperparameters are $(r,s_1,s_2)$, which determine the complexity of sign functions. We choose $(r,s_1,s_2)$ via a grid search based on cross-validation regression error. The resolution in grid search depends on the problem size; for instance, in our brain connectivity data example with $d_1=d_2=68$ in Section~\ref{sec:brain}, we search for the optimal values of $r, s_1,s_2$ over $[d]$, with an increment of 5, under the natural constraint $r\leq s_1=s_2$. The algorithm hyperparameters are $(H, \lambda, \rho)$. For $H$ and $\lambda$, their optimal choices are given in Theorems \ref{thm:regression} and \ref{thm:extension}, respectively. In practice, we default $H=\min(20, \sqrt{n})$, and $\lambda=\min(0.1,n^{-1})$, which perform well in our numerical experiments. For the step size $\rho$ that controls the closeness between the dual and primal variables, we initialize from $1$, and increase its value geometrically by 1.1 during the iterations until the relative change in the primal residual $\FnormSize{}{\mB-\bar \mS}$ falls below a threshold~\citep{parikh2014proximal}. In our numerical analyses, we observe this scheme provides a stable optimization trajectory. 

\begin{algorithm}[t!]
\caption{{\bf Nonparametric low-rank two-way sparse matrix regression via ADMM} } \label{alg:weighted}
\begin{algorithmic}[1] 
\INPUT data $(\mX_i,Y_{\pi, i})_{i\in[n]}$, rank $r$, support $(s_1,s_2)$, ridge parameter $\lambda$, resolution parameter $H$.
\For {$\pi \in \tH=\{ -1, \ldots, -{1\over H}, 0, {1\over H},\ldots, 1\}$}
\State initialize dual variable $\mS$ randomly, Lagrangian multiplier $\mLambda=\mathbf{0}$, step size $\rho=1$, and $\bar Y_{\pi, i}$.
\Repeat
\State update $(\mB,b) \leftarrow \argmin L(\mB, b|\mS,\mLambda,\rho)$.
\State update $\mS \leftarrow  \argmin \FnormSize{}{\mS-{1\over 2\rho}(2\rho\mB+\mLambda)}^2 \ \text{subject to }\rank(\mS)\leq r$ and $\text{supp}(\mS)\leq (s_1,s_2)$.
\State update $\mLambda \leftarrow \mLambda+2\rho(\mB-\mS)$.
\State update $\rho\leftarrow1.1\rho$.
\Until convergence
\State return trace function estimate, $\hat \phi_\pi\colon \mX\mapsto \langle \hat \mB, \mX \rangle+\hat b$.
\EndFor
\OUTPUT nonparametric regression function estimate, $\hat f= {1\over 2H+1}\sum_{\pi \in \tH}\sign \hat \phi_\pi$.
\end{algorithmic}
\end{algorithm}

\subsection{Large-margin statistical guarantees}\label{sec:large-margin}

We next establish the statistical accuracy for the large-margin estimators under some additional technical assumptions. Let $\bayespif=\sign(f-\pi)$ denote the ground truth sign function at $\pi\in[-1,1]$, and let
\begin{align} \label{eq:riskdef}
\begin{split}
\risk(\phi) & =  {1\over 2}\mathbb{E}|Y-\pi| \big|\sign(Y-\pi)-\sign \phi(\mX)\big|, \\
\riskF(\phi) & =  \mathbb{E}|Y-\pi|F\big(\phi(\mX)\sign(Y-\pi)\big), 
\end{split}
\end{align}
denote the 0-1 risk and F-risk, respectively, where $F$ is the surrogate continuous loss, and the expectation is taken with respect to $(\mX,Y)\sim \mathbb{P}_{\mX,Y}$. For simplicity, we assume $d_1 = d_2 = d$ and $\FnormSize{}{\mX}\leq 1$ with probability 1. We consider the high dimensional regime where both $n$ and $d$ grow, while $(r,s_1,s_2)$ remain fixed. We need the following assumptions. 
 
\begin{assumption}[Assumptions on surrogate loss]\label{ass:main} \hfill
\begin{enumerate}
\item[(a)] (Approximation error) For any given $\pi\in[-1,1]$, assume there exist a sequence of functions $\phi^{(n)}_\pi\in\Phi(r,s_1,s_2)$, such that $\riskF(\phi^{(n)}_\pi)-\riskF(\bayespif)\leq a_n$, for some sequence $a_n\to 0$ as $n\to\infty$. Furthermore, assume $\FnormSize{}{\phi_{\pi}^{(n)}} \leq J$ for some constant $J>0$. 
\item[(b)] (Common loss) $F(z)=(1-z)_{+}$ is hinge loss, or $F(z)=2\min(1,(1-z)_{+})$ is psi-loss. 
\end{enumerate}
\end{assumption}

\noindent
Assumption~\ref{ass:main}(a) quantifies the representation capability of $F$ and $\Phi(r, s_1, s_2)$. We note that, although the Bayes rule $\bayespif$ also depends on $n$ implicitly through $d=d(n)$, we drop the dependence on $n$ for simpler notation. Assumption~\ref{ass:main}(b) implies the Fisher consistency bound for the weighted risk \citep{scott2011surrogate},
\begin{equation*} \label{eq:fisher}
\risk(\phi)-\risk(\bayespif)\leq C[\riskF(\phi)-\riskF(\bayespif)], \text{ for all $\pi\in[-1,1]$ and all $\phi$}.
\end{equation*}
where $C=1$ for the 0-1 or the hinge loss, and $C=1/2$ for the psi-loss; see Lemma~\ref{lem:prepare} in Appendix. Therefore, it suffices to bound the excess $F$-risk in order to bound the usual 0-1 risk. Under Assumption \ref{ass:main}, we establish the estimation accuracy guarantee for the large-margin estimators \eqref{eq:large-margin} and \eqref{eq:stepfunction-large-margin}. 

\begin{thm}[Large-margin estimation]~\label{thm:extension} 
Consider the same setup as in Theorem~\ref{thm:sparse}, and denote $t_n = {r(s_1+s_2)\log d \over n}$. Suppose the surrogate loss $F$ satisfies Assumption~\ref{ass:main} with $a_n \lesssim t_n^{(\alpha+1)/(\alpha+2)}$. Set $H\asymp t_n^{-1/2}$ in~\eqref{eq:stepfunction-large-margin} and $\lambda\asymp t_n^{(\alpha+1)/(\alpha+2)}+t_n/\rho(\pi,\tN)$ in \eqref{eq:large-margin}. Then, with high probability at least $1-\exp(-nt_n)$ over the training data $(\mX_i,Y_i)_{i\in[n]}$, we have:
\begin{enumerate}[label=(\alph*)]
\item (Sign function estimation). For all $\pi\in[-1,1]$ except for a finite number of levels,
\begin{equation*}
\onenormSize{}{\sign\hat \phi_{\pi,F}-\sign(f-\pi)}\lesssim t_n ^{\alpha\over 2+\alpha}+{1\over \rho^2(\pi,\tN)}t_n.
\end{equation*}

\item (Regression function estimation). 
\begin{equation*}
\onenormSize{}{\hat f_{F}- f} \lesssim  \left(t_n\log n\right)^{\min\left( {1\over 2},{\alpha \over 2+\alpha} \right)}.
\end{equation*}
\end{enumerate}
\end{thm}

\section{Simulations}
\label{sec:simulation}

In this section, we first evaluate the empirical performance of our method \NonparaM\ through four experiments, with varying sample size, response type, matrix dimension, and model complexity. We then compare \NonparaM\ with some alternative methods.

\subsection{Impacts of sample size, matrix dimension, and model complexity}
\label{sec:validation}

We consider a random matrix predictor $\mX\in\mathbb{R}^{d\times d}$ with i.i.d.\ entries sampled from Uniform[0,1], and simulate two types of response, continuous and binary, through 
\begin{itemize}
\item Continuous regression: $Y=f(\mX)+\varepsilon$, where $\varepsilon \sim \text{Normal}(0,0.1^2)$;
\item Binary regression: $Y\in\{-1,1\}$, with $\mathbb{P}(Y=1|\mX)={1\over 2}(f(\mX)+1)$.
\end{itemize}
We set the regression function $f(\mX)=h(z)$, where $h\colon \mathbb{R}\to[-1,1]$ is a non-decreasing function, $z\in\mathbb{R}$ is a nonlinear predictor that $z= (G^{-1}\circ \bar G)(\langle \mX, \mB \rangle)$, $\circ$ denotes function composition, $\mB\in\mathbb{R}^{d\times d}$ is a fixed rank-$r$, supp-$(s,s)$ matrix, $\bar G\colon \mathbb{R}\to[0,1]$ is the CDF of $\langle \mX, \mB \rangle$ induced by $\mX\sim \mathbb{P}_{\mX}$ so that $\bar G(\langle \mX, \mB \rangle)\sim$ Uniform[0,1], and $G\colon \mathbb{R}\to[0,1]$ is the CDF of some reference distribution. This construction yields a highly nonlinear function $f$. We set the matrix dimension $d=20,30,\ldots,60$, the training sample size $n=150, 200, \ldots, 400$, and various combinations of $(r,s)$. In this study, we set $\lambda=10^{-2}$, $H=20$, and use the true $(r,s)$ in Algorithm~\ref{alg:weighted}, and study parameter tuning in Section \ref{sec:comparison}. 

The first experiment assesses the impact of the sample size $n$ for the continuous regression. We set $h(z)=[\exp(z)-1] / [\exp(z)+1]$, $G$ as the CDF of a standard normal distribution, the matrix dimension $d=20$, and the model complexity $(r,s)=(2,2),(2,3),(5,5)$. Fig~\ref{fig:logistic}(a) summarizes the main model configurations, including the density of $z=z(\mX)$, the function $h=h(z)$, and the resulting density of $f(\mX)$. Fig~\ref{fig:logistic}(b) reports the prediction error, $\onenormSize{}{\hat f - f}$, as the sample $n$ increases. We see that the error decays polynomially with $n$. We also see that a higher rank $r$ or a higher support $s$ leads to a larger error, as reflected by the upward shift of the curve as $(r,s)$ increases, since it implies a higher model  complexity.  

The second experiment considers a binary response. Fig~\ref{fig:logistic}(c) reports the prediction error $\onenormSize{}{\hat f - f}$ as the sample size $n$ increases. We see that the error decays polynomially with $n$. We also note that, in both cases, the matrix predictor has the dimension $20 \times 20=400$ whereas $n$ is on the order of hundreds. Nevertheless, our nonparametric method consistently learns the function $f$ well from limited data without specifying a priori the functional form.

The third experiment evaluates the impact of the matrix dimension $d$. We fix the sample size $n=200$ and increase $d$. Fig~\ref{fig:logistic}(d) reports the prediction error. We see that  the error increases slowly with $d$, and the growth appears well controlled by the log rate. Note that, in this example, as $d$ increases, the number of effective entries remains unchanged, but the combinatoric complexity increases in the model space. The increasing error is an unavoidable price to pay for not knowing the positions of the $s$ active entries. This example shows the ability of our method to effectively handle a massive number of noisy features. 

\begin{figure}[b!]
\centering
\includegraphics[width=\textwidth]{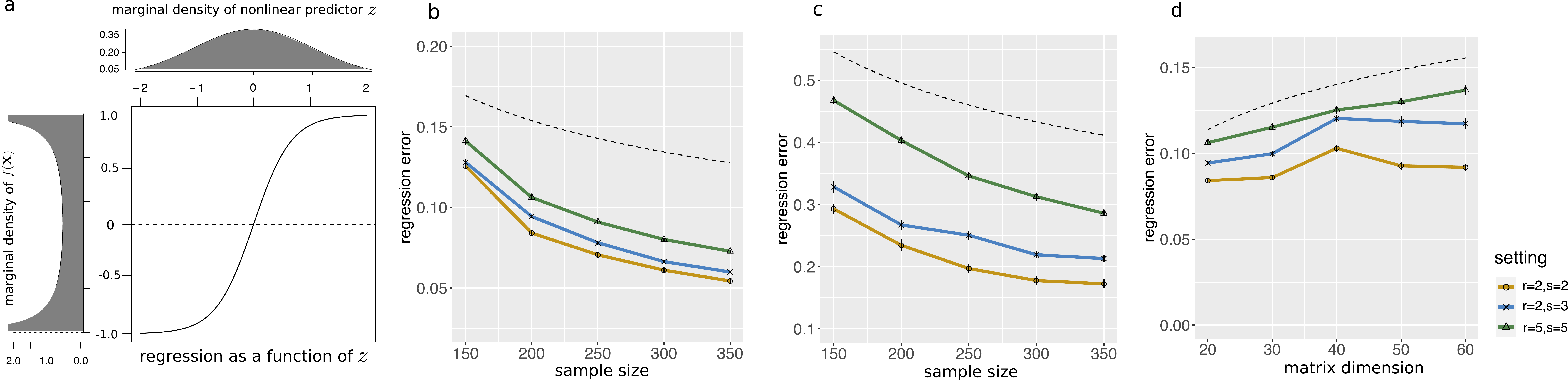}
\caption{Finite sample performance under a smooth function. (a) simulation setup; (b) prediction error with varying $n$ and $d=20$ for the continuous response; (c) for the binary response; (d) with varying $d$ and $n=200$. The dashed lines in panels (b)-(d) represent upper bounds $\tO(n^{-1/3})$, $\tO(n^{-1/3})$, and $\tO(\log d)$, respectively. The results are based on 30 data replications.} 
\label{fig:logistic}
\end{figure}

\begin{figure}[t!]
\centering
\includegraphics[width=\textwidth]{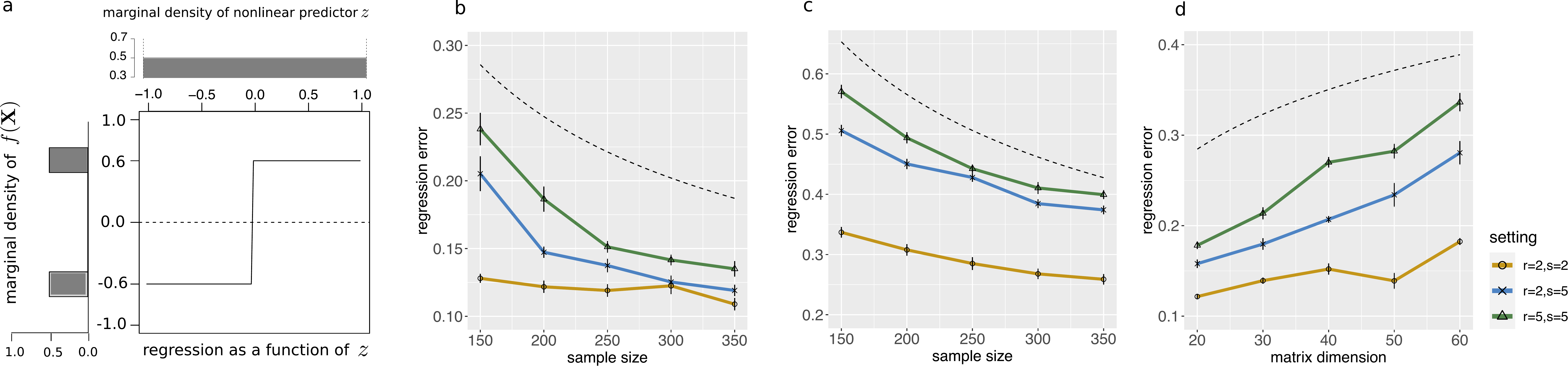}
\caption{Finite sample performance under a non-smooth function. The setup is similar as Fig~\ref{fig:logistic}. The dashed lines in panels (b)-(d) represent upper bounds $\tO(n^{-1/2})$, $\tO(n^{-1/2})$, and $\tO(\log d)$, respectively.}
\label{fig:step}
\end{figure}

The fourth experiment investigates the impact of smoothness in regression function. In Section~\ref{sec:idea}, we show that the probabilistic behavior of $f(\mX)$ plays a key role in our learning reduction approach. Here we assess the empirical performance by repeating all the above experiments using a model configuration with $z=z(\mX)\sim \text{Uniform}[-1,1]$, $h(z)=-0.6+1.2\mathds{1}(z>0)$, and $(r,s)=c(2,2),(2,5),(5,5)$. This case falls on the other end of the spectrum in contrast to the infinity smooth function in Fig~\ref{fig:logistic}(a). That is, $f(\mX)$ now concentrates at two mass points $\pi=\pm 0.6$. This makes the $\pi$-sign function estimation challenging around $\pi=\pm 0.6$ because of the non-identifiability. Fig~\ref{fig:step} reports the new model configurations and the corresponding results. Interestingly, we find that our method still maintains a good performance. Such a robustness may be explained by the fact that we aggregate in total $2H+1$ sign functions, each of which incurs at most $1/(2H+1)$ error to the regression function estimation. Therefore, our function estimate is robust against some off-target sign estimates, as long as the majority are accurate. This observation is consistent with the consistency result established in Section \ref{sec:bridge}.

\subsection{Comparison with alternative methods}
\label{sec:comparison}

Next, we compare our method with several popular alternative solutions. In this comparison, we adopt the simulation setup as in \cite{relion2019network}, but add more challenging matrix effects. Particularly, in this setup, the response is binary, and the predictor is a symmetric matrix that encodes a network. In this article, we have been targeting a general matrix predictor, which is directly applicable to a symmetric matrix, though we do not focus on symmetry. Moreover, as we show in Section~\ref{sec:joint} of the Appendix, the data generating model falls into our general family of nonparametric trace regression when there is no noise, but no longer so when there is noise. Therefore, we also investigate the performance of our method under model misspecification when including the noise. 

More specifically, we simulate from a latent variable model $(\mX,Y)|\pi$, where we generate $\pi$ i.i.d.\ from Uniform[0,1], and conditional on $\pi$, we generate $Y \sim \text{Bernoulli}(\pi)$, and 
\begin{eqnarray}\label{eq:pattern}
\mX=\entry{\mX_{ij}},\  \mX_{ij}\stackrel{\text{indep.}}{\sim} \text{Normal}\left( g_{ij}(\pi)\mathds{1}(\text{edge $(i,j)$ is active}), \sigma^2 \right), 
\end{eqnarray}
where the edge connectivity strength, denoted by $g_{ij}(\pi)$, varies depending on the location of $(i,j)\in[d]^2$, and the mean response $\pi$. Fig~\ref{fig:region} shows the activation pattern we consider that specifies the locations of the active edges. The active region is further divided into several subregions, each of which has its own signal function $g_{ij}(\cdot)\colon [0,1]\to \mathbb{R}$. The function form of $g_{ij}(\cdot)$ is randomly drawn from a pre-specified library consisting of common polynomial, log, and trigonometric functions. We set $d=68$, the training sample size $n=160$, and the testing size $80$. In the noiseless case $\sigma=0$ in~\eqref{eq:pattern}, the cross and block patterns are low-rank with $r = 3$ and 5, respectively, whereas the star and circle patterns are nearly full-rank, with a numerical rank $r \approx 30$ on the supported submatrix. 

\begin{figure}[t!]
\centering
\includegraphics[width=3.8cm]{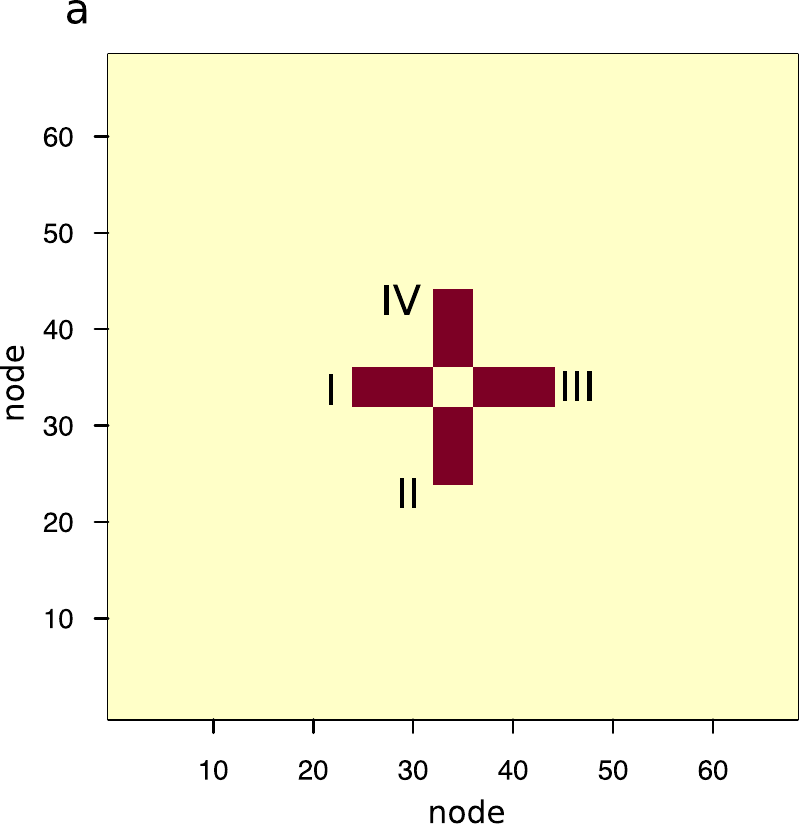}
\includegraphics[width=3.8cm]{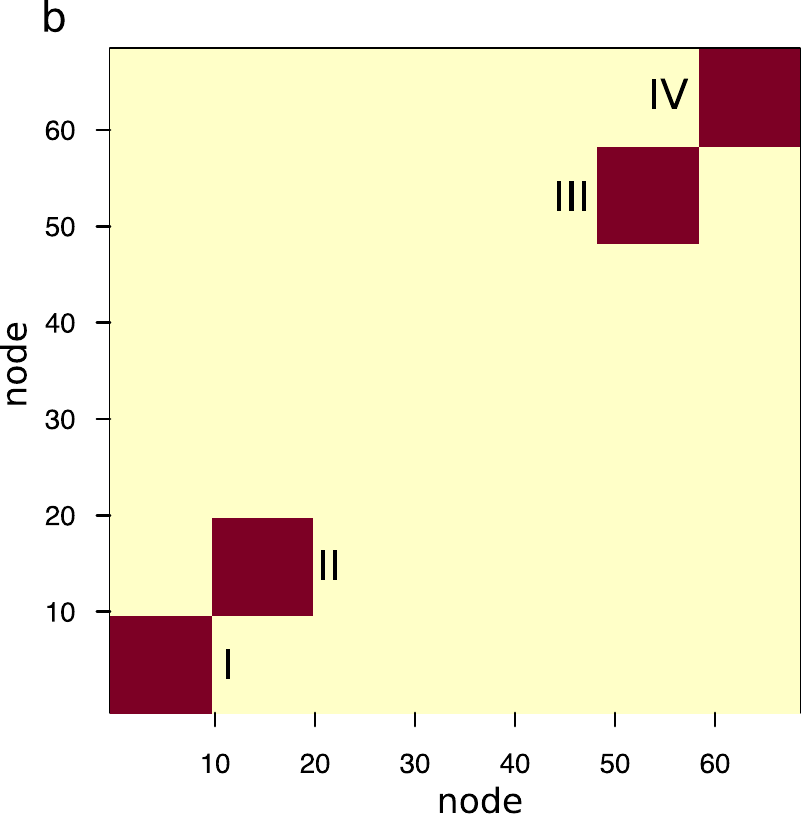}
\includegraphics[width=3.8cm]{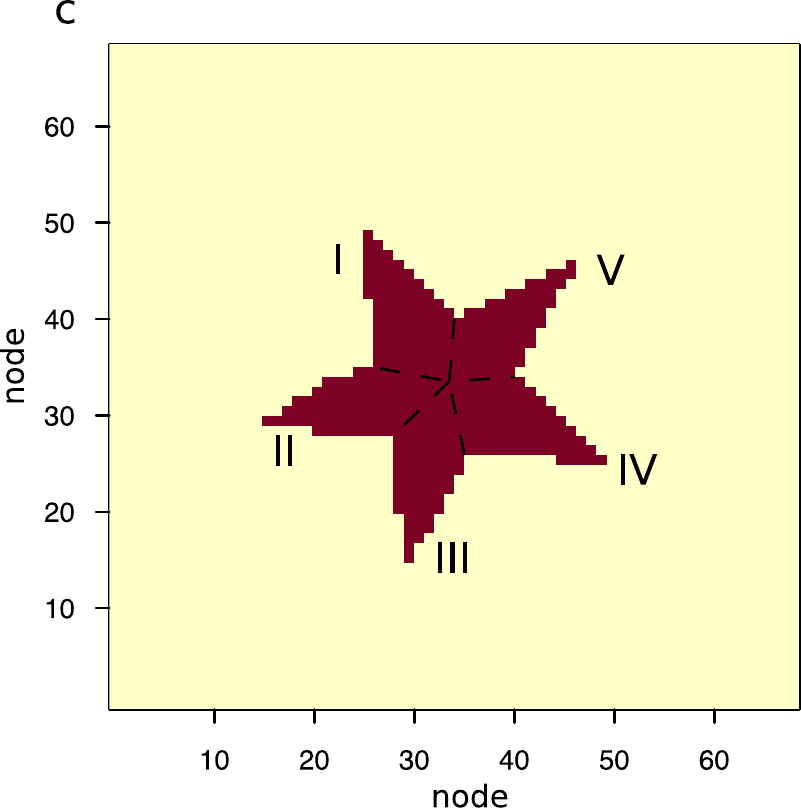}
\includegraphics[width=4.7cm]{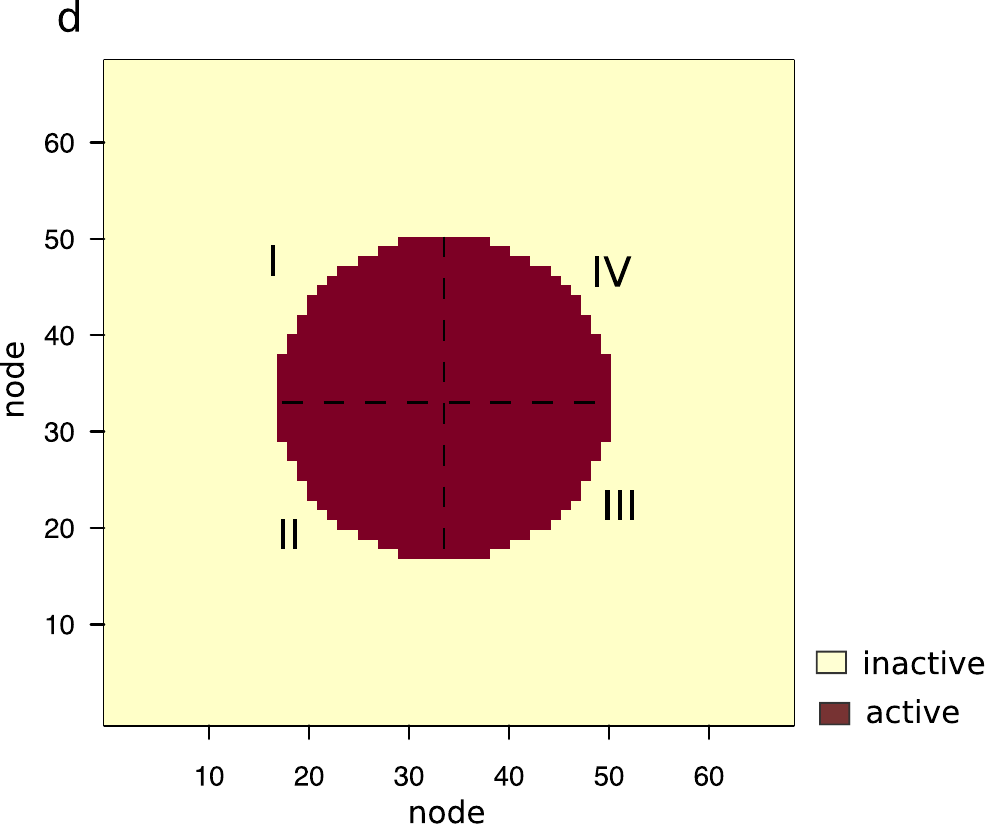}
\caption{Four activation patterns in simulations. The active region is divided into four or five subregions, denoted by I, II, ..., V, each of which has its own edge connectivity signal $g_{pq}(\pi)$.}
\label{fig:region}
\end{figure}

We compare the following four estimation methods. 

\begin{itemize}
\item Unstructured logistic regression for vector predictors (\Lasso, \citep{Zou2005}). This method vectorizes the matrix predictor into a high dimensional vector, then employs a logistic loss with an elastic net penalty. 

\item Generalized trace regression for matrix predictors (\LogisticM,~\citep{relion2019network}). This method fits a parametric trace regression model with a logistic link and a symmetric matrix predictor. It imposes a group lasso penalty to encourage two-way sparsity. 

\item Convolutional Neural Network (\CNN) with two hidden layers implemented in Keras \citep{chollet2018deep}. We apply 64 filters with $3\times 3$ convolutional kernels to the matrix-valued predictor, followed by a pooling layer with size $5\times 5$. The resulting features are fed to a fully connected layer of neural network with ReLU activation. 

\item {\bf A}ggegration of {\bf S}tructured {\bf SI}gn {\bf S}eries for {\bf T}race regression (\NonparaM), our method. 
\end{itemize}

\noindent 
Among these methods, \Lasso\ serves as a baseline to assess the gain of modeling a matrix predictor over a vector predictor, \LogisticM\ is a parametric model, whereas \CNN\ and \NonparaM\ are nonparametric solutions for matrix predictors. We feed each method with the binary response and the network adjacency matrix as the predictor after randomly permuting the node indices. Because \LogisticM\ only supports a symmetric matrix predictor, we provide it with $(\mX+\mX^T)/2$ as the input. We use the default parameters of \LogisticM, and select the tuning parameters of \Lasso, \CNN, and our method \NonparaM, including the rank $r$ and sparsity parameters $(r,s)$, by 5-fold cross validation. 

\begin{figure}[b!]
\centering
\includegraphics[width=\textwidth]{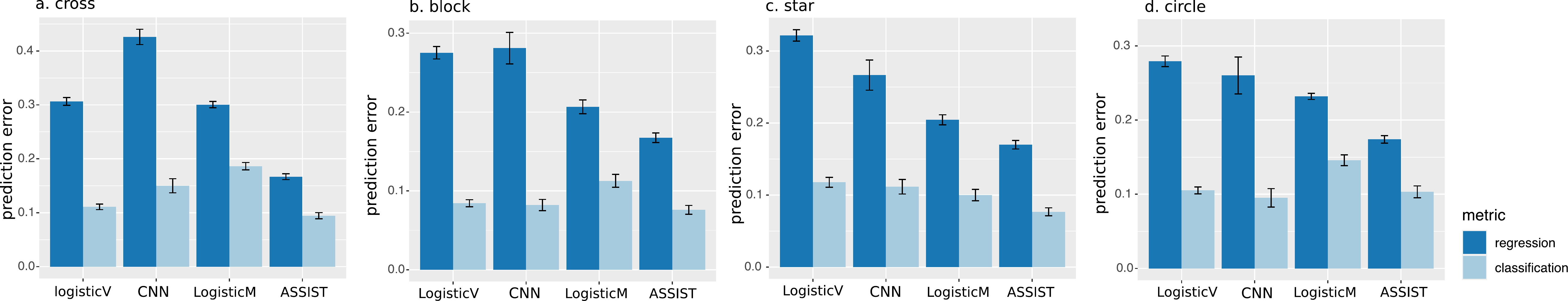}
\caption{Performance comparison of various methods under four different activation patterns. Reported are the prediction error $\onenormSize{}{\hat f - f}$, denoted by ``regression",  and the misclassification error at $\pi=1/2$, denoted by ``classification". The results are based on 30 data replications.}
\label{fig:compare}
\end{figure}

Fig \ref{fig:compare} reports both the prediction error $\onenormSize{}{\hat f - f}$ and the misclassification error at $\pi=1/2$ of the four methods evaluated on the testing data. For prediction, we see that \NonparaM\ consistently outperforms the alternatives, and the improvement is substantial. For example, the relative reduction using \NonparaM\ over the next best approach, \LogisticM, is over 20\% for patterns (a) and (d), and over 15\% for patterns (b) and (c). These results clearly demonstrate the benefit of our nonparametric approach. Moreover, we find that neither \Lasso\ nor \CNN\ has a satisfactory prediction. A possible explanation is that \Lasso\ takes the vectorized matrix as the input and therefore loses the two-way pairing information. Meanwhile, \CNN\ assumes spacial ordering within row and column indices. Although local similarity is important for the usual imaging analysis, the row and column indices take no particular order for a network. Actually, adjacency matrices after row or column permutation represent the same network, and thus the index-invariant methods, such as \LogisticM\ and \NonparaM, perform better. For classification, we also see that our method overall performs the best. The only exception is the circle pattern where \CNN\ has a slightly lower classification error. This is perhaps due to the fact that the circle is nearly full rank and thus favors a more complicated model.  Interestingly, we also find that the advantage of our method is more substantial in regression prediction than in classification, since classification is easier than regression. Moreover, with model noise included, our method still performs well even though the true model does not exactly follow our model specification. 

\begin{figure}[t!]
\centering
\includegraphics[width=8.1cm]{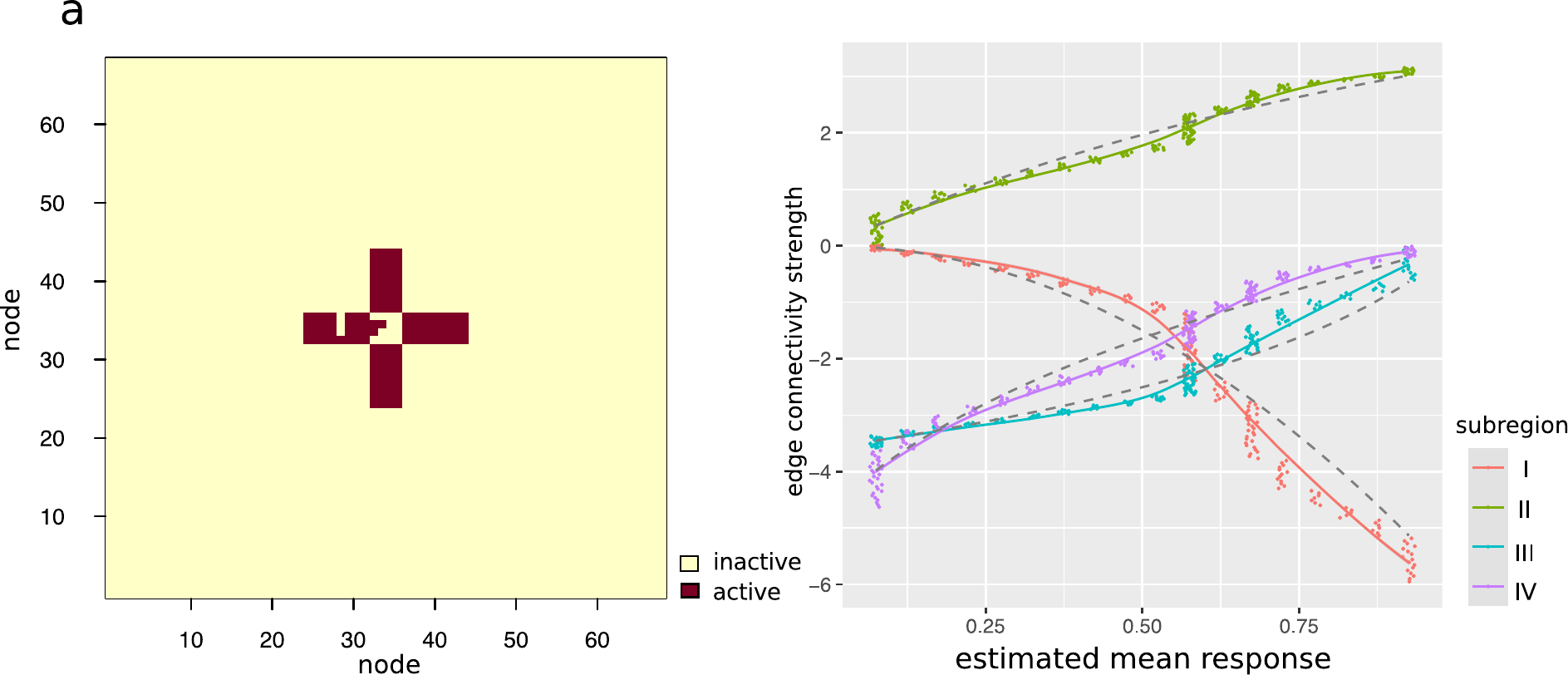} 
\includegraphics[width=8.1cm]{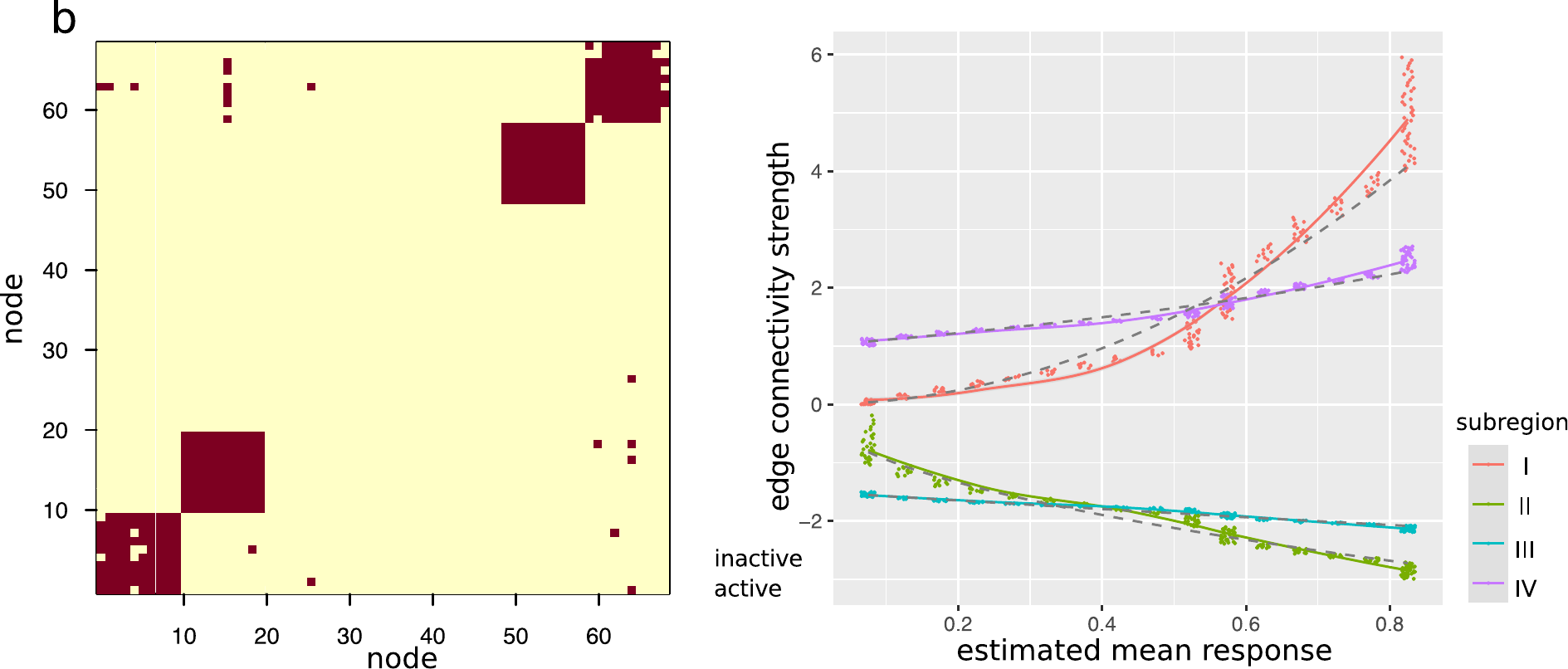}
\includegraphics[width=8.1cm]{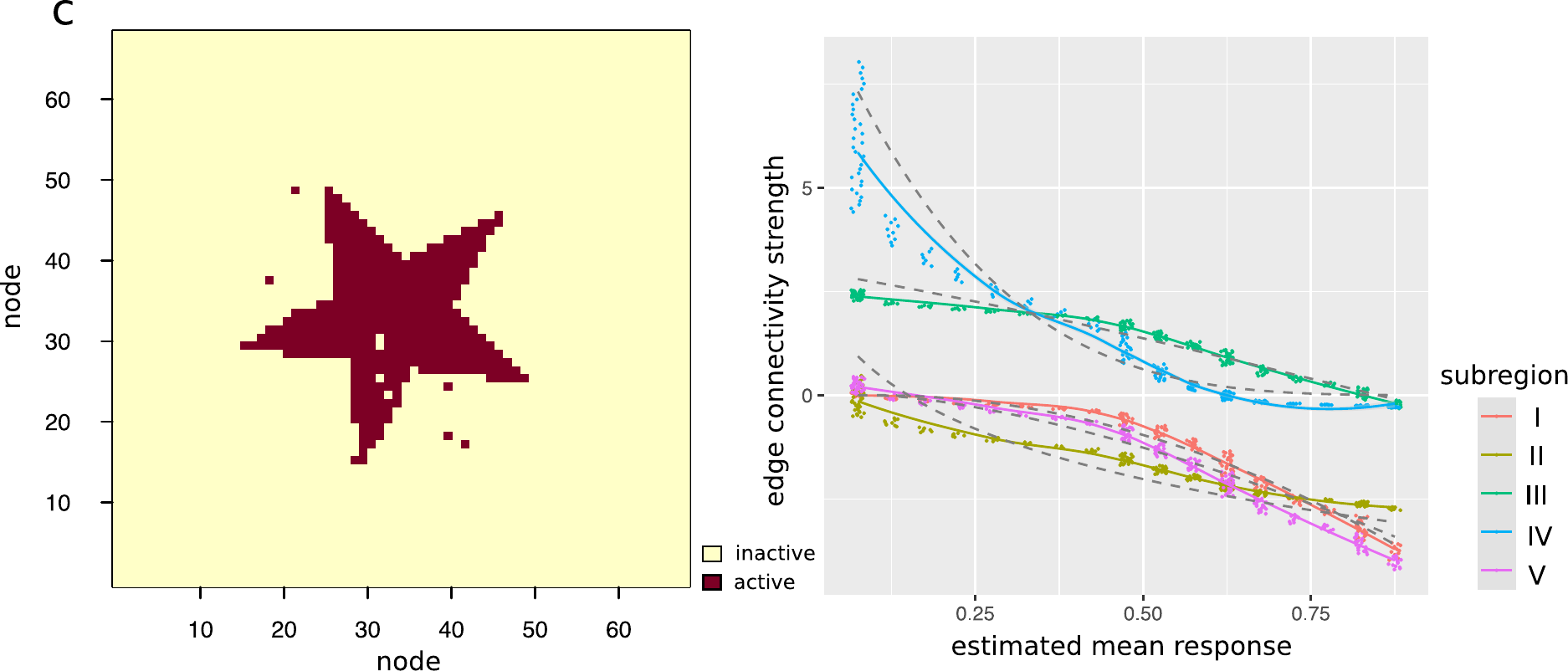}
\includegraphics[width=8.1cm]{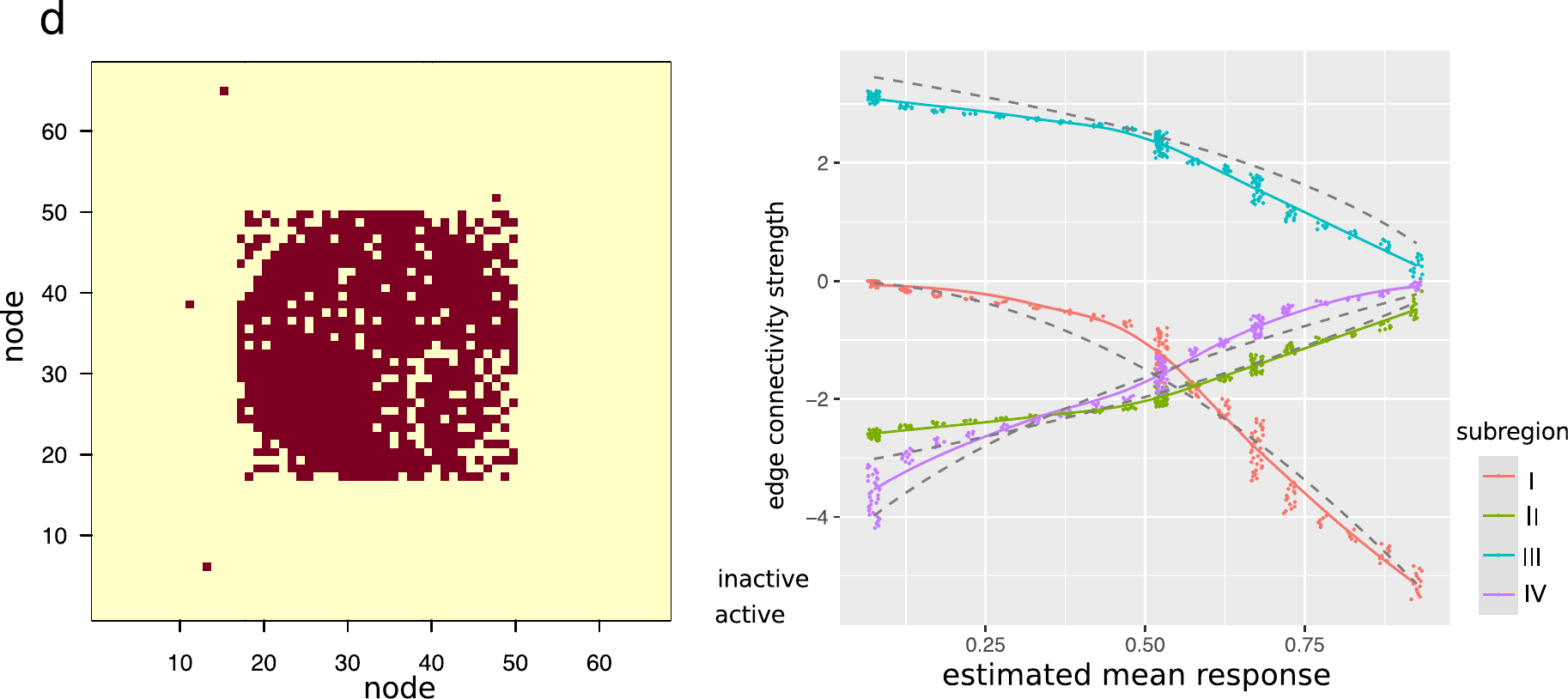}
\caption{Example output returned by {\bf \scriptsize ASSIST} based on the moving average of the feature weights, and the scatter plot of the edge connectivity strength, averaged by each subregion, versus the estimated mean response. The dashed curve shows the true function. }
\label{fig:compare2}
\end{figure}

Finally, to illustrate its capability of producing an estimate of high interpretability, Fig \ref{fig:compare2} reports the output of \NonparaM\ based on the moving average of the feature weights $(\hat \mB_\pi)_{\pi\in\Pi}$. It is observed that the identified activation region agrees well with the truth. We also investigate the relationship between the edge connectivity for individual $i$ and the estimated mean response $\hat \pi_i$ for $i=1,\ldots,n$. The trajectory accurately resembles the ground truth function in each subregion, demonstrating that our method is able to recover the pattern in the matrix predictors $\mX_i$ against $\hat \pi_i$ on a continuous spectrum.

\section{Real data applications}
\label{sec:realdata}

We present two real data applications, in parallel to the two matrix learning tasks studied in Section \ref{sec:examples}. The first task is the binary-valued trait prediction based on brain connectivity matrix regression, and the second is the continuous-valued matrix completion for imaging analysis.

\subsection{Brain connectivity analysis}
\label{sec:brain}

The first example is a brain connectivity data analysis, which aims to understand the relation between brain connectivity network and cognitive performance. The data is obtained from the Human Connectome Project (HCP) \citep{van2013wu}, and consists of $n=212$ healthy subjects. For each subject, a binary connectivity network is extracted, with nodes corresponding to $d=68$ brain regions-of-interest following the Desikan atlas \citep{desikan2006automated}, and links corresponding to the structural connectivity evaluated by diffusion tensor imaging \citep{zhang2018mapping}. The outcome is the dichotomized version of a visuospatial processing test score, corresponding to a high or low performance score \citep{wang2019common}. We adjust age and gender as additional covariates in our analysis. We note that, although our model focuses on a matrix predictor, it is straightforward to incorporate additional vector-valued covariates. We use a random 60-20-20 split of the data for training, validation, and testing. 

\begin{table}[t!]
\caption{Brain connectivity analysis. (a) Comparison of prediction accuracy measured by AUC, with standard errors over 5-fold cross validation in the parentheses. For {\scriptsize \bf CNN}, there is no report for node selection. (b) Top edges selected by the method {\scriptsize \bf ASSIST-p}. The letters ``r'' and ``l'' in node names indicate the right and left hemisphere, respectively. The $p$-value is calculated from the two-sample test of edge connection strength between two individual groups. }
\label{fig:real}
\resizebox{\columnwidth}{!}{
\begin{tabular}{ll}
a\hspace{6.5cm}b\\
\begin{tabular}{c|cc}
\hline
Method &  AUC  & \% of Active Nodes\\
\hline
{\bf \footnotesize ASSIST-p} &{\bf 0.73 (0.03)} &88.2   \\
\NonparaM& {\bf 0.77 (0.04)}  &97.3 \\
 LogisticM&0.72 (0.02)& 100.0\\
 LogisticV&0.68 (0.01)&89.7\\
CNN&0.67 (0.03)&-$^{}$\\
\hline
\end{tabular}
\begin{tabular}{c|ccc}
\hline
Rank &Node &  Node& $p$-value \\  
\hline
1&r-inferiortemporal&r-middletemporal&$0.01$\\
2&r-parstriangularis&r-supramarginal&3e-5\\
3&l-posteriorcingulate&r-precentral&0.01\\
4& l-caudalmiddlefronta& l-isthmuscingulate&2e-5\\
5 &l-lateralorbitofrontal&r-parstriangularis&1e-4\\
   \hline
\end{tabular}
\end{tabular}
}
\end{table}

\begin{figure}[b!]
\centering
\includegraphics[width=.8\textwidth]{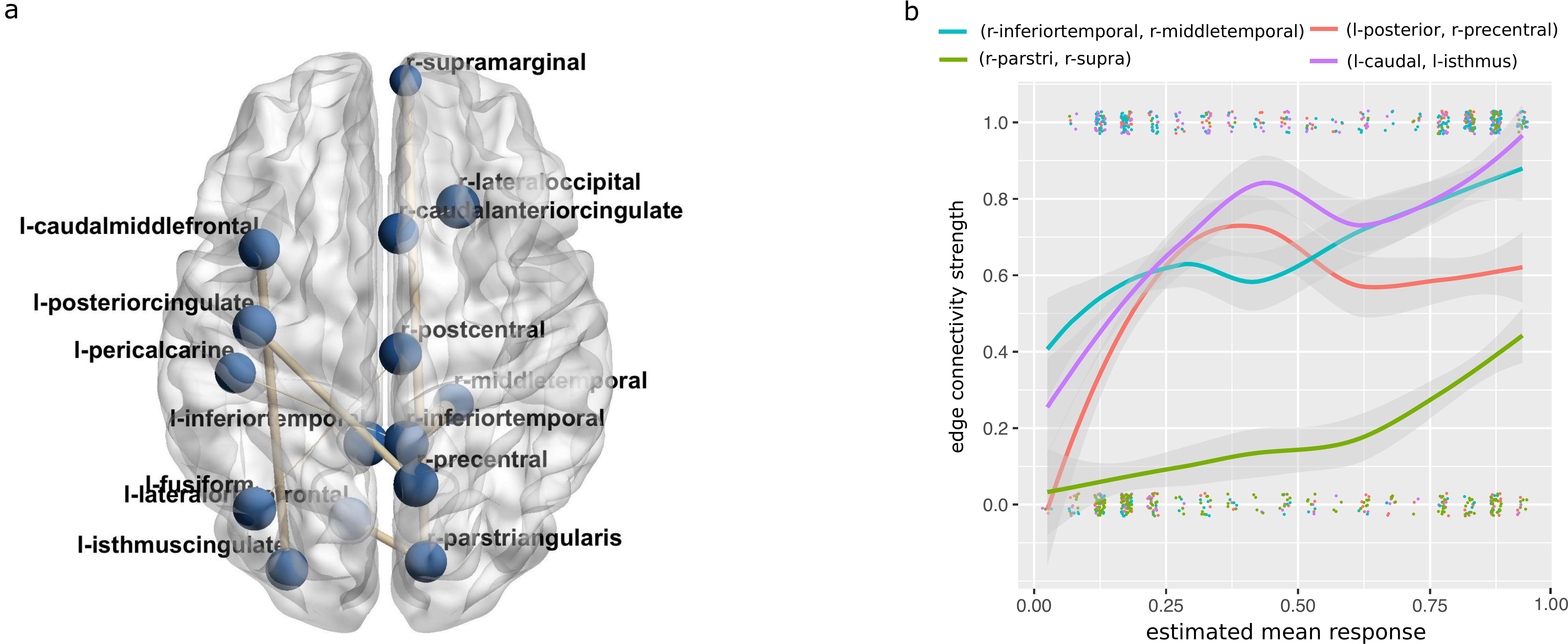}
\caption{Brain connectivity analysis.  (a) Top edges overlaid on a brain template. (b) Edge connectivity strength versus estimated mean response. Colored curves represent the moving averages of connectivity strengths, gray bands represent one standard error, and jitter points represent the raw connectivity values (0 or 1).}
\label{fig:real2}
\end{figure}

We compare our method with the same alternatives as in Section \ref{sec:comparison}. Table~\ref{fig:real}(a) shows that our method achieves the highest accuracy, measured by the area under receiver operating characteristic (AUC). Moreover, as common in the high dimensional setting, we see the model with a good cross-validation accuracy tends to include a large number of noise variables. A useful heuristic called the ``one-standard-error rule'', suggested by \cite{hastie2015statistical}, selects the most parsimonious model with cross-validation accuracy within one standard error of the best. We apply this rule and report the results as {\bf \footnotesize ASSIST-p}. It is remarkable to see that {\bf \footnotesize ASSIST-p} results in 12\% reduction of active nodes but still achieves a comparable accuracy to the best one. Table~\ref{fig:real}(b) lists the top brain links identified by our method. The edges are ranked by their maximal values in the feature weights $(\hat \mB_\pi)_{\pi \in \tH}$ via moving averaging. We find that the top edges involve connections between frontal and occipital regions in the right hemisphere. This is consistent with recent findings of dysfunction in right posterior regions for deficits in visuospatial processing \citep{wang2019common}. Fig~\ref{fig:real2}(a) shows the top selected edges overlaid on a brain template. Moreover, we find the relationship between the edge connection strength and the mean response to be nonlinear. Fig~\ref{fig:real2}(b) plots the edge connectivity strength versus the estimated mean response. We see that the connection between r-parstriangularis and r-supramarginal grows slowly when the mean response is small but fast when it is large. In contrary, the connection between r-posteriorcingulate and r-precentral  grows fast initially, then reaches a plateau as the mean response increases. Such patterns suggest heterogeneous changes in brain connectivity with respect to the visuospatial processing capability.

\subsection{Imaging matrix completion}
\label{sec:completion}

The second application is an imaging matrix completion, where the goal is to recover and restore the partially observed gray-scaled hot air balloon image. This image is a standard benchmark in computer vision, and is organized as a 217-by-217 matrix, whose entries represent pixel values in $[0,1]$. We randomly mask a subset of entries and perform matrix completion based on the observed entries. 

We compare our method with three alternatives: a soft imputation method based on matrix nuclear norm regularization (\SoftImpute)~\citep{hastie2015matrix}, a hard imputation method with ridge regression (\HardImpute)~\citep{mazumder2010spectral}, and a hard imputation based on alternating SVD (\ALT)~\citep{rennie2005fast}. We evaluate the recovery accuracy by MAE on the unobserved entries, and we tune all the parameters based on 5-fold cross-validation. 

\begin{figure}[t!]
\includegraphics[width = \textwidth]{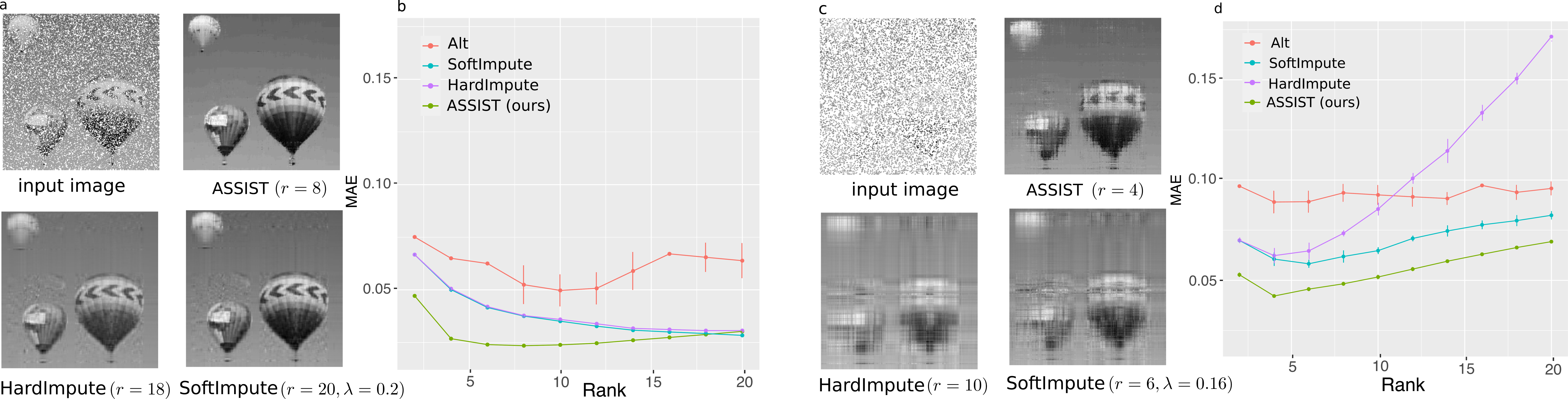}
\caption{Matrix completion analysis. (a)-(b) correspond to the 40\% missing rate, and (c)-(d) the 80\% missing rate. Error bars represent the standard error over 5-fold cross-validation. Numbers in the parentheses represent the selected tuning parameters for each method. In (a) and (c), we omit the worst method {\bf \scriptsize ALT} for space consideration.}
\label{fig:braincv}
\end{figure}

We investigate missing percentages at $40\%$ and $80\%$, and vary the rank $r=2,4,\ldots,20$. Fig \ref{fig:braincv} reports the performances of the four methods. We see clearly that our method achieves the best image recovery, with the smallest MAE. Besides, the advantage of our method compared to the alternative solutions is more clear when the missing percentage increases.

\section{Discussion}
\label{sec:discussion}

We have developed a nonparametric trace regression model for studying the relationship between a scalar response and a high dimensional matrix predictor. We propose a learning reduction approach, \NonparaM, using the structured sign function series, which bridges between regression and classification. We establish the theoretical bounds, which concern the fundamental statistical errors, are independent of specific algorithms, and serve as a benchmark on how well any algorithmic procedure could perform. Our numerical results demonstrate the competitive performance of the proposed method.

Our work unlocks several possible future directions. One is nonparametric modeling of other nonconventional predictors, such as tensors, functions, and manifold data. Other directions include multi-task learning and compressed sensing. Moreover, our learning reduction approach can be coupled with more sophisticated classifiers, such as neural networks, decision trees, and boosting, for sign function estimation. Finally, the theoretical guarantees we obtain are for the global optimum. How to characterize the behavior of the actual minimizer, or relatedly, the computational error for non-convex matrix based regression remains challenging and open. All these questions are warranted for future research.

\section*{Acknowledgements}
The research was supported in part by NSF DMS-1915978, NSF DMS-2023239, Wisconsin Alumni Research Foundation (to M.\ Wang), NIH R01 AG061303 (to L.\ Li), and NSF CCF-1740858 (to H.\ Zhang)

\bibliographystyle{plainnat} 
\bibliography{rArxiv_v1}    

\newpage
\appendix
\begin{center}
{\Large\bf Appendix for ``Nonparametric Trace Regression in High Dimensions via Sign Series Representation''}
\end{center}

\section{Additional theoretical results}\label{sec:additional}
\subsection{Sign rank and matrix rank}\label{sec:signrank}

In the main paper, we have provided several examples with high matrix rank but low sign rank. This section provides more examples and their proofs. 

\begin{example}[Max graphon]\label{example:max} Suppose the matrix $\mTheta\in\mathbb{R}^{d\times d}$ takes the form 
\[
\mTheta(i,j)=\log\left(1+{1\over d}\max(i,j)\right), \ \text{for all }(i,j)\in[d]^2.
\]
 Then 
 \[
 \rank(\mTheta)=d, \quad \text{and}\quad \srank(\mTheta-\pi)\leq 2\ \text{for all }\pi\in\mathbb{R}. 
 \]
\end{example}
\begin{proof}
The full-rankness of $\mTheta$ is verified from elementary row operations as follows
\begin{align}
\begin{pmatrix}
(\mTheta_2-\mTheta_1)/(\log(1+\frac{2}{d})-\log(1+\frac{1}{d}))\\(\mTheta_3-\mTheta_2)/(\log(1+\frac{3}{d})-\log(1+\frac{2}{d}))\\\vdots\\ (\mTheta_d-\mTheta_{d-1})/(\log(1+\frac{d}{d})-\log(1+\frac{d-1}{d}))\\\mTheta_d/\log(1+\frac{d}{d})
\end{pmatrix} = \begin{pmatrix}
 1&          0  &      \ddots  &        \ddots       &          0 \\
1& 1 & \ddots &            \ddots   &   \ddots          \\
      \vdots &     \vdots & \ddots &       \ddots &    \ddots         \\
 1 & 1 &1 & 1 &0\\
 1 & 1 &1 & 1 &1
\end{pmatrix},
\end{align}
where $\mTheta_i$ denotes the $i$-th row of $\mTheta$. 
Now it suffices to show $\srank(\mTheta-\pi)\leq 2$ for $\pi$ in the feasible range $(\log(1+{1\over d}),\ \log 2)$. In this case, there exists an index $i^*\in\{2,\ldots,d\}$, such that $\log(1+{i^*-1\over d})< \pi\leq \log(1+{i^*\over d})$. By definition, the sign matrix $\sign (\mTheta-\pi)$ takes the form
\begin{equation}\label{eq:matrix}
\sign (\mTheta(i,j)-\pi)=
\begin{cases}
-1, & \text{both $i$ and $j$ are smaller than $i^*$};\\
1, & \text{otherwise}.
\end{cases}
\end{equation}
Therefore, the matrix $\sign (\mTheta-\pi)$ is a rank-2 block matrix, which implies $\srank(\mTheta-\pi)=2$. 
\end{proof}

In fact, Example~\ref{example:max} is a special case of the following proposition. 

\begin{prop}[Min/Max graphon] Let $g\colon \mathbb{R}\to \mathbb{R}$ be a continuous function such that $g(z)=0$ has at most $r\geq 1$ distinct real roots. For given numbers $x_i, y_j\in[0,1]$ all $(i,j)\in[d]^2$, define a matrix $\mTheta \in\mathbb{R}^{d\times d}$ with entries
\begin{equation}\label{eq:max}
\mTheta(i,j)=g(\max(x_i,y_j)), \quad (i,j)\in[d]^2.
\end{equation}
Then, the sign rank of $\mTheta$ satisfies
\[
\srank(\mTheta)\leq 2r.
\]
The same conclusion holds if we use $\min$ in place of $\max$ in~\eqref{eq:max}. 
\end{prop}

\begin{proof} 
Without loss of generality, assume $x_1\leq \cdots\leq x_d$ and $y_1\leq \cdots \leq y_d$. Based on the construction of $\mTheta$, the reordering does not change the rank of $\mTheta$. Let $z_1<\cdots<z_r$ be the $r$ distinct real roots for the equation $g(z)=0$. We separate the proof for two cases, $r=1$ and $r\geq 2$. 

\begin{itemize}
\item When $r=1$. The continuity of $g(\cdot)$ implies that the function $g(z)$ has at most one sign change point. Based on the similar argument as in Example~\ref{example:max}, the matrix $\sign(\mTheta)$ is a rank-2 block matrix; i.e., 
\begin{align}
\sign(\mTheta)=1-2\ma\otimes \mb \quad \text{ or } \quad \sign(\mTheta) = 2\ma\otimes\mb -1,
\end{align}
where $\ma, \mb$ are binary vectors defined by
\[
\ma=(\KeepStyleUnderBrace{1,\ldots,1,}_{\text{positions for which $x_{i}<z_1$}}0,\ldots,0)^T,\quad\mb=(\KeepStyleUnderBrace{1,\ldots,1,}_{\text{positions for which $y_j<z_1$}}0,\ldots,0)^T.
\]
Therefore, $\srank(\mTheta)\leq \rank(\sign(\mTheta)) = 2$. 

\item When $r\geq 2$. By continuity, the function $g(z)$ is non-zero and remains an unchanged sign in each of the intervals $(z_s, z_{s+1})$, for $1\leq s\leq r-1$. Define the index set 
\[
\tI=\{s\in\mathbb{N}_{+}\colon \text{the interval $(z_s, z_{s+1})$ in which $g(z)<0$}\}.
\] 
We now prove that the sign matrix $\sign(\mTheta)$ has rank bounded by $2r-1$. To see this, consider the matrix indices for which $\sign(\mTheta)=-1$,
\begin{align}\label{eq:support}
\{(i,j)\colon \mTheta(i,j) <0 \} & = \{(i,j) \colon g(\max(x_i,y_j))<0\} \notag \\
&=\cup_{s\in \tI} \{(i,j)\colon \max(x_i,y_j)\in(z_s,z_{s+1})\}\notag\\
&=\cup_{s\in \tI}\Big( \{(i,j)\colon x_{i}< z_{s+1}, y_j<z_{s+1}\} \cap \{(i,j)\colon x_{i}\leq z_{s}, y_j\leq z_{s+1}\}^c\Big).
\end{align}
The equation~\eqref{eq:support} is equivalent to 
\begin{align}\label{eq:indicator}
\mathds{1}(\mTheta(i,j)< 0)&
=\sum_{s\in \tI}\left(  \mathds{1}(x_{i}< z_{s+1}) \mathds{1}(y_{j}< z_{s+1})- \mathds{1}(x_{i}\leq z_{s})\mathds{1}(y_{j}\leq z_{s})\right),
\end{align}
for all $(i,j)\in[d]^2$, where $\mathds{1}(\cdot)\in\{0,1\}$ denotes the indicator function. The equation~\eqref{eq:indicator} implies the low-rank representation of $\sign(\mTheta)$,
\begin{equation}\label{eq:sum}
\sign(\mTheta)=1-2\sum_{s\in \tI } \left(\ma_{s+1}\otimes \mb_{s+1} - \bar \ma_s\otimes \bar \mb_s\right),
\end{equation}
where $\ma_{s+1}, \bar \ma_{s}$ are binary vectors defined by
\[
\ma_{s+1}=(\KeepStyleUnderBrace{1,\ldots,1,}_{\text{positions for which $x_{i}<z_{s+1}$}}0,\ldots,0)^T,\quad \text{and}\quad
\bar \ma_s=(\KeepStyleUnderBrace{1,\ldots,1,}_{\text{positions for which $x_{i}\leq z_{s}$}}0,\ldots,0)^T,
\]
and $\mb_{s+1}, \bar \mb_{s}$ are binary vectors defined similarly by using $y_j$ in place of $x_i$. 
Therefore, by~\eqref{eq:sum} and the assumption $|\tI|\leq r-1$, we conclude that 
\[
\srank(\mTheta)\leq 1+2(r-1)=2r-1.
\]
\end{itemize}
Combining two cases yields that $\srank(\mTheta)\leq 2r$ for any $r\geq 1$.
\end{proof}

\begin{example}[Banded matrices]\label{example:banded} Let $\ma=(1,2,\ldots,d)^T$ be a $d$-dimensional vector, and define a $d$-by-$d$ banded matrix $\mM=|\ma\otimes \mathbf{1}-\mathbf{1}\otimes \ma|$. Then
\[
\rank(\mM)=d,\quad \text{and}\quad \srank(\mM-\pi)\leq 3, \quad \text{for all }\pi\in \mathbb{R}.
\]
\end{example}
\begin{proof}
Note that $\mM$ is a banded matrix with entries
\[
\mM(i,j)={|i-j|}, \quad \text{for all }(i,j)\in[d]^2.
\]
Elementary row operation shows that $\mM$ is full rank as follows,
\begin{align}
\begin{pmatrix}
(\mM_1+\mM_d)/(d-1)\\
\mM_1-\mM_2\\
\mM_2-\mM_3\\
\vdots\\
\mM_{d-1}-\mM_{d}
\end{pmatrix} = 
\begin{pmatrix}
1&1&1&\cdots&1&1\\
-1&1&1&\cdots&1&1\\
-1&-1&1&\cdots&1&1\\
\vdots & \vdots & \vdots & \vdots & \vdots & \vdots\\
-1&-1&-1&\cdots&-1&1
\end{pmatrix}.
\end{align}

We now show $\srank(\mM-\pi)\leq 3$ by construction. Define two vectors $\mb=(2^{-1},2^{-2},\ldots,2^{-d})^T\in\mathbb{R}^d$ and $\text{rev}(\mb)=(2^{-d},\ldots,2^{-1})^T\in\mathbb{R}^d$. We construct the following matrix
\begin{equation}\label{eq:A}
\mA=\mb\otimes\text{rev}(\mb)+\text{rev}(\mb)\otimes\mb.
\end{equation}
The matrix $\mA\in\mathbb{R}^{d\times d}$ is banded with entries
\[
\mA(i,j)=\mA(j,i)=\mA(d-i,d-j)=\mA(d-j,d-i)=2^{-d-1}\left(2^{j-i}+2^{i-j}\right),\ \text{for all }(i,j)\in[d]^2.
\] 
Furthermore, the entry value $\mA(i,j)$ decreases with respect to $|i-j|$; i.e., 
\begin{equation}\label{eq:decrease}
\mA(i,j) \geq \mA(i',j'), \quad \text{for all }|i-j|\geq |i'-j'|.
\end{equation}
Notice that for a given $\pi\in\mathbb{R}$, there exists $\pi'\in\mathbb{R}$ such that $\sign(\mA-\pi')=\sign(\mM-\pi)$. This is because both $\mA$ and $\mM$ are banded matrices satisfying monotonicity~\eqref{eq:decrease}. By definition~\eqref{eq:A}, $\mA$ is a rank-2 matrix. Henceforce, $\srank(\mM-\pi)=\srank(\mA-\pi')\leq 3.$
\end{proof}

\begin{example}[Identity matrices]
Let $\mI$ be a $d$-by-$d$ identity matrix. Then
\[
\rank(\mI)=d,\quad\text{and}\quad  \srank(\mI-\pi)\leq 3 \ \text{for all }\pi\in\mathbb{R}.
\]
\end{example}
\begin{proof}
Depending on the value of $\pi$, the sign matrix $\sign(\mI-\pi)$ falls into one of the two cases: 
\begin{enumerate}
\item[(a)] $\sign(\mI-\pi)$ is a matrix of all $1$, or of all $-1$; 
\item[(b)] $\sign(\mI-\pi)=2\mI-\mathbf{1}_d\otimes \mathbf{1}_d$.
\end{enumerate}
The first case is trivial, so it suffices to show $\srank(\mI-\pi)\leq 3$ in the second case. Based on Example~\ref{example:banded}, the rank-2 matrix $\mA$ in~\eqref{eq:A} satisfies 
\[
\mA(i,j)
\begin{cases}
=2^{-d}, & i=j,\\
\geq 2^{-d}+2^{-d-2}, & i\neq j.
\end{cases}
\]
Therefore, $\sign\left(2^{-d}+2^{-d-3}-\mA\right)=2\mI-\mathbf{1}_d\otimes \mathbf{1}_d$. We conclude that $\srank(\mI-\pi)\leq \rank(2^{-d}+2^{-d-3}-\mA)=3$. 
\end{proof}

\subsection{Extension to sub-Gaussian noise}\label{sec:sub-Gaussian}
In the main paper, we have assumed the bounded noise (and thus bounded response) in the regression model. Here we extend the results to unbounded response with sub-Gaussian noise. For notational simplicity, we state the results for the matrix completion problem with $d_1=d_2=d$. The results extend similarly to general nonparamatrix matrix regression; we omit the elaboration but only state the difference in the remark. 

Consider the signal plus noise model on matrix $\mY\in\mathbb{R}^{d\times d}$,
\begin{align*}
\mY = \mTheta+\mE,
\end{align*}
where $\mE$ consists of zero-mean, independent noise entries, and $\mTheta\in\caliM(r)$ is an $\alpha$-smooth matrix.
Theoretical results in Section~\ref{sec:examples} of the main paper are based on bounded observation $\|\mY\|_\infty\leq 1$. Here, we extend the results to unbounded observation with the following assumption.

\begin{assumption}[Sub-Gaussian noise]\label{assm:subg}\hfill

\begin{enumerate}
\item There exists a constant $\beta>0$, independent of matrix dimension, such that $\|\mTheta\|_\infty\leq \beta$. Without loss of generality, we set $\beta = 1$.
\item The noise entries $\mE(\omega)$ are independent zero-mean sub-Gaussian random variables with variance proxy $\sigma^2>0$; i.e, $\mathbb{P}(|\mE(\omega)|\geq B)\leq 2e^{-B^2/2\sigma^2}$ for all $B>0$.  
\end{enumerate}
\end{assumption}

We say that an event $A$ occurs ``with high probability'' if $\mathbb{P}(A)$ tends to 1 as the dimension $d\to \infty$. The following result show that the sub-Gaussian noise incurs an additional $\log d$ factor compared to the bounded case. 

\begin{thm}[Extension of Theorem~\ref{thm:estimation} to sub-Gaussian noise]\label{thm:extension_gaussian} Consider the same conditions of Theorem~\ref{thm:estimation}. Under Assumption~\ref{assm:subg}, with high probability over training data $\mY_{\Omega}$, we have
\begin{enumerate}
\item [(a)](Sign matrix estimation). For all $\pi\in [-1,1]$ except for a finite number of levels,
\begin{equation}\label{eq:matrix_sign}
 \textup{MAE}(\sign \hat Z_\pi,\sign (\mTheta-\pi))\lesssim \left({r \sigma^2d\log d \over |\Omega|}\right)^{\alpha+1\over \alpha+2}+{1\over \rho(\pi,\tN)}\left({r \sigma^2d\log d \over |\Omega|}\right).
\end{equation}
\item [(b)](Signal matrix estimation) Set $H\asymp \left({ |\Omega|\over r \sigma^2d\log d}\right)^{1/2}$. We have
\[
 \textup{MAE}(\hat \mTheta,\mTheta)\lesssim \tO\left\{\left({r \sigma^2d\log d\log|\Omega|\over |\Omega|}\right)^{\min({\alpha\over\alpha+2},\ \frac{1}{2})}\right\}.
\]
\end{enumerate}
\end{thm}
The proof is provided in Section~\ref{sec:sub-Gaussianproof}.  

\begin{rmk}[Extending to general non-parametric matrix regression] We have used matrix completion as an example to show the extension to unbounded noise; similar result applies to general matrix regression. 
For matrix nonparametrix regression (Theorem~\ref{thm:regression} of the main paper), the extension of bounded noise to sub-Gaussian noise incurs an additional $\log n$ factor, where $n$ is the sample size. The techniques of handling sub-Gaussian noise is identical to the above extension, and is thus omitted in the paper. 
\end{rmk}

\subsection{Extension to unbounded number of mass points}\label{sec:unbounded}
Theorem~\ref{thm:estimation} of our main paper assumes the bounded $|\tN|_{\text{cover}}<c<\infty$ for some constant $c>0$, where $|\tN|_{\text{cover}}$ is defined as the covering number of $\tN$ with $2\Delta s$-bin's. Recall that $\tN$ corresponds to regions of jumps greater than $\Delta s = {1/d^2}$ in the CDF $G(\pi)=\mathbb{P}_{\omega\sim \Pi}(\Theta(\omega)\leq \pi)$. This setup gives a cleaner exposition of our results but may be restricted in some cases. For example, the high-rank matrices in Example~\ref{ex:high-rank} and Figure~\ref{fig:limit}(b) are excluded, because $\alpha=\infty$ and $|\tN|_{\text{cover}}=d$ in this setup. Fortunately, our framework still applies to this family of matrices with a little amendment. 

We now extend the setup to allow for more general structured matrices including those in Example~\ref{ex:high-rank}. Redefine $\Delta s = {1/d}$. Correspondingly, redefine the smoothness index $\alpha$ and the set $\tN$ for the psudo density of $\mTheta(\omega)$ with new bin width $2\Delta s$. Let $|\tN|_{\text{cover}}$ be the covering number of $\tN$ with new $2\Delta s$-bin's. Under this new setup, the signal matrix in Example~\ref{ex:high-rank} has $|\tN|_{\text{cover}}=0$ and $\alpha<\infty$. Following the same line as in Theorem~\ref{thm:estimation} and use the fact that $\Delta s\lesssim t_d$, we obtain that
\[
\textup{MAE}(\hat \mTheta, \mTheta)\lesssim (t_d \log H)^{\alpha/(\alpha+2)}+{1\over H}+t_dH\log H,\quad \text{with }t_d={d r\over |\Omega|}.
\]
Therefore, setting $H\asymp t_d^{-1/2}$ yields the error bound
\begin{equation}\label{eq:inf}
\textup{MAE}(\hat \mTheta, \mTheta)\leq \tO\left\{\left( {dr\log|\Omega|\over |\Omega|}\right)^{\min({\alpha\over 2+\alpha},\ {1\over 2})}\right\}.
\end{equation}
The result~\eqref{eq:inf} applies to cases when the signal matrices belong to $\caliM(r)$ and have at most $d$ distinct entries with repetition patterns.

\subsection{Connection to structured matrix model with functional coefficients}\label{sec:joint}
In Section~\ref{sec:comparison} of the main paper, we simulate data $(\mX_i,Y_i)_{i=1}^n$ from latent variable model $(\mX,Y)|\pi$ based on the following scheme,
\[
\pi \stackrel{\text{i.i.d.}}{\sim} \text{Unif}[0,1] \stackrel{\text{conditional on $\pi$}}{\longrightarrow}
\begin{cases}
Y\sim \text{Ber}(\pi),\ Y\perp \mX|\pi, \\
\mX=\entry{\mX_{pq}}, \ \text{where\ } \mX_{pq}\stackrel{\text{indep.}}{\sim} \tN(g_{pq}(\pi)\mathds{1}(\text{edge $(p,q)$ is active}), \sigma^2).
\end{cases}
\]
Notice that, for any given $\pi$, $\mX$ is a rank-$r$, $(s_1,s_2)$ matrix as shown in Fig~\ref{fig:region} of the main paper. 

Here we provide justification to this simulation. We will show that the, in the absence of noise $\sigma=0$, the conditional expectation $\mathbb{E}(Y|\mX)=f(\mX)$ from the above simulation falls into the low-rank sign-representable function family of our interest.

Specifically, we consider a structured matrix model with functional coefficients
\begin{equation}\label{eq:scheme}
\mX_\pi\stackrel{\text{def}}{=}\mB_0+ \sum_{s=1}^rg_s(\pi)\mB_s+\sigma \mE,\quad Y_\pi\sim\text{Ber}(\pi),\quad \mX_\pi \perp Y_\pi |\pi,
\end{equation}
where $\pi \in[0,1]$ is drawn from $\text{Unif}[0,1]$; $\mE$ is a noise matrix consisting of i.i.d.\ entries in $N(0,1)$; $\sigma$ is the noise level; $\mB_0$ is an arbitrary baseline matrix; $(\mB_s)_{s=1}^r$ is a set of rank-1 matrices in $\{0,1\}^{d_1\times d_2}$ that satisfy three conditions:
\begin{enumerate}
\item non-overlapping supports, i.e., $\langle \mB_s, \mB_{s'}\rangle=0$ for all $s\neq s'$
\item bounded total support, i.e., $\sum_{s\in[r]}\supp(\mB_s)\leq (s_1,s_2)$;
\item At least one of the functions $(g_s)_{s=1}^r$ is strictly monotonic with respect to $\pi$ for all $s\in[r]$. \\
\end{enumerate}

\begin{prop}[Connection to structured matrix model with functional coefficients] Let $\mathbb{P}_{\mX,Y}$ denote the joint distribution induced by $(\mX_\pi,Y_\pi)_{\pi\in[0,1]}$ drawn from from~\eqref{eq:scheme}. In the noiseless case $\sigma = 0$, let $f(\mX)=\mathbb{E}(Y|\mX)$ denote the regression function based on $\mathbb{P}_{\mX,Y}$. Then $f\in\caliF(r,s_1,s_2)$.
\end{prop}

\begin{proof}
We restrict ourselves to the noiseless case with $\sigma=0$ in~\eqref{eq:scheme}. 
Let 
\[
\tX=\{\mX_\pi\colon \mX_\pi \text{ has structure specified in~\eqref{eq:scheme} for $\pi\in[0,1]$}\}
\]
denote the predictor space. The mapping between $\pi$ and $\mX\in \tX$ is one-to-one based on the construction of $\mX_\pi$. We use $\Pi\colon [0,1]\to \tX$ to denote the mapping and $\Pi^{-1}$ the inverse. Based on the property 3, without loss of generality, assume $g_1$ is a strictly increasing function.

For any given $\pi\in[0,1]$, we have
\[
\mathbb{E}_{Y|\pi}[Y|\pi]=\pi=\Pi^{-1}(\mX).
\]
This implies the regression function $f=\Pi^{-1}$. To show $f\in \caliF(r,s_1,s_2)$, it suffices to show $\Pi^{-1}\in\caliF(r,s_1,s_2)$. For any given $\pi'\in[0,1]$, write 
\begin{align}
\{\mX\in \tX \colon \sign(\Pi^{-1}-\pi')=1\}&=\{\mX\in \tX \colon \Pi^{-1}(\mX)\geq \pi'\}\\
&=\{\mX\in \tX \colon g_1(\Pi^{-1}(\mX))\geq g_1(\pi')\}\\
&=\left\{\mX\in \tX \colon \langle \mX, \mB_1\rangle \geq g_1(\pi')\langle \mB_1,\mB_1\rangle + \langle \mB_0, \mB_1 \rangle \right\},
\end{align}
where the second line uses the fact that $g_1$ is strictly increasing. 

Therefore, the sign function $\sign(\Pi^{-1}-\pi')$ can be expressed as the sign of trace function,
\[
\sign(\Pi^{-1}-\pi')=\sign(\KeepStyleUnderBrace{\langle \mX,\mB_1 \rangle}_{\text{trace}} - \KeepStyleUnderBrace{g_1(\pi')\langle \mB_1,\mB_1\rangle-\langle \mB_0,\mB_1\rangle}_{\text{intercept}}),\quad \text{for all }\mX\in\tX,
\]
where $\mB_1$ is a rank-1, supp-$(s_1,s_2)$ matrix coefficient. The proof is complete. 
\end{proof}

\begin{rmk}
The above result shows the connection of our method to joint matrix model~\eqref{eq:scheme} $(\mX_\pi, Y_\pi)_{\pi\in[0,1]}$. We should point out, despite of the seeming similarity, a fundamental challenge arises in our setting when the latent index $\pi$ is unobserved. Our sign aggregation approach essentially learns the right ordering of $\mX_\pi$ against the index $\pi\in[0,1]$ (see Figure~\ref{fig:method} of the main paper), thereby facilitating the estimation of regression function $f$. 
\end{rmk}

\subsection{Adjusting for intercept and additional covariates}\label{sec:intercept}
In the main paper, we estimate the trace function $\hat \phi_{\pi,F}\colon \mX\mapsto \langle \hat \mB, \mX \rangle + \hat b$ using optimization
\begin{align}\label{eq:phi}
(\hat \mB,\hat b)& = \argmin_{(\mB,b)} \left\{{1\over n}\sum_{i=1}^n|\bar Y_{\pi, i}| F\big( [\langle \mX_i,\mB \rangle+b] \sign \bar Y_{\pi, i}\big) + \lambda\FnormSize{}{\mB}^2\right\},\notag \\
\text{subject to }& \rank(\mB)\leq r,\ \supp(\mB)\leq (s_1,s_2).
\end{align}
The optimizer may not be unique; however, the following lemma shows that we can always choose an optimizer $(\hat \mB, \hat b)$ with bounded intercept without loss of generality. 

\begin{lem}[bounded intercept]\label{lem:intercept} Consider 0-1 loss, hinge loss, or phi-loss. Let $(\mX_i,Y_i)_{i\in[n]}$ be an arbitrary sample with $\FnormSize{}{\mX_i}\leq 1$. Then, there exists a global optimizer $( \mB_{\textup{opt}}, \mb_{\textup{opt}})$ of~\eqref{eq:phi} such that $|\mb_{\textup{opt}}|\leq \FnormSize{}{\mB_{\textup{opt}}}+1$. 
\end{lem}
Therefore, in this appendix, we will always assume the trace function family has the additional structure as in Lemma~\ref{lem:intercept}, i.e, 
\[
\Phi(r,s_1,s_2):=\{\phi\colon \mX\mapsto \langle \mX,\mB \rangle+b \ \big|\ \rank(\mB)\leq r,\ \supp(\mB)\leq (s_1,s_2),\ |b|\leq \FnormSize{}{\mB}+1\}.
\]
For ease of notation, we still use $\Phi(r,s_1,s_2)$ to denote this constrained trace function family.

\begin{proof}[Proof of Lemma~\ref{lem:intercept}]
We show that there always exists a global optimizer $(\mB_{\text{opt}}, b_{\text{opt}})$ of~\eqref{eq:phi} such that 
\begin{equation}\label{eq:b}
\min_{i\in [n]} |\langle \mX_i, \mB_{\text{opt}} \rangle +b_{\text{opt}}|\leq 1.
\end{equation}
Let $(\hat \mB, \hat b)$ be an arbitrary global optimizer of~\eqref{eq:phi}. Write $\hat \phi(\mX_i)=\langle \mX_i, \hat \mB \rangle+\hat b$, and $\bar Y_i=\bar Y_{\pi,i}$ for all $i\in[n]$. If $(\hat \mB, \hat b)$ satisfies~\eqref{eq:b}, then we keep this $(\hat \mB, \hat b)$. Otherwise, we aim to construct another global optimizer that satisfies~\eqref{eq:b}. Without loss of generality, assume that $(\hat \mB, \hat b)$ does not satisfy \eqref{eq:b}. The construction is divided into two cases based on loss functions. 
\begin{enumerate}[label={2.\arabic*},wide, labelwidth=!, labelindent=0pt]
\item[Case 1:] $F$ is 0-1 loss or psi-loss. 

Denote
\[
i^*=\argmin_{i\in[n]}|\hat \phi(\mX_i)|, \quad \text{and}\quad m:= \min_{i\in[n]}|\hat \phi(\mX_i)|=|\hat \phi(\mX_{i^*})|>1.
\]
We construct a shifted trace function,
\[
\phi^*\colon \mX\mapsto \hat \phi(\mX) - (m-1)\sign\hat \phi(\mX_{i^*}) = \KeepStyleUnderBrace{\langle \mX, \hat \mB \rangle}_{\text{trace}}+\KeepStyleUnderBrace{\hat b -(m-1)\sign\hat \phi(\mX_{i^*}) }_{\text{new intercept $=:\hat b^*$}}.
\]
The assumption $m>1$ implies that, for each $i\in[n]$, $\hat \phi(\mX_i) \sign \bar Y_{i}$ is either $\geq m>1$ or $\leq -m <-1$. By the definition of $\phi^*$ and loss function $F$, we have
\[
F\left(\phi^*(\mX_i)\sign \bar Y_{i}\right)=
\begin{cases}
F\big(\hat \phi(\mX_i)\sign \bar Y_{i}\big),   & \text{if $\hat \phi(\mX_i) \sign \bar Y_{i}\geq m>1$},\\
F\big(\hat \phi(\mX_i)\sign \bar Y_{i}\big),   & \text{if $\hat \phi(\mX_i) \sign \bar Y_{i}\leq m< -1$}.\\
\end{cases}
\]
Therefore, $(\hat \mB, \hat b^*)$ is also an optimizer of~\eqref{eq:phi}. Notice that $|\phi^*(\mX_{i^*})|=|\langle \mX_{i^*},\hat \mB \rangle+\hat b^*|=1$. Hence, we have found a global optimizer that satisfies~\eqref{eq:b}.
\item[Case 2:] $F$ is hinge loss. 

We construct $\hat \phi^*$ based on misclassified sample points. Denote
\[
\tI_+=\{i\in[n]\colon \sign \hat \phi(\mX_i)=-1 \text{ and } \sign \bar Y_i=1\},\quad \tI_-=\{i\in[n]\colon \sign \hat \phi(\mX_i)=1 \text{ and } \sign \bar Y_i=-1\}.
\]
If $\tI_+=\tI_-=\emptyset$, then we construct a shifted trace function $\hat \phi^*$ as in Case 1. Straightforward calculation shows that the resulting $(\hat \mB, \hat b^*)$ satisfies~\eqref{eq:b}. Now, suppose at least one of $\tI_+, \tI_-$ is nonempty. Define
\[
L_+=\sum_{i\in\tI_+}|\bar Y_i|,\quad \quad \text{and}\quad \quad L_-=\sum_{i\in\tI_-}|\bar Y_i|,
\]
where we make the conversion that the sum $\sum|\bar Y_i|$ is $-\infty$ if the index set is empty.  Define
\[
i^*=
\begin{cases}
\argmin_{i\in \tI_+}|\hat \phi(\mX_i)|, & \text{if }L_+\geq L_-,\\
\argmin_{i\in \tI_-}|\hat \phi(\mX_i)|, & \text{otherwise},\\
\end{cases}
\quad \quad \text{and}\quad \quad m:= |\hat \phi(\mX_{i^*})|>1,
\]
Notice that the construction of $i^*$ ensures $(L_+-L_-)\sign \hat \phi(\mX_{i^*})=-|L_+-L_-|$. We construct a shifted trace function
\begin{equation}\label{eq:new}
\phi^*\colon \mX\mapsto \hat \phi(\mX) - (m-1)\sign\phi(\mX_{i^*}) = \KeepStyleUnderBrace{\langle \mX, \hat \mB \rangle}_{\text{trace}}+\KeepStyleUnderBrace{\hat b -(m-1)\sign\hat \phi(\mX_{i^*}) }_{\text{intercept}}.
\end{equation}
By construction, 
\[
F\left(\phi^*(\mX_i)\sign \bar Y_{i}\right)=
\begin{cases}
F\left(\hat \phi(\mX_i)\sign \bar Y_{i}\right)=0,   & \text{if $\hat \phi(\mX_i) \sign \bar Y_i\geq m>1$},\\
F\left(\hat \phi(\mX_i)\sign \bar Y_{i}\right)+(m-1)\sign\hat \phi(\mX_{i^*}),   & \text{if }i\in \tI_{+},\\
F\left(\hat \phi(\mX_i)\sign \bar Y_{i}\right)-(m-1)\sign\hat \phi(\mX_{i^*}),   & \text{if }i\in \tI_{-}.\\
\end{cases}
\]
Therefore $\phi^*$ defined in~\eqref{eq:new} is a global optimizer of~\eqref{eq:phi}, since
\[
\sum_{i\in[n]}|Y_i|F(\hat \phi^*(\mX_i)\sign \bar Y_i) =\sum_{i\in[n]}|Y_i|F(\hat \phi(\mX_i)\sign \bar Y_i) - (m-1)|L_{+}-L_{-}| \leq \sum_{i\in [n]}|Y_i|F(\hat \phi(\mX_i)\sign \bar Y_i).
\]
Notice that $|\phi^*(\mX_{i^*})|=1$. Hence, we have found a global optimizer that satisfies~\eqref{eq:b}.
\end{enumerate}
Finally, the property~\eqref{eq:b} implies that
\[
|b_\text{opt}|\leq 1+\max_{i\in[n]}|\langle \mX_i, \mB_{\text{opt}}\rangle|\leq 1+\FnormSize{}{\mB_{\text{opt}}}.
\]
\end{proof}

Our Algorithm~\ref{alg:weighted} in the main paper can be extended to a mixture of matrix-valued predictors and usual vector-valued predictors. Specifically, we consider classifiers of the type $f(\mX)=\langle \mX, \mB\rangle+\mW^T\mC$, where $\mX\in\mathbb{R}^{d_1\times d_2}$ represents the matrix-valued predictor of our interest, $\mW\in\mathbb{R}^{p}$ represents the additional covariate including intercept, and $\mC\in\mathbb{R}^p$ is the unconstrained coefficient parameter. In our neuroimaging analysis (see Section~\ref{sec:brain} of the main paper), we have used $\mW$ to capture covariates such as age, gender, etc, in the prediction model. Our algorithm is amenable to this case. The only change is the primal update in the algorithm (Line 4 in Algorithm~\ref{alg:weighted} of main paper). The decision variables now consist of $(\mB,\mC)$ and we solve them simultaneously. Because both $\mB$ and $\mC$ are unconstrained decision variables, the algorithm lends itself well to this context.

\clearpage
\section{Proofs}\label{sec:proofs}
\subsection{Main notation}
\begin{table}[ht]
\begin{tabular}{l|l}
Notation & Definition \\
\hline
$(\mX,Y)$ & matrix predictor and univariate response\\ 
$(\mX_i,Y_i)_{i=1}^n$ & a sample of size $n$\\
$\tX$ & predictor space \\
$\shift=Y-\pi$ & shifted response\\
$f\colon \mX\mapsto \mathbb{E}(Y|\mX)$ & ground truth regression function \\
$\hat f\colon \mX\mapsto \mathbb{R}$ & estimated regression function \\
$\bayespif=\sign(f-\pi)$ & Bayes classifier at level $\pi$\\
$\bayesS(\pi)=\{\mX\in\tX\colon f(\mX)\geq \pi\}$ & Indicator set corresponding to $\bayespif$\\
$r$& matrix rank\\
$(s_1,s_2)$ & support parameter \\
$\caliF(r)$ & set of $r$-sign representable functions\\
$\Phi(r,s_1,s_2)$ & rank-$r$, supp-$(s_1,s_2)$ trace functions\\
$\Phi(r)$ & family of rank-$r$ trace functions\\
$\mB$ & rank-$r$, supp-$(s_1,s_2)$ matrix in trace function\\
$\alpha$ & smoothness index of $G(\pi)$\\
$\tN$ & set of mass points associated with CDF $G(\pi)=\mathbb{P}_{\mX}\left[f(\mX)\leq \pi\right]$ \\
$\rho(\pi,\tN)$ & distance from $\pi$ to nearest point in $\tN$\\
$H$ & resolution parameter in sign aggregation \\
$\phi$ & an arbitrary classifier function from $\tX$ to $\mathbb{R}$\\
$S_{\phi} = \{\mX\in\tX\colon \phi(\mX)\geq 0\}$ & Indicator set corresponding to $\phi$\\
$F$ & surrogate large-margin loss function from $\mathbb{R}$ to $\mathbb{R}_{\geq 0}$\\
$\hat \phi_{\pi,F}$ & estimated classifier function based on regularized empirical $F$-risk\\
$\ell_{\pi,F}$ & weighted $F$-loss function, i.e., $\ell_{\pi,F}(\phi;(\mX,Y)) = |\bar{Y}_\pi|F(\phi(\mX)\sign\bar{Y}_\pi)$\\
$\risk$ & weighted 0-1 risk \\
$\riskF$ & weighted surrogate $F$-risk\\
$\eriskF$ & empirical weighted  $F$-risk,  $\erisk$ is when $F$ is the 0-1 risk\\
$S$, $S_1$, $S_2$ & subsets in $\tX$\\
$d_{\Delta}(S_1,S_2)$&probability set difference, equal to $\mathbb{P}_{\mX}(\mX\in\tX\colon \mX\in S_1/S_2 \text{ or }S_2/S_1)$\\
$d_{\pi}(S_1,S_2)$& risk difference, equal to $\risk(\sign S_1)-\risk(\sign S_2)$\\
$\mY$ & data matrix with complete observation\\
$\Omega\subset[d_1]\times[d_2]$ & index set of observations\\
$\mY_{\Omega}$ & data matrix with incomplete observation\\
$\caliM(r)$ & family of rank-$r$ sign representable matrices\\
$\mTheta\in\caliM(r)$ & signal matrix in matrix completion problem\\
$\mE$ & noise matrix\\
$\mZ$ & an arbitrary matrix\\
\end{tabular}
\end{table}
\clearpage

\subsection{Proof of Theorem~\ref{thm:oracle}}
\begin{proof}
Fix $\pi\in[-1,1]$. For any arbitrary function $\phi\in \Phi(r)$, we evaluate the excess risk between $\sign(f-\pi)$ and $\sign \phi$,
\begin{align}\label{eq:risk}
&\risk(\sign \phi)- \risk(\sign(f-\pi)) \notag \\
= &\ {1\over 2}\mathbb{E}_{\mX}\KeepStyleUnderBrace{\mathbb{E}_{Y|\mX}\left\{|Y-\pi|\left[\left|\sign(Y-\pi)-\sign\phi \right|-\left|\sign(Y-\pi)-\sign(f-\pi)\right|\right]\right\}}_{\stackrel{\text{def}}{=}I}.
\end{align}
Here, $I=I(\mX)$ is a function of $\mX$, and its expression can be simplified as
\begin{align}\label{eq:I}
I&= \mathbb{E}_{Y|\mX}\left[ (Y-\pi)(\sign(f-\pi) - \sign \phi)\mathds{1}(Y\geq \pi)+(\pi-Y)(\sign\phi -\sign (f-\pi))\mathds{1}(Y< \pi)\right]\notag \\
&= \mathbb{E}_{Y|\mX}\left[(\sign(f-\pi)-\sign \phi) (Y-\pi)\right]\notag \\
&= \left[\sign(f-\pi)-\sign\phi \right]\left[f-\pi\right]\notag \\
&= |\sign(f-\pi)-\sign \phi ||f-\pi|,
\end{align}
where the third line uses the fact $\mathbb{E}_{Y|\mX}Y=f(\mX)$. Combining~\eqref{eq:I} with~\eqref{eq:risk}, we conclude that, for all $\phi\in\Phi(r)$, 
\begin{equation}\label{eq:minimum}
\risk(\sign \phi)- \risk(\sign(f-\pi)) ={1\over 2} \mathbb{E}_{\mX} |\sign(f-\pi)-\sign \phi ||f-\pi|\geq 0,
\end{equation}
where the last line equals to zero when $\sign \phi=\sign(f-\pi)$ or $f\equiv \pi$ is a constant function. Note that $(f-\pi)$ is $r$-sign representable by assumption. Therefore, 
\[
\risk(\sign(f-\pi))=\inf\{\risk(\sign \phi )\colon \phi\in \Phi(r)\}. 
\]
Based on the definition of 0-1 classification loss, the $\risk(\cdot)$ relies only on the sign of the argument function. Therefore, for all functions $\bar f \colon \tX\to\mathbb{R}$ that have the same sign as $\sign(f-\pi)$, we have
\[
\risk(\bar f)=\inf\{\risk(\sign \phi )\colon \phi\in \Phi(r)\}=\inf\{\risk( \phi )\colon \phi\in \Phi(r)\}.
\]
\end{proof}

\subsection{Proof of Theorem~\ref{thm:identifiability}}
\begin{proof}
Fix $\pi\in[-1,1]$. For ease of notation, we drop the dependence of $\pi$ in $\bayesS(\pi)$ and simply write $\bayesS$. Based on~\eqref{eq:I} in the proof of Theorem~\ref{thm:oracle}, we have
\begin{align}\label{eq:excess}
d_\pi(S,\bayesS) &\stackrel{\text{def}}{=} \risk(\sign (S))-\risk(\sign(\bayesS))\notag \\
&={1\over 2}\mathbb{E}_{\mX}\left( \left|\sign(S)-\sign(\bayesS) \right||\pi-f| \right)\notag \\
&=\int_{\mX\in S\Delta S_{\text{bayes}}}|f(\mX)-\pi|d \mathbb{P}_{\mX}.
\end{align}

We divide the proof into two cases: $\alpha >0$ and $\alpha=\infty$. 
\begin{enumerate}[label={2.\arabic*},wide, labelwidth=!, labelindent=0pt]
\item[Case 1:] $\alpha >0$.

Consider an arbitrary set $S\subset\mathbb{R}^{d_1\times d_2}$. Let $t$ be an arbitrary number in the interval $[0,1]$, and define the set $A=\{\mX\in \tX \colon |f(\mX)-\pi|>t\}$. 
\begin{align}
\int_{\mX\in S\Delta S_{\text{bayes}}}|f(\mX)-\pi|d \mathbb{P}_{\mX} &\geq t \left[\mathbb{P}_{\mX}(\left(S\Delta S_{\text{bayes}}) \cap A\right)\right] \\
&\geq t\left( \mathbb{P}_{\mX}\left(S\Delta S_{\text{bayes}}\right) - \mathbb{P}_{\mX}(A^c)\right)\\
&\geq t\left( \mathbb{P}_{\mX}\left(S\Delta S_{\text{bayes}}\right) - Ct^{\alpha}\right),\quad \text{for all }0\leq t < \rho(\pi,\tN),
\end{align}
where the last inequality is from $\alpha$-globally smoothness condition.
Combining the above inequality with the identity~\eqref{eq:excess} yields
\begin{equation}\label{eq:tail2}
d_\pi(S,\bayesS)\geq t\left( \mathbb{P}_{\mX}\left(S\Delta S_{\text{bayes}}\right) - Ct^{\alpha}\right),\quad \text{for all }0\leq t < \rho(\pi,\tN).
\end{equation}
We maximize the lower bound of~\eqref{eq:tail2} with respect to $t$, and obtain the optimal $t_{\text{opt}}$,
\[
t_{\text{opt}}=\begin{cases}
\rho(\pi,\tN), & \text{if}\quad \mathbb{P}_{\mX}\left(S\Delta S_{\text{bayes}}\right)> C(1+\alpha)\rho^\alpha(\pi,\tN),\\
\left[{1\over 2C(1+\alpha)}\mathbb{P}_{\mX}\left(S\Delta S_{\text{bayes}}\right)\right]^{1 / \alpha}, & \text{if}\quad \mathbb{P}_{\mX}\left(S\Delta S_{\text{bayes}}\right)\leq C(1+\alpha) \rho^\alpha(\pi,\tN).
\end{cases}
\]
The corresponding lower bound of the inequality~\eqref{eq:tail2} becomes
\begin{align}
  d_\pi(S,\bayesS)\geq\begin{cases}
c_1\rho(\pi,\tN)\mathbb{P}_{\mX}(S\Delta \bayesS), & \text{if}\quad \mathbb{P}_{\mX}\left(S\Delta S_{\text{bayes}}\right)> C(1+\alpha)\rho^\alpha(\pi,\tN),\\
c_2 \left[\mathbb{P}_{\mX}\left(S\Delta S_{\text{bayes}}\right)\right]^{1+\alpha \over \alpha}, & \text{if}\quad \mathbb{P}_{\mX}\left(S\Delta S_{\text{bayes}}\right)\leq C(1+\alpha) \rho^\alpha(\pi,\tN),
\end{cases}
\end{align}
where $c_1,c_2>0$ are two constants independent of $S$. Combining both cases gives
\begin{align}\label{eq:cmultiidentity}
    d_\Delta(S,\bayesS) \stackrel{\text{def}}{=}\mathbb{P}_{\mX}(S\Delta\bayesS)\lesssim \left[d_\pi(S,\bayesS)\right]^{\alpha \over 1+\alpha}+\frac{1}{\rho(\pi,\tN)}d_\pi(S,\bayesS),
\end{align}
where we have absorbed the constants into the relationship $\lesssim$. 

\item [Case 2:] $\alpha = \infty$.

The inequality~\eqref{eq:tail2} now becomes
\begin{equation}\label{eq:infty}
d_\pi(S,\bayesS)\geq t\mathbb{P}_{\mX}(S\Delta \bayesS) = td_\Delta(S,\bayesS), \quad \text{for all }0\leq t < \rho(\pi,\tN).
\end{equation}
The conclusion \eqref{eq:cmultiidentity} follows by taking $t={\rho(\pi,\tN)\over 2}$ in the inequality~\eqref{eq:infty}. 
\end{enumerate}
\end{proof}

\begin{rmk}[Bounding $L_1$ distance by classification risk]\label{eq:rmk} The bound controls the $L_1$ distance to $\bayespif=\sign(f-\pi)$ using the classification excess risk to $\risk(\bayespif)$. The result applies uniformly to $\pi\in[-1,1]$ if $f$ is globally-$\alpha$ smooth; i.e., the bound
\begin{equation}\label{eq:L1}
\onenormSize{}{\sign \phi-\bayespif} \lesssim \left[\risk(\phi)-\risk(\bayespif)\right]^{\alpha \over 1+\alpha}+{1\over \rho(\pi, \tN)}\left[\risk(\phi)-\risk(\bayespif)\right]
\end{equation}
holds for all functions $\phi\colon \tX\to\mathbb{R}$ and for all $\pi\in[-1,1]$ except for a finite number of points. 
In fact, the similar inequality holds by replacing the 0-1 risk to hinge risk or $T$-truncated hinge risk. Specifically, the following bound holds for all functions $\phi\colon \tX\to \mathbb{R}$ and all $\pi\in[-1,1]$ except for a finite number of points.
\begin{itemize}
\item For hinge loss $F(z)=(1-z)_{+}$,
\[
\onenormSize{}{\phi-\bayespif} \lesssim \left[\riskF(\phi)-\riskF(\bayespif)\right]^{\alpha \over 1+\alpha}+{1\over \rho(\pi, \tN)}\left[\riskF(\phi)-\riskF(\bayespif)\right].
\]
\item For $T$-truncated hinge loss $F(z)=\min((1-z)_{+},T)$ with $T\geq 2$,
\[
\onenormSize{}{\phi^T-\bayespif} \lesssim \left[\riskF(\phi)-\riskF(\bayespif)\right]^{\alpha \over 1+\alpha}+{1\over \rho(\pi, \tN)}\left[\riskF(\phi)-\riskF(\bayespif)\right],
\]
where $\phi^T$ is a truncation of function $\phi$; see formal definition in~\eqref{eq:Tphi}.
\end{itemize}
See Lemma~\ref{lem:hingeL1} for proofs. 
\end{rmk}

\subsection{Proofs of Theorem~\ref{thm:main} and Part (a) in Theorems~\ref{thm:extension}}\label{sec:sign}
We provide a unified framework that incorporates Theorem~\ref{thm:main}, Part (a) in Theorems~\ref{thm:extension} in the main paper. In addition, part of the proof in Theorem~\ref{thm:sparse} is given with the same framework.  For any given $\pi\in[-1,1]$, write $\shift=Y-\pi$, and let $\ell_{\pi,F}(\phi; (\mX, Y))$ denote the weighted $F$-loss
\[
\ell_{\pi,F}(\phi; (\mX, Y))\stackrel{\text{def}}{=}|\shift| F\big(\phi(\mX)\sign(Y-\pi)\big),
\]
where the loss function $F$ could be either standard 0-1 loss $F(z)=\mathds{1}(z>0)$ or surrogate loss satisfying Assumption~\ref{ass:main}. Assume $\mathbb{P}(\FnormSize{}{\mX}\leq 1)=1$. Consider the large-margin estimate
\begin{align}
\hat \phi_{\pi,F}=\argmin_{\phi \in\Phi(r,s_1,s_2)}\left\{ {1\over n}\sum_{i=1}^n \ell_{\pi,F}(\phi; (\mX_i,Y_i))+ \lambda \FnormSize{}{\phi}^2\right\},
\end{align}
where the trace function family
\[
\Phi(r,s_1,s_2)=\{\phi\colon \mX\mapsto \langle \mX, \mB \rangle +b \ \big| \text{rank}(\mB)\leq r,\  \text{supp}(\mB)\leq (s_1,s_2), |b|\leq \FnormSize{}{\mB}+1\}
\]
is the search domain. Notice that we have imposed the additional constraint $|b|\leq \FnormSize{}{\mB}+1$ without altering the estimation; see Section~\ref{sec:intercept}.

The following theorem states the accuracy for sign function estimate $\sign \hat \phi_{\pi,F}\colon \tX\to\{-1,1\}$. 
 \begin{thm}[Sign estimation]~\label{thm:unified} Fix $\pi\notin\tN$. Suppose the regression function $f\in\caliF(r,s_1,s_2)$ is $(\pi,\alpha)$-smooth over $\tX$. Then, with high probability at least $1-\exp(-nt_n)$ over training data $(\mX_i,Y_i)_{i\in[n]}$, the estimate~\eqref{eq:phi} satisfies
\begin{equation}\label{eq:unified_sign}
\onenormSize{}{\sign \hat \phi_{\pi,F}-\bayespif}\lesssim t_n^{\alpha/( 2+\alpha)}+{1\over \rho^2(\pi,\tN)}t_n,
\end{equation}
under the following three specifications:
\begin{enumerate}[wide, labelwidth=!, labelindent=0pt]
\item[(a)] (Theorem~\ref{thm:main}) 0-1 loss $F(z)=\mathds{1}(z>0)$, no penalization $\lambda=0$, $(s_1,s_2)=(d_1,d_2)$, and $t_n={1\over n}rd_{\max}$;
\item[(b)] (Theorem~\ref{thm:sparse}) 0-1 loss $F(z)=\mathds{1}(z>0)$, no penalization $\lambda=0$, constant $(s_1,s_2)$, and $t_n={1\over n}r(s_1+s_2) \log d_{\max}$;
\item[(c)] (Theorem~\ref{thm:extension}) Surrogate loss satisfying Assumption~\ref{ass:main}, constant $(s_1,s_2)$, $t_n={1\over n} r(s_1+s_2) \log d_{\max}$, penalization $\lambda\asymp t^{(\alpha+1)/(\alpha+2)}_n+t_n/\rho(\pi,\tN)$, approximation error $a^{(\alpha+1)/(\alpha+2)}_n \leq t_n$.
\end{enumerate}
Here, the constants suppressed in the $\lesssim$ of~\eqref{eq:unified_sign} are independent of $\pi$. 
\end{thm}

\begin{rmk}[One-sided tail]\label{rmk:lt}
Inspection of the proof shows that the conclusion~\eqref{eq:unified_sign} holds for all $t\geq t_n$. That is, for all $t\geq t_n$, with high probability at least $1-\exp(-nt)$, we have
\begin{equation}
\onenormSize{}{\sign \hat \phi_{\pi,F}-\bayespif}\lesssim t^{\alpha/( 2+\alpha)}+{1\over \rho^2(\pi,\tN)}t.
\end{equation}
\end{rmk}

\begin{rmk}[Ridge penalization]
The estimation under 0-1 loss requires no penalization, because only the sign, but not the magnitude, of $\phi$ affects the 0-1 risk. One can constrain $\FnormSize{}{\phi}=1$ in the empirical 0-1 risk minimization without altering the solution. In contrast, the surrogate loss such as as hinge loss is scale-sensitive, rending the possible unboundedness of $\phi$. We impose penalization to control the magnitude of the $\FnormSize{}{\phi}$ and thus the local complexity. The resulting estimation enjoys the fast convergence as in sieve estimate~\citep{shen1994convergence} under well tuned $\lambda$.   
\end{rmk}

We provide the proof after introducing two main lemmas.
There are two key ingredients in the proof. The first step is to quantify the convergence of $\hat \phi_{\pi,F}$'s excess $F$-risk using Lemmas~\ref{lem:prepare} and~\ref{lem:risk}. The second step is to relate the excess $F$-risk to excess 0-1 risk using Lemma~\ref{lem:prepare}, and then establish the sign function accuracy using Theorem~\ref{thm:identifiability}.

Recall that $\hat \phi_{\pi,F}$ is the minimizer of empirical $F$-risk.  To quantify the $\hat \phi_{\pi,F}$'s excess $F$-risk, we notice that 
\begin{align}
&\riskF(\hat \phi_{\pi,F})-\inf_{\text{all }{\phi}}\riskF(\phi)\\
=&
 \KeepStyleUnderBrace{\riskF(\hat \phi_{\pi,F})-\inf_{\phi\in\Phi(r,s_1,s_2)}\riskF(\phi)}_{\text{estimation error}}+\KeepStyleUnderBrace{
 \inf_{\phi\in\Phi(r,s_1,s_2)}\riskF(\phi)-\inf_{\text{all }\phi}\riskF(\phi)}_{\text{approximation error}},
 \end{align}
The simplest way to bound $\hat \phi_{\pi,F}$'s excess risk is to use a uniform convergence of excess risk over classifiers $\Phi(r,s_1,s_2)$; however, this approach ignores the local complexity around $\hat \phi_{\pi,F}$ and yields a suboptimal rate. Here we adopt the local iterative techniques of~\citet[Theorem 3]{wang2008probability} to obtain a better rate. The improvement stems from the fact that, under considered assumptions, the variance of the excess loss is bounded in terms of its expectation. Because the variance decreases as we approach the optimal $\phi^*_{\pi}:=\argmin_{\phi\in\Phi(r,s_1,s_2)}\riskF(\phi)$, the risk of the empirical minimizer converges more quickly to the optimal risk than the simple uniform converge results would suggest. 

The following result summarizes the key properties of four common losses: 0-1 loss, hinge loss, $T$-truncated hinge loss, and psi-loss. Here, the $T$-truncated hinge loss is defined as $F(z)=\min((1-z)_{+},T)$ for a given $T\geq 2$. We will use $T$-truncated hinge loss to facilitate the proofs of Lemma~\ref{lem:risk} and Theorem~\ref{thm:unified}. 

\begin{lem}[Conversion inequalities]\label{lem:prepare} Suppose the regression function $f$ is $(\pi,\alpha)$-smooth, and denote $\bayespif=\sign(f-\pi)$ for $\pi\in[-1,1]$. Let $F$ be 0-1 loss, hinge loss, $T$-truncated hinge loss, or psi-loss. Then, the following three properties hold for all $\pi\in[-1,1]$.
\begin{itemize}[label={2.\arabic*},wide, labelwidth=!, labelindent=0pt]
\item[(a)] Optimality: $\inf_{\text{all }\phi}\riskF(\phi)=\riskF(\bayespif)$.
\item[(b)] Excess risk bound: for all classifers $\phi\colon \tX\to\mathbb{R}$,
\begin{align}\label{eq:b1}
\risk(\phi)-\risk(\bayespif) \leq  C\left[\riskF(\phi)-\riskF(\bayespif)\right],
\end{align}
where $C=1$ for 0-1, hinge loss or $T$-truncated loss, and $C=1/2$ for psi-loss. 
\item[(c)] Variance-to-mean relationship: Suppose $F$ is 0-1 loss, $T$-truncated loss, or psi-loss. Then, for all classifiers $\phi\colon \tX\to\mathbb{R}$, 
\begin{align}\label{eq:b2}
&\textup{Var}\left[\ell_{\pi,F}(\phi; (\mX,Y))-\ell_{\pi,F}(\bayespif; (\mX,Y))\right]  \notag \\
\lesssim &\ 
\left[\riskF(\phi)-\riskF(\bayespif)\right]^{\alpha/(1+\alpha)}+ {1\over \rho(\pi, \tN)}\left[\riskF(\phi)-\riskF(\bayespif)\right].
 \end{align}
 \end{itemize}
\end{lem}
\begin{rmk}
The property (c) holds for bounded loss functions only, i.e, excluding hinge loss. 
\end{rmk}

Below we establish the estimation convergence rate for $\hat \phi_{\pi,F}$'s excess F-risk. The variance-to-mean relationship in Lemma~\ref{lem:prepare} plays a key role in determining the convergence rate based on \citet[Theorem 3]{shen1994convergence}; also see Theorem~\ref{thm:refer} in Section~\ref{sec:auxiliary}. Our proof of Lemma~\ref{lem:risk} adopts the local iterative techniques from~\citet[Theorem 3]{wang2008probability}. Similar techniques have been used in \citet[Theorem 4]{bartlett2006convexity} for similar estimate but without ridge penalization.

\begin{lem}[Classification risk error]\label{lem:risk}Consider the set-up as in Theorem~\ref{thm:unified}. Then, with high probability and $t_n$ specified in Theorem~\ref{thm:unified}, the following holds for all $\pi\notin\tN$. 
\begin{enumerate}[label={2.\arabic*},wide, labelwidth=!, labelindent=0pt]
\item[(a)] If $F$ is 0-1 loss or psi-loss, then
\[
\risk(\hat \phi_{\pi,F})-\risk(\bayespif)\lesssim \riskF(\hat \phi_{\pi,F})-\riskF(\bayespif) \lesssim t_n^{(\alpha+1)/(\alpha+2)}+ {1\over \rho(\pi,\tN)}t_n.
\]
\item[(b)] If $F$ is hinge loss, then
\[
\risk(\hat \phi_{\pi,F})-\risk(\bayespif)\lesssim \textup{Risk}_{F'}(\hat \phi_{\pi,F})-\textup{Risk}_{F'}(\bayespif) \lesssim t_n^{(\alpha+1)/(\alpha+2)}+ {1\over \rho(\pi,\tN)}t_n,
\]
where $\textup{Risk}_{F'}(\phi):=\mathbb{E}\left[|\shift|F'( \phi(\mX)\sign\shift)\right]$ denotes the risk evaluated under $T$-truncated hinge loss $F'=\min(T,\ (1-z)_{+})$, and $T=\max(2,J) \geq \max(2, \FnormSize{}{\phi_\pi^{(n)}})$ is a constant based on Assumption~\ref{ass:main}(a). 
\end{enumerate}
\end{lem}

\begin{proof}[Proof of Theorem~\ref{thm:unified}]
Write $\rho=\rho(\pi,\tN)$. Combining Theorem~\ref{thm:identifiability} and Lemma~\ref{lem:risk} gives
\begin{align}
\onenormSize{}{\sign \hat\phi_{\pi,F}-\bayespif} &\lesssim \left[\risk(\hat \phi_{\pi,F})-\risk(\bayespif)\right]^{\alpha/(\alpha+1)}+{1\over \rho}\left[\risk(\hat \phi_{F,\pi})-\risk(\bayespif)\right]\\
&\lesssim t_n^{\alpha/(\alpha+2)}+{1\over \rho^{\alpha/\alpha+1}}t_n^{\alpha/(\alpha+1)}+{1\over \rho}t_n^{(\alpha+1)/(\alpha+2)}+{1\over \rho^2}t_n\notag \\
&\leq 4t_n^{\alpha/(\alpha+2)}+{4\over \rho^2}t_n,
\end{align}
where the last line follows from the fact that $a(b^2+b^{(\alpha+2)/(\alpha+1)}+b+1) \leq 4 a (b^2+1)$ with $a=\rho^{-2}t_n$ and $b=\rho t_n^{-1/(\alpha+2)}$. The proof is complete by specializing $t_n$ in each context. 
\end{proof}

We now provide the proofs for the two key Lemmas~\ref{lem:prepare} and~\ref{lem:risk}.

\begin{proof}[Proof of Lemma~\ref{lem:prepare}]
\begin{enumerate}[label={2.\arabic*},wide, labelwidth=!, labelindent=0pt]

\item[Case 1:] $F(z)=\mathds{1}(z<0)$ is 0-1 loss. 

Properties (a) and (b) directly follow from Theorem~\ref{thm:oracle}. To prove (c), we expand the variance by
\begin{align}\label{eq:mae}
\textup{Var}\left[\ell_{\pi}(\phi;(\mX,Y))-\ell_{\pi}(\bayespif,(\mX,Y)\right] &\lesssim \mathbb{E}|\ell_{\pi}(\phi;(\mX,Y))-\ell_{\pi}(\bayespif,(\mX,Y)|^2\notag \\
&\lesssim \mathbb{E}|\ell_{\pi}(\phi;(\mX,Y))-\ell_{\pi}(\bayespif,(\mX,Y)|\notag \\
& \lesssim \mathbb{E}\left||\sign \shift - \sign \phi(\mX)|-|\sign \shift -\bayespif(\mX)|\right|\notag\\
&\leq \mathbb{E}|\sign \phi-\bayespif|,
\end{align}
where the second line comes from the boundedness of 0-1 loss, and the third line comes from the boundedness of weight $|\shift|$, and fourth line comes from the inequality $||a-b|-|c-b||\leq |a-b|$ for $a,b,c\in\{-1,1\}$. Here we have absorbed the constant multipliers in $\lesssim$. Therefore, the conclusion~\eqref{eq:b2} then directly follows by applying Remark~\ref{eq:rmk} to~\eqref{eq:mae}. 

\item[Case 2:] $F(z)=(1-z)_{+}$ is hinge loss. 

Property (a) was firstly introduced in~\citet[Lemma 1]{wang2008probability}, and here we provide an alternative proof. 

A direct calculation (see\ Lemma~\ref{lem:hingeL1}) shows that
\[
\riskF(\phi)-\riskF(\bayespif)\geq \mathbb{E}|\phi-\bayespif||f-\pi|\geq0,
\]
Therefore, $\inf_{\text{all }\phi}\riskF(\phi)=\riskF(\bayespif)$. Property~\eqref{eq:b1} is from \citet[Corollary 1]{scott2011surrogate} (see also Theorem~\ref{thm:scott} in Section~\ref{sec:auxiliary}).

\item[Case 3:] When $F(z) = 2\min(1,(1-z)_+)$ is psi-loss. 

Again, the property (a) follows from~\citet[Lemma 1]{wang2008probability}. For the property~\eqref{eq:b1}, we use Theorem~\ref{thm:scott} to find the transformation function $\psi$ that relates 0-1 risk to F-risk:
\[
\psi(\risk(\phi)-\risk(\bayespif))\leq \riskF(\phi)-\riskF(\bayespif). 
\]
To put our problem in the context of Theorem~\ref{thm:scott}, we need additional notation. For any function measurable $g\colon x\mapsto g(x)$, we write $g=g^{+}-g^{-1}$, where $g^{+}$ and $g^{-}$ are two non-negative functions given by
\begin{align}
g^{+}(x)=\max\{ g(x),0 \} =
\begin{cases}
g(x), & \text{if }g(x)>0,\\
0, & \text{otherwise},\\
\end{cases}\quad 
g^{-}(x)&=\max\{ -g(x),0 \} =
\begin{cases}
-g(x), & \text{if }g(x)<0,\\
0, & \text{otherwise}.
\end{cases}
\end{align}
Under this notation, we have $|g|=g^{+}+g^{-1}$. 

Define the conditional $F$-risk
\[
C_{\pi,F}(\mX,t):=F(t)\mathbb{E}_{Y|\mX}(Y-\pi)^{+}+F(-t)\mathbb{E}_{Y|\mX}(Y-\pi)^{-}.
\]
A direct calculation shows that
\[
C_{\pi,F}(\mX,t)=
\begin{cases}
2\mathbb{E}_{Y|\mX}(Y-\pi)^{-}, & \text{if }t\geq 1,\\
2\mathbb{E}_{Y|\mX}|Y-\pi|-2t\mathbb{E}_{Y|\mX}(Y-\pi)^{+},  &\text{if }t\in[0,1),\\
2\mathbb{E}_{Y|\mX}|Y-\pi|+2t\mathbb{E}_{Y|\mX}(Y-\pi)^{-},  &\text{if }t\in[-1,0),\\
2\mathbb{E}_{Y|\mX}(Y-\pi)^{+}, & \text{if }t<-1.
\end{cases}
\]
Therefore, following the notation of Theorem~\ref{thm:scott}, we have
\begin{align}
H_{\pi,F}(\mX)&:=\inf_{t\in\mathbb{R}\colon t(f(\mX)-\pi)\leq 0}C_{\pi,F}(\mX, t)-\inf_{t\in\mathbb{R}} C_{\pi,F}(\mX, t) = 2|f(\mX)-\pi|.
\end{align}
Applying Theorem~\ref{thm:scott} to the above setup gives the excess risk transformation rule: $\psi: z\to2|z|$. Therefore, the property~\eqref{eq:b1} is proved. 

To prove~\eqref{eq:b2}, notice that 
\begin{align}\label{eq:excess-hinge}
&\text{Var}\left\{|\shift|\left[F(\phi(\mX)\sign \shift)-F(\bayespif(\mX)\sign \shift)\right]\right\} \notag \\
\lesssim &\  \mathbb{E}|\shift||F(\phi(\mX)\sign\shift)-F(\bayespif(\mX)\sign \shift)|\notag \\
 \lesssim &\  \KeepStyleUnderBrace{\mathbb{E}\left| 1-\sign(\phi(\mX)\shift)- F(\bayespif(\mX)\sign \shift)\right
|}_{=:\text{(i)}} +\KeepStyleUnderBrace{\mathbb{E}|\shift|\left|F( \phi(\mX)\sign \shift)- \left(1-\sign(\phi(\mX)\shift)\right) \right|}_{=:\text{(ii)}}.
\end{align}
The first term (i) is bounded as follows
\begin{align}
\text{(i)}=\mathbb{E}\left|\sign (\phi(\mX)\shift)-\sign (\bayespif(\mX)\shift)\right|&\lesssim d_\Delta(S_\phi,\bayesS(\pi))\\&\lesssim  d^\alpha_{\pi}(S_\phi,\bayesS(\pi))+{1\over \rho(\pi, \tN)}d_{\pi}(S_\phi,\bayesS(\pi)),
\end{align}
where the first line uses the fact that $F(1)=0$ and $F(-1)=2$, and last inequality is from Theorem~\ref{thm:identifiability}. Here we define indicator set corresponding $\phi$ as $S_\phi = \{\mX\in\tX\colon \phi(\mX)\geq 0\}$.
The second term (ii) is bounded as follows
\begin{align}
    \text{(ii)}
    &=\mathbb{E}\left[ |\shift| F(\phi(\mX)\sign \shift)- |\shift|\left(1-\sign(\phi(\mX)\shift)\right)\right] \\
    &= \mathbb{E}\left[|\shift|F(\phi(\mX)\sign \shift)-|\shift|F(\bayespif(\mX)\sign \shift)\right]\\
    & \quad + \mathbb{E}\left[|\shift|(1-\sign(\bayespif\shift))-|\shift|(1-\sign(\phi(\mX)\shift))\right]\\
    &\leq [\riskF(\phi)-\riskF(\bayespif)]+d_\pi(S_\phi,\bayesS(\pi)),
\end{align}
where the first equality is based on $F(z) = 1-\sign(z)$ if $z = 1$ or $-1$, and the last inequality is from definition of $d_\pi(\cdot,\cdot)$.  Notice we have  $d_\pi(S_\phi,\bayesS(\pi)) = \risk(\phi)-\risk(\bayespif)$ by definition. Therefore, the proof is complete by combining \eqref{eq:excess-hinge}, \eqref{eq:b1} and bounds (i)-(ii).

\item[Case 4:] $F(z)=\min((1-z)_+,T)$ for $T$-truncated hinge loss, for given $T\geq 2$. 
A direct calculation (c.f. Remark~\ref{rmk:truncate} after Lemma~\ref{lem:hingeL1}) shows that
\[
\riskF(\phi)-\riskF(\bayespif)\geq \mathbb{E}|\phi^T -\bayespif||f-\pi|\geq0, 
\]
where $\phi^T\colon \tX\to[-(T-1),\ (T-1)]$ denotes the $(T-1)$-truncation of $\phi$,
\begin{equation}\label{eq:Tphi}
\phi^T=
\begin{cases}
T-1 & \text{if }\phi>T-1,\\
\phi, & \text{if }|\phi|\leq T-1,\\
-(T-1), & \text{if }\phi<-(T-1).
\end{cases}
\end{equation}
Therefore, $\inf_{\text{all }\phi}\riskF(\phi)=\riskF(\bayespif)$. To show property~\eqref{eq:b1}, we again use Theorem~\ref{thm:scott} to find the transformation function $\psi$ that relates 0-1 risk to F-risk:
\[
\psi(\risk(\phi)-\risk(\bayespif))\leq \riskF(\phi)-\riskF(\bayespif). 
\]
Using similar arguments as in Case 3, we obtain the conditional $F$-risk
\[
C_{\pi,F}(\mX,t)=
\begin{cases}
\min\left\{T,\ (1+t)\mathbb{E}_{Y|\mX}(Y-\pi)^{-}\right\}, & \text{if } t\geq 1,\\
\mathbb{E}_{Y|\mX}|Y-\pi|-t(f(\mX)-\pi), & \text{if } t\in[0,1),\\
\mathbb{E}_{Y|\mX}|Y-\pi|+t(f(\mX)-\pi), & \text{if } t\in[-1,0),\\
\min\left\{T,\ (1-t)\mathbb{E}_{Y|\mX}(Y-\pi)^{+}\right\}, & \text{if } t<-1.\\
\end{cases}
\]
Therefore, following the notation of Theorem~\ref{thm:scott}, we have
\begin{align}
H_{\pi,F}(\mX)&:=\inf_{t\in\mathbb{R}\colon t(f(\mX)-\pi)\leq 0}C_{\pi,F}(\mX, t)-\inf_{t\in\mathbb{R}} C_{\pi,F}(\mX, t) = |f(\mX)-\pi|.
\end{align}
Applying Theorem~\ref{thm:scott} to the above setup gives the excess risk transformation rule: $\psi: z\to|z|$. Therefore, the property~\eqref{eq:b1} is proved. 

To prove~\eqref{eq:b2}, we use Lemma~\ref{lem:hingeL1} and the boundedness condition of $\norm{F}_\infty\leq T$. Specifically, we bound the variance using the $L$-1 distance between $\phi$ and $\bayespif$, 
\begin{align}
&\text{Var}\left\{|\shift|\left[F(\phi(\mX)\sign \shift)-F(\bayespif(\mX)\sign \shift\right]\right\}\\
\leq&\ 4 \mathbb{E}|F(\phi(\mX)\sign \shift)-F(\bayespif(\mX)\sign \shift)|^2\\
\lesssim &\ T\mathbb{E}|F(\phi(\mX)\sign \shift)-F(\bayespif(\mX)\sign \shift)|\\
\lesssim &\ T\mathbb{E}|\phi^T-\bayespif|,
\end{align}
where $T>0$ is the upper bound of truncated hinge loss, the first inequality comes from the boundedness of $|\shift|$, the second inequality comes from the boundedness of the $T$-truncated hinge loss, and the last line comes from the definition of $F$. Applying Remark~\ref{rmk:truncate} in Lemma~\ref{lem:hingeL1} to the last inequality complete the proof.
\end{enumerate}
\end{proof}

\begin{proof}[Proof of Lemma~\ref{lem:risk}]

Fix $\pi\notin\tN$, and write $\rho=\rho(\pi,\tN)$, $L_n=t_n^{(\alpha+1)/(\alpha+2)}$. 
We first consider the (bounded) psi-loss, and then consider the (unbounded) hinge loss. The 0-1 loss incurs only slight difference in the proof, and we address this case at last. 

\begin{enumerate}[label={2.\arabic*},wide, labelwidth=!, labelindent=0pt]
\item[Case 1:] psi-loss, $\lambda\asymp L_n+t_n/\rho$, and $ a_n\lesssim L_n$.

For any function $\phi\in\Phi(r,s_1,s_2)$ of consideration, define the empirical weighted $F$-risk
\begin{equation}\label{eq:F}
\eriskF(\phi) = \frac{1}{n}\sum_{i=1}^n\ell_{\pi,F}(\phi; (\mX_i,Y_i)).
\end{equation}

Under the notation, our estimate $\hat \phi_{\pi,F}$ is the minimizer of the regularized empirical $F$-risk,
\begin{equation}\label{eq:def2}
\hat \phi_{\pi,F}=\argmin_{\phi\in\Phi(r,s_1,s_2)}\Big\{ \eriskF(\phi)+\lambda\FnormSize{}{\phi}^2\Big\}.
\end{equation}
We are interested in the convergence rate of $\hat \phi_{\pi,F}$'s excess risk, $\riskF(\hat\phi_{\pi,F})-\riskF(\bayespif)$. Let $L_n\asymp t_n^{(\alpha+1)/(\alpha+2)}$ denote the desired convergence rate to seek. By the definition of $\hat \phi_{\pi,F}$, we have
\[
 \eriskF(\hat\phi_{\pi,F})+\lambda\FnormSize{}{\hat \phi_{\pi,F}}^2\leq \eriskF(\phi_{\pi}^{(n)})+\lambda J^2,
\]
where $\phi_{\pi}^{(n)}$ is a sequence of functions in Assumption~\ref{ass:main}(a). Therefore, we have the following inclusion of probability events,
\begin{align}\label{eq:outer}
&\left\{(\mX_i,Y_i)_{i\in[n]}\colon \riskF(\hat\phi_{\pi,F})-\riskF(\bayespif)\geq 2L_n \right\}\notag \\
 \subset &
 \bigg\{(\mX_i,Y_i)_{i\in[n]}\colon \exists \phi\in\Phi(r,s_1,s_2),\ \text{s.t.}\  \riskF(\phi; (\mX,Y))-\riskF(\bayespif)\geq 2L_n, \notag \\&\hspace*{6.3cm} \text{and}\  \eriskF(\phi)+\lambda\FnormSize{}{\phi}^2\leq \eriskF(\phi_{\pi}^{(n)})+\lambda J^2 \bigg\}\notag \\
 \subset &
\left\{(\mX_i,Y_i)_{i\in[n]} \colon \sup_{\substack{\phi\in\Phi(r,s_1,s_2)\\  
\riskF(\phi; (\mX,Y))-\riskF(\bayespif)\geq 2L_n  }}\left[\eriskF(\phi_\pi^{(n)})+\lambda J^2-\eriskF(\phi)-\lambda\FnormSize{}{\phi}^2\right]\geq 0\right\} \notag \\
\subset & \bigcup_{\phi\in A_{s,k}}\left\{(\mX_i,Y_i)_{i\in[n]}\colon \sup_{\phi\in A_{s,k}} \left[\eriskF(\phi_\pi^{(n)})+\lambda J^2-\eriskF(\phi)- \lambda\FnormSize{}{\phi}^2\right] \geq 0 \right \}.
\end{align}
In the last line of~\eqref{eq:outer}, we have partitioned the set $\{\phi\in\Phi(r,s_1,s_2)\colon \riskF(\phi; (\mX,Y))-\riskF(\bayespif)\geq 2L_n\}$ into a union of $A_{s,k}$, with 
\begin{align}
A_{s,k}&=\{\phi\in \Phi(r,s_1,s_2)\colon (s+1)L_n\leq \riskF(\phi)-\riskF(\bayespif)< (s+2) L_n, (k-1) J^2 \leq \FnormSize{}{\phi}^2< k J^2\},
\end{align}
for $s,k=1,2,\ldots$.

Let $\Gamma$ denote the target probability for the first line in~\eqref{eq:outer}. To bound $\Gamma$, it suffices to bound the sum of probabilities over sets $A_{s,k}$.  For  each $ A_{s,k}$, we consider the centered empirical process,
\begin{align}\label{eq:empro}
v_n(\phi)&:=\left[\eriskF(\phi_\pi^{(n)})-\eriskF(\phi)\right]-\left[\riskF(\phi_\pi^{(n)})-\riskF(\phi)\right]\notag \\
&=\frac{1}{n}\sum_{i\in[n]}\left\{\ell_{\pi,F}(\phi_\pi^{(n)};(\mX_i,Y_i))-\ell_{\pi,F}(\phi;(\mX_i,Y_i))-\mathbb{E}\left[\ell_{\pi,F}(\phi_\pi^{(n)};(\mX_i,Y_i))-\ell_{\pi,F}(\phi;(\mX_i,Y_i))\right]\right\}.
\end{align}
Notice that 
\begin{align}\label{eq:first}
\riskF(\phi)-\riskF(\phi_\pi^{(n)}) &= \riskF(\phi)-\riskF(\bayespif)+\riskF(\bayespif)-\riskF(\phi_\pi^{(n)})\nonumber\\&\geq (s+1)L_n -a_n\nonumber
\\&\geq  sL_n,
\end{align}
where the first inequality is from the fact that $\phi\in A_{s,k}$ and Assumption~\ref{ass:main}(a), and the last inequality uses the condition that $a_n\lesssim L_n$.

Combining the definition of $v_n$ in \eqref{eq:empro} and inequality \eqref{eq:first} gives \eqref{eq:outer} as
\begin{align}\label{eq:union}
\Gamma&\leq \sum_{s,k=1}^\infty\mathbb{P}\left\{\sup_{\phi\in A_{s,k}}  \left[v_n(\phi) -\lambda \FnormSize{}{\phi}^2 \right]\geq sL_n -\lambda J^2\right\}\nonumber\\&\leq \sum_{s,k=1}^\infty\mathbb{P}\left\{\sup_{\phi\in A_{s,k}}  v_n(\phi)\geq sL_n + \lambda(k-2) J^2=: M(s,k)\right\},
\end{align}
where  $M(s,k)>0$ for all $s,k\in\mathbb{N}$ from the condition $\lambda J^2\leq L_n/2$ by the choice of $(L_n,\lambda)$. Verification of this condition is deferred to when we specify $(\lambda,L_n)$ in \eqref{eq:delta}. 

The variance of the empirical process is bounded by
\begin{align}\label{eq:second}
&\sup_{\phi\in A_{s,k}}\textup{Var}\left[\ell_{\pi,F}(\phi_\pi^{(n)};(\mX,Y))-\ell_{\pi,F}(\phi;(\mX,Y)\right]\notag \\
\leq&\sup_{\phi\in A_{s,k}}2\bigg\{ \textup{Var}\left[\ell_{\pi,F}(\phi_\pi^{(n)};(\mX,Y))-\ell_{\pi,F}(\bayespif;(\mX,Y)\right]\notag \\
& \hspace{1.5cm}+\textup{Var}\left[\ell_{\pi,F}(\phi;(\mX,Y))-\ell_{\pi,F}(\bayespif;(\mX,Y)\right]\bigg\}\notag \\
\lesssim &\ [M(s,k)]^{\alpha/(1+\alpha)}+{M(s,k)\over \rho}=:V(s,k),
\end{align}
where the last inequality is from Lemma~\ref{lem:prepare}.

We next bound the right-hand-side of~\eqref{eq:union} by choosing $(L_n,\lambda)$ that satisfies the conditions in Theorem~\ref{thm:refer}. (The specification of $(L_n,\lambda)$ is deferred to the next paragraph). Once such $(L_n,\lambda)$ is chosen, then it follows from Theorem~\ref{thm:refer} that
\begin{align}\label{eq:gamma}
\Gamma& \lesssim \sum_{s,k}\exp\left( - {nM^2(s,k)\over V(s,k)+2M(s,k)}\right)\notag \\
&\lesssim \sum_{s,k}\exp(-\rho nM(s,k)) = \sum_{s,k}\exp\left(-n\rho sL_n-n\rho \lambda(k-2)J^2 \right)\notag\\
&\leq \left({e^{-n\rho L_n}\over 1-e^{-n\rho L_n} }\right) \left({e^{n\rho \lambda J^2} \over 1-e^{-n\rho \lambda J^2}}\right)\notag \\ 
&\leq {e^{-n\rho L_n/2}\over (1-e^{-n\rho L_n})(1-e^{-n\rho\lambda J^2})},
\end{align}
where the first line uses the boundedness of psi-loss, and the last inequality is from the condition $\lambda J^2\leq L_n/2$ by the choice of $(\lambda,L_n)$.

Now, we specify $(L_n,\lambda)$ that satisfies the condition of Theorem~\ref{thm:refer}. The pair $(L_n,\lambda)$ is determined by the solution to the following inequality,
\begin{equation}\label{eq:equation}
\sup_{k\geq 1, s\geq 1}{1\over x}\int_{x}^{\sqrt{x^{\alpha/(\alpha+1)}+x/\rho}}\sqrt{\tH_{[\ ]}(\varepsilon, \Phi^{k},\vnormSize{}{\cdot})}d\varepsilon \lesssim n^{1/2}, \quad \text{where }x=sL_n+\lambda (k-2)J^2.
\end{equation}
In particular, the smallest $L_n$ satisfying~\eqref{eq:equation} yields the best upper bound of the error rate. Here $\tH_{[\ ]}(\varepsilon, \Phi^{k}, \vnormSize{}{\cdot})$ denotes the $L_2$-norm, $\varepsilon$-bracketing number (c.f. Definition~\ref{pro:inftynorm}) for function family $\Phi^{k}$, and, we have denoted $\Phi^{k}=\{\phi\in\Phi(r,s_1,s_2)\colon \FnormSize{}{\phi}^2\leq k\}$, i.e., the subset of functions in $\Phi(r,s_1,s_2)$ with magnitudes bounded by $k$, for $k\geq 1$.

It remains to solve for the smallest possible $L_n$ in~\eqref{eq:equation}. Based on Lemma~\ref{lem:metric}, the inequality~\eqref{eq:equation} is satisfied with the choice
\begin{equation}\label{eq:delta}
L_n\asymp  t_n^{(\alpha+1)/(\alpha+2)}+{t_n\over \rho}, \quad\text{and} \quad \lambda = {L_n \over 2J^2},
\end{equation}
where
\begin{equation}\label{eq:tn}
t_n=\begin{cases}
{rd_{\max}\over n}, \quad \text{low-rank model $\phi\in\Phi(r)$,}\\
{r(s_1+s_2)\log d_{\max}\over n}, \quad \text{low-rank and two-way sparse model $\phi\in\Phi(r,s_1,s_2)$.}\\
\end{cases}
\end{equation}
Notice that this choice of $(L_n,\lambda)$ guarantees the conditions for earlier calculation in~\eqref{eq:union} and~\eqref{eq:gamma}. Specifically, we have the assumption $\lambda\asymp t_n^{(\alpha+1)/(\alpha+2)}+t_n/\rho $ from the setup of Theorem~\ref{thm:unified}. Given this $\lambda$, we choose an $L_n$ with a suitable constant factor such that $\lambda J^2\leq L_n/2$. So conditions for earlier calculation in~\eqref{eq:union} and~\eqref{eq:gamma} are verified.

Plugging~\eqref{eq:delta} into~\eqref{eq:gamma} gives that
\begin{align}\label{eq:tail}
\Gamma&=\mathbb{P}\left[\riskF(\hat \phi_{\pi,F}) - \riskF(\bayespif)  \geq L_n \right]\\&\leq{e^{-n\rho L_n/2}\over (1-e^{-n\rho L_n})(1-e^{-n\rho\lambda J^2})}\\
&\lesssim e^{- n\rho L_n}\leq e^{- n t_n},
\end{align}
where the last line uses the fact that $\rho\lambda J^2\asymp \rho L_n \gtrsim t_n\gtrsim n^{-1}$ by \eqref{eq:delta} and~\eqref{eq:tn}. The proof is then complete by bounding the 0-1 risk by $F$-risk. 

\item[Case 2:] hinge loss, $\lambda\asymp L_n+t_n/\rho$, and $ a_n\lesssim L_n$.

For unbounded hinge loss, we seek to bound the $F'$-risk of $\hat \phi_{\pi,F}$, where $F'$ is $T$-truncated version of $F$. The general strategy is to evaluate $\hat \phi_{\pi,F}$'s error using $F'$-risk.  Note that the estimate $\hat \phi_{\pi,F}$~\eqref{eq:def2} is defined under unbounded loss $F$. Therefore, the inclusion~\eqref{eq:outer} changes to
\begin{align}
&\left\{(\mX_i,Y_i)_{i\in[n]}\colon \textup{Risk}_{\pi,F'}(\hat\phi_{\pi,F})-\textup{Risk}_{\pi,F'}(\bayespif)\geq 2L_n \right\}\notag \\
\subset &
 \bigg\{(\mX_i,Y_i)_{i\in[n]}\colon \exists \phi\in\Phi(r,s_1,s_2), \ \text{s.t.}\  \textup{Risk}_{\pi,F'}(\phi; (\mX,Y))-\textup{Risk}_{\pi,F'}(\bayespif)\geq 2L_n, \notag \\&\hspace*{6.2cm} \text{and}\  \eriskF(\phi)+\lambda\FnormSize{}{\phi}^2\leq \eriskF(\phi_{\pi}^{(n)})+\lambda J^2 \bigg\}\notag \\
\subset &
 \bigg\{(\mX_i,Y_i)_{i\in[n]}\colon \exists \phi\in\Phi(r,s_1,s_2),\  \text{s.t.}\  \textup{Risk}_{\pi,F'}(\phi; (\mX,Y))-\textup{Risk}_{\pi,F'}(\bayespif)\geq 2L_n, \notag \\&\hspace*{6.2cm} \text{and}\  
 \widehat{\textup{Risk}}_{\pi,F'}(\phi)+\lambda\FnormSize{}{\phi}^2\leq  \widehat{\textup{Risk}}_{\pi,F'}(\phi_{\pi}^{(n)})+\lambda J^2 \bigg\},
\end{align}
where the last line comes from 
\[
\widehat{\textup{Risk}}_{\pi,F'}(\phi)\leq \widehat{\textup{Risk}}_{\pi,F}(\phi) \text{ for all } \phi\in \Phi(r,s_1,s_2)\quad\text{ and } \quad \widehat{\textup{Risk}}_{\pi,F'}(\phi_\pi^{(n)})= \widehat{\textup{Risk}}_{\pi,F}(\phi_\pi^{(n)}),
\]
because the truncation constant is $T= \max(2,J)>\max(2,\sup_n\FnormSize{}{\phi_\pi^{(n)}})$. Notice that the last line is exactly the same with \eqref{eq:outer} except $F$ being replaced by $F'$. 
The remaining proof follows the same line of argument as in Case 1. In particular, we invoke Lemma~\ref{lem:prepare} to control the variance-to-mean relationship for bounded $F'$-loss in~\eqref{eq:second}. The final conclusion follows from the excess bound inequality for $T$-truncated risk (c.f. Lemma~\ref{lem:prepare}).

\item[Case 3:] 0-1 loss, $\lambda = 0$ and $a_n = 0$.

Under 0-1 loss, only the sign, but not the magnitude, of $\phi$ affects the 0-1 risk. Without loss of generality, we assume  $\FnormSize{}{\phi} \leq 1$. Then, we have the following inclusion of probability events,
\begin{align}
\Gamma :=&\left\{(\mX_i,Y_i)_{i\in[n]}\colon \risk(\hat\phi_{\pi})-\risk(\bayespif)\geq L_n \right\}\notag \\
 \subset &
 \bigg\{(\mX_i,Y_i)_{i\in[n]}\colon \exists \phi\in\Phi(r,s_1,s_2),\ \text{s.t.}\  \risk(\phi; (\mX,Y))-\risk(\bayespif)\geq L_n \notag \\&\hspace*{6.2cm} \text{and}\  \erisk(\phi)\leq \erisk(\bayespif)\bigg\}\notag \\
 \subset &
\left\{(\mX_i,Y_i)_{i\in[n]} \colon \sup_{\substack{\phi\in\Phi(r,s_1,s_2)\\  
\risk(\phi; (\mX,Y))-\risk(\bayespif)\geq L_n  }}\left[\erisk(\bayespif)-\erisk(\phi)\right]\geq0\right\} \notag \\
\subset & \bigcup_{\phi\in A_{s}}\left\{(\mX_i,Y_i)_{i\in[n]}\colon \sup_{\phi\in A_{s}} \left[\erisk(\bayespif)-\erisk(\phi)\right]\geq 0 \right \},
\end{align}
where we have partitioned $\{\phi\in\Phi(r,s_1,s_2)\colon \risk(\phi; (\mX,Y))-\risk(\bayespif)\geq L_n\}$ into a union of $A_s$ with
\begin{align}
A_{s}&=\{\phi\in \Phi(r,s_1,s_2)\colon sL_n\leq \risk(\phi)-\risk(\bayespif)< (s+1) L_n\},
\end{align}
for $s=1,2,\ldots$. 
Similar to Case 1, we consider empirical process, 
\begin{align*}
v_n(\phi):= [\erisk(\bayespif)-\erisk(\phi)]-\risk(\bayespif)-\risk(\phi)].
\end{align*}
Then,  our goal is to bound
\begin{align}\label{eq:01bd}
\Gamma\leq \sum_{s=1}^\infty \mathbb{P}\left\{\sup_{\phi\in A_s} v_n(\phi)\geq sL_n:=M(s)\right\}.
\end{align}
Notice the variance of empirical process is bounded by
\begin{align}
\sup_{\phi\in A_{s}}\textup{Var}\left[\ell_{\pi,F}(\bayespif;(\mX,Y))-\ell_{\pi,F}(\phi;(\mX,Y)\right] \lesssim [M(s)]^{\alpha/(1+\alpha)}+{M(s)\over \rho}=:V(s)
\end{align}
where $F$ is 0-1 loss and the inequality is from Lemma~\ref{lem:prepare}.
Applying Lemma~\ref{lem:metric} with  finite $k = 1$ and $\lambda = 0$  shows that $L_n\asymp t_n^{(\alpha+1)/(\alpha+2)}+ t_n/\rho$ satisfies the conditions Theorem~\ref{thm:refer}, where 
\[
t_n=\begin{cases}
{rd_{\max}\over n}, \quad \text{low-rank model $\phi\in\Phi(r)$,}\\
{r(s_1+s_2)\log d_{\max}\over n}, \quad \text{low-rank and two-way sparse model $\phi\in\Phi(r,s_1,s_2)$.}\\
\end{cases}
\]
Therefore, it follows from Theorem~\ref{thm:refer} and \eqref{eq:01bd} that 
\begin{align}
\Gamma&\lesssim  \sum_{s} \exp\left(-nM^2(s)\over V(s)+M(s)\right)\\&\lesssim \sum_{s}\exp(-\rho snL_n)\\&\leq \left(e^{-n\rho L_n}\over 1-e^{-n\rho L_n}\right)\\&\lesssim e^{-nt_n},
\end{align}
where the last line uses the fact that $\rho L_n\gtrsim t_n \gtrsim{1\over n}$ by our choice of $L_n$ and $t_n$.
\end{enumerate}
\end{proof}

\subsection{Proofs of Theorem~\ref{thm:regression}, Theorem~\ref{thm:sparse}, and Part (b) in Theorem~\ref{thm:extension}}\label{sec:regression}
\begin{proof}[Proof of Theorem~\ref{thm:regression}]
For any  $t\geq t_n$ with $t_n$ specified in Theorem~\ref{thm:unified}, define the event
\[
A=\left\{\onenormSize{}{\sign \hat \phi_\pi- \sign (f-\pi)} \leq t^{\alpha/(2+\alpha)}+{t\over \rho^2(\pi,\tN)} \text{ for all }\pi\in\tH\right\}.
\]
We first show that the event $A$ implies
\begin{align}\label{eq:mb}
\onenormSize{}{\hat f-f}\lesssim t^{\alpha/(\alpha+2)}+{1\over H}+tH.
\end{align} 
It follows from the definition of $\hat f$ that
\begin{align}\label{eq:pfmain}
\onenormSize{}{\hat f-f}&=\mathbb{E}\left|{1\over 2H+1} \sum_{\pi\in \tH}\sign \hat \phi_\pi - f\right|\notag\\
&\leq \mathbb{E}\left|{1\over 2H+1}\sum_{\pi \in \tH} (\sign \hat \phi_\pi - \sign (f-\pi)) \right| +\mathbb{E}\left|{1\over 2H+1}\sum_{\pi \in \tH}\sign(f-\pi)-f\right|\notag\\
& \leq{1\over 2H+1}\sum_{\pi \in \tH}\onenormSize{}{\sign \hat \phi_\pi - \sign (f-\pi)}+{1\over H},
\end{align}
where the last line comes from the triangle inequality and the inequality
\[
\left|{1\over 2H+1}\sum_{\pi \in \tH}\sign(f(\mX)-\pi)-f(\mX)\right|\leq {1\over H}, \quad \text{for all } \mX\in\tX.
\]
It suffices to bound the first term in~\eqref{eq:pfmain}. 

Theorem~\ref{thm:unified} shows that the sign function accuracy depends on the closeness of $\pi\in \tH$ to the mass points in $\tH$. Therefore, we partition the level set $\pi \in \tH$ based on their closeness to $\tH$. Specifically, let $\tN_H \stackrel{\text{def}}{=}\bigcup_{\pi'\in\tN}\left(\pi'-\frac{1}{H},\pi'+\frac{1}{H}\right)$ denote the set of levels at least $1\over H$-close to the mass points. We expand left hand side of~\eqref{eq:pfmain} by
\begin{align}\label{eq:twobounds}
&{1\over 2H+1}\sum_{\pi \in \tH}\onenormSize{}{\sign \hat \phi_\pi - \sign (f-\pi)} \notag \\
=& \ {1\over 2H+1}\sum_{\pi \in \tH \cap \tN_H} \onenormSize{}{\sign \hat \phi_\pi - \sign (f-\pi)} +{1\over 2H+1}\sum_{\pi \in \tH \cap \tN_H^c} \onenormSize{}{\sign \hat \phi_\pi - \sign (f-\pi)} .
\end{align}
By assumption, the first term of~\eqref{eq:twobounds} involves only finite number of summands and thus can be bounded by $4C/(2H + 1)$ where $C > 0$ is a constant such that $|\tN|\leq C$. We bound the second term using the explicit forms of $\rho(\pi, \tN)$ in the sequence $\pi \in\Pi\cap \tN_H^c$. 
\begin{align}\label{eq:Hset}
{1\over 2H+1}\sum_{\pi \in \tH\cap \tN_H^c} \onenormSize{}{\sign \hat \phi_\pi- \sign (f-\pi)} &\lesssim  {1\over 2H+1}\sum_{\pi\in \tH\cap \tN_H^c} t^{\alpha/(2+\alpha)}+{t\over 2H+1}\sum_{\pi \in \tH\cap \tN_H^c}{1\over \rho^2(\pi, \tN)}\\
&\leq t^{\alpha/(2+\alpha)}+{t\over 2H+1} \sum_{\pi \in \tH\cap \tN_H^c} \sum_{\pi' \in \tN}{1\over |\pi-\pi'|^2}\\
&\leq  t^{\alpha/(2+\alpha)}+{t\over 2H+1} \sum_{\pi'\in \tN} \sum_{\pi \in \tH\cap \tN_H^c}{1\over |\pi-\pi'|^2}\\
&\leq t^{\alpha/(2+\alpha)}+ 2CHt,
\end{align}
where the first inequality uses the property of event $A$, and the last inequality follows from Lemma~\ref{lem:H}. Combining \eqref{eq:pfmain}, \eqref{eq:twobounds} and~\eqref{eq:Hset} comletes the proof of~\eqref{eq:mb}; that is
\begin{equation}\label{eq:A}
\mathbb{P}\left(\onenormSize{}{\hat f-f}\lesssim t^{\alpha/(\alpha+2)}+{1\over H}+tH\right)\geq\mathbb{P}(A), \quad \text{for all }t\geq t_n.
\end{equation}
Based on Remark~\ref{rmk:lt} and union bound over $\pi\in\tH$, we have,
\begin{align}\label{eq:prob}
\mathbb{P	}(A)&\geq 1-\sum_{\pi\in\tH}\mathbb{P}\left(\onenormSize{}{\sign \hat\phi_\pi-\sign(f-\pi)}\leq t^{\alpha/(2+\alpha)}+{t\over \rho^2(\pi,\tN)} \text{ for a given } \pi
\right)\notag \\
&\gtrsim 1-(2H+1)\exp(-nt)\gtrsim 1-\exp(-nt+\log H).
\end{align}
We choose $t \asymp t_n\log H$ in~\eqref{eq:prob} so that $\log H$ is negligible compared to $nt$. It then follows from~\eqref{eq:A} and~\eqref{eq:prob} that
\begin{align*}
\onenormSize{}{\hat f-f}\lesssim (t_n\log H)^{\alpha/(\alpha+2)}+{1\over H}+t_nH\log H,
\end{align*}
with probability at least $1-\exp(-nt_n\log H)\geq 1-\exp(-nt_n)$. Setting $H\asymp t^{-1/2}_n$ yields the desired conclusion. 

Proofs of Theorem~\ref{thm:sparse} and Part (b) in Theorem~\ref{thm:extension} follow the same argument with $t_n$ specified in Theorem~\ref{thm:unified}.
 \end{proof}

\begin{lem}\label{lem:H}
Fix $\pi'\in\tN$ and a sequence $\tH=\{-1,\ldots,-1/H,0,1/H,\ldots,1\}$ with $H\geq 2$. Then, 
\[
\sum_{\pi \in \tH\cap \tN_H^c}{1\over 
|\pi-\pi'|^2}\leq 4H^2. 
\]
\end{lem}
\begin{proof}[Proof of Lemma~\ref{lem:H}]
Notice that all points $\pi\in\tH\cap\tN_H^c$ satisfy $|\pi-\pi'|>{1\over H}$ for all $\pi'\in\tN$. We use this fact to compute the sum
\begin{align}
   \sum_{\pi \in \Pi\cap \tN_H^c}{1\over |\pi-\pi'|^2}&= \sum_{\frac{h}{H}\in\tH\cap \tN_H^c } {1\over |\frac{h}{H}-\pi'|^2}\\
   &\leq 2H^2\sum_{h=1}^{H}{1 \over h^2}\\
 &\leq 2H^2\left\{ 1+\int_{1}^2{1\over x^2}dx+ \int_{2}^3{1\over x^2}dx+\cdots + \int_{H-1}^H{1\over x^2}dx\right\}\\
&= 2H^2\left(1+\int^{H}_{1}{1\over x^2}dx\right) \leq 4H^2,
\end{align}
 where the third line uses the monotonicity of ${1\over x^2}$ for $x\geq 1$. 
 \end{proof}

\subsection{Proofs of Theorem~\ref{thm:estimation} and Theorem~\ref{thm:extension_gaussian}}\label{sec:sub-Gaussianproof}
\begin{proof}[Proof of Theorem~\ref{thm:estimation}]
Theorem~\ref{thm:estimation} follows from the same line of proof as in Theorem~\ref{thm:regression}, with slight modification to account for discrete measure space. For any matrix $\mZ\in\mathbb{R}^{d\times d}$, we use $f_{\mZ}\colon[d]^2 \to \mathbb{R}$ to denote the function induced by matrix $\mZ$ such that $f_{\mZ}(\omega)=\mZ(\omega)$ for $\omega\in[d]^2$. Set $\tX=\{\me^T_i\me_j\colon(i,j)\in[d]^2\}$ be the discrete feature space, and $n=|\Omega|$ the sample size. Under this set up, $\onenormSize{}{\hat f - f}=\mathbb{E}_{\mX}|\hat f (\mX)- f(\mX)| = \mathbb{E}_{\omega}|\hat \mTheta(\omega)-\mTheta(\omega)|=\textup{MAE}(\hat \mTheta-\mTheta)$. Notice that the small tolerance $\Delta s=1/d^2$ in the pseudo density is dominated by the derived convergence rate. Applying Theorem~\ref{thm:regression} to this setting finishes the proof. 
\end{proof}

\begin{proof}[Proof of Theorem~\ref{thm:extension_gaussian}]
By setting $s=\log(d_{\max})$ in Lemma~\ref{lem:subg}, we have
\[
\mathbb{P}(\mnormSize{}{\mE}\geq \sqrt{4\sigma^2\log d} )\leq 2d^{-2}.
\]
We divide the sample space into two exclusive events:
\begin{itemize}
\item Event I: $\mnormSize{}{\mE}\geq \sqrt{4\sigma^2\log d}$;
\item Event II: $\mnormSize{}{\mE}< \sqrt{4\sigma^2\log d}$.
\end{itemize}
Because the Event I occurs with probability tending to zero, we restrict ourselves to the Event II only by following the proof of Theorem~\ref{thm:main}. We summarize the key difference compared to Section~\ref{thm:regression}. For ease of notation,  define $\bar \mY = \mY-\pi$ and $\bar\mTheta = \mTheta-\pi$. Let $\ell_\omega(\cdot, \cdot)$ denote the 0-1 loss evaluated at the $\omega$-th value of two matrices. We expand the variance by 
\begin{align}
    \label{eq:variance2}
    \text{Var}\left[\ell_\omega\left(\mZ,\bar \mY_\Omega\right)-\ell_\omega\left(\bar\mTheta,\bar\mY_\Omega\right)\right]&\leq \mathbb{E}|\ell_\omega(\mZ(\omega),\bar\mY(\omega))-\ell_\omega(\bar\mTheta(\omega),\bar\mY(\omega))|^2\notag \\
    &= \mathbb{E}|\bar \mY(\omega)-\bar \mTheta(\omega)+\bar\mTheta(\omega)|^2|\text{sgn}\mZ(\omega)-\text{sgn}\bar\mTheta(\omega)| \notag \\
    &\leq 2\left(4 \sigma^2\log d+2\right) \mathbb{E}|\text{sgn}\mZ-\text{sgn}\bar\mTheta| \notag \\
    & \lesssim (\sigma^2 \log d) \text{MAE}(\sign \mZ, \sign \bar \mTheta),
    \end{align}
where the third line uses the facts $\mnormSize{}{\bar \mTheta}\leq 2$ and $\mnormSize{}{\bar \mY-\bar \mTheta}^2=\mnormSize{}{\mE}^2<4 \sigma^2\log d$ within the Event II; the last line comes from the definition of MAE and the asymptotic $\sigma^2\log d\gg 1$ provided that $\sigma>0$ with $d$ sufficiently large. 

Based on \eqref{eq:variance2}, the $(\alpha,\pi)$-smoothness of $\mTheta$ implies that for all measurable functions $f_{\mZ}$, we have
\begin{align}\label{eq:vartomean}
&\text{Var}\left[\ell_\omega\left(\mZ,\bar \mY_\Omega\right)-\ell_\omega\left(\bar\mTheta,\bar\mY_\Omega\right)\right]\\&\lesssim \left(\sigma^2\log d\right) \left\{\left[\mathbb{E}\left[\ell_\omega\left(\mZ,\bar \mY_\Omega\right)-\ell_\omega\left(\bar\mTheta,\bar\mY_\Omega\right)\right]\right]^{\alpha\over1+\alpha}+\frac{1}{\rho}\mathbb{E}\left[\ell_\omega\left(\mZ,\bar \mY_\Omega\right)-\ell_\omega\left(\bar\mTheta,\bar\mY_\Omega\right)\right]\right\}.
\end{align}
The empirical process with variance-to-mean relationship \eqref{eq:vartomean} gives that
\begin{align}\label{eq:empriskbd}
\mathbb{P	}\left(\text{Risk}(\hat\mZ)-\text{Risk}(\bar\mTheta)\geq L_d\right)\lesssim \exp(-|\Omega|t_d),
\end{align}
where the convergence rate $L_d$ is obtained by the same way in the proof of Lemma~\ref{lem:metric} to make sure the conditions hold in Theorem~\ref{thm:refer},  
\begin{align}\label{eq:subgbd}
L_d\asymp t_d^{(\alpha+1)/(\alpha+2)}+\frac{1}{\rho}t_d,\quad\text{ with } t_d =  {r \sigma^2 d\log d  \over |\Omega|}.
\end{align}
Combining \eqref{eq:empriskbd} and \eqref{eq:subgbd}, we obtain that, with high probability, 
\begin{align}\label{eq:riskunbd}
   \text{Risk}(\hat\mZ)-\text{Risk}(\bar\mTheta)\lesssim \left( {r \sigma^2  d\log d \over |\Omega|}\right)^{(\alpha+1)/(\alpha+2)}+\frac{1}{\rho(\pi,\tN)} \left({r \sigma^2  d \log d\over |\Omega|} \right),
\end{align} 
where constants have been absorbed into the $\lesssim$ relationship. Therefore, combining \eqref{eq:riskunbd} and the proof of Theorem~\ref{thm:unified} completes the proof for sign matrix estimation error in~\eqref{eq:matrix_sign}. The signal estimation error follows the same proof of Theorem~\ref{thm:regression}.
\end{proof}

\clearpage
\section{Auxiliary lemmas}\label{sec:auxiliary}

\begin{lem}[Hinge loss and $L$-1 distance]\label{lem:hingeL1} Consider the same set-up as in Theorem~\ref{thm:extension}. Let $F(z)=(1-z)_{+}$ be the hinge loss. Then, the $L$-1 distance between $\phi$ and $\bayespif$ is bounded by their excess risk; i.e,
\begin{equation}\label{eq:L}
\onenormSize{}{\phi-\bayespif}
\leq
\left[\riskF(\phi)-\riskF(\bayespif)\right]^{\alpha\over 1+\alpha}+
 {1\over \rho(\pi, \tN)}\left[\riskF(\phi)-\riskF(\bayespif)\right],
\end{equation}
for all functions $\phi\colon \tX\to\mathbb{R}$.
\end{lem}

\begin{rmk}[Truncated hinge loss and $L$-1 distance]\label{rmk:truncate}With little modification in the proof, similar inequality also holds for $T$-truncated hinge loss $F(z)=\min(T,(1-z)_{+})$ with $T\geq 2$. Specifically, 
\[
\onenormSize{}{\phi^T-\bayespif}\leq \left[\riskF(\phi)-\riskF(\bayespif)\right]^{\alpha\over 1+\alpha}+
 {1\over \rho(\pi, \tN)}\left[\riskF(\phi)-\riskF(\bayespif)\right],
\]
where $\phi^T\colon\tX\to[-(T-1),\ (T-1)]$ is the truncated $\phi$ defined in~\eqref{eq:Tphi}.
\end{rmk}

\begin{proof}[Proof of Lemma~\ref{lem:hingeL1}] For ease of notation, we drop the random variable $\mX$ in the function expression, and simply use $\phi, \bayespif$, $f$, to represent the trace function, Bayes rule, and the regression function, respectively. The meaning should be clear given the contexts. 

We expand the excess risk using the definition of hinge loss,
\begin{align}\label{eq:function}
&\riskF(\phi)-\riskF(\bayespif)\notag\\
=&\  \mathbb{E}[|\shift|(1-\phi\sign\shift)_{+}]-\mathbb{E}[|\shift|(1-\bayespif\sign\shift)_{+}] \notag\\
= &\ \int_{\mX} (1-\phi)_{+} \int_{y>\pi}(y-\pi)dy d\mathbb{P}_{\mX}+\int_{\mX}(1+\phi)_{+}  \int_{y\leq \pi}(\pi-y)dy d\mathbb{P}_{\mX} \notag \\
&\ -\int_{\mX} (1-\bayespif)_{+} \int_{y>\pi}(y-\pi)dy d\mathbb{P}_{\mX}-\int_{\mX}(1+\bayespif)_{+}  \int_{y\leq \pi}(\pi-y)dy d\mathbb{P}_{\mX}.
\end{align}

In order to evaluate the integral, we divide the domain $\mX$ into four exclusive regions:
\begin{itemize}
\item Region I $= \{\mX\colon f<\pi \text{ and }\phi\geq -1\}$. In this region, $\bayespif=-1$, and the integrant in~\eqref{eq:function} reduces to
\begin{align}
\Phi_{\textup{I}}&:=\left[(1-\phi)_{+}-2\right]\mathbb{E}_{Y|\mX}(Y- \pi)\mathds{1}(Y> \pi)+(\phi+1)_{+}\mathbb{E}_{Y|\mX}(\pi-Y)\mathds{1}(Y\leq \pi)\\
&\geq -(\phi+1)\mathbb{E}_{Y|\mX}(Y-\pi)\mathds{1}(Y>\pi)-(\phi+1)\mathbb{E}_{Y|\mX}(Y-\pi)\mathds{1}(Y\leq \pi)\\
&=(\phi+1)(\pi-f)=|\phi-\bayespif||f-\pi|.
\end{align}
\item Region II $= \{\mX\colon f < \pi \text{ and }\phi<-1\}$. In this region, $\bayespif=-1$, and the integrant in~\eqref{eq:function} reduces to
\[
\Phi_{\textup{II}}:=-(\phi+1)\mathds{E}_{Y|\mX}(Y-\pi)\mathds{1}(Y>\pi)\geq-|\phi+1|(f-\pi) =|\phi-\bayespif||f-\pi|.
\]
\item Region III $=\{ \mX\colon f\geq \pi \text{ and }\phi\leq 1\}$. In this region, $\bayespif=1$, and the integrant in~\eqref{eq:function} reduces to
\begin{align}
\Phi_{\textup{III}}&:=(1-\phi)_{+}\mathbb{E}_{Y|\mX}(Y-\pi)\mathds{1}(Y>\pi)+\left[(1+\phi)_{+}-2\right]\mathbb{E}_{Y|\mX}(\pi-Y)\mathds{1}(Y\leq \pi)\\
&\geq (1-\phi)\mathbb{E}_{Y|\mX}(Y-\pi)\mathds{1}(Y>\pi)+(\phi-1)\mathbb{E}_{Y|\mX}(\pi-Y)\mathds{1}(Y\leq \pi)\\
&=(1-\phi)(f-\pi)=|\phi-\bayespif||f-\pi|.
\end{align}
\item Region IV $=\{\mX\colon f\geq \pi \text{ and }\phi> 1\}$. In this region, $\bayespif=1$, and the integrant in~\eqref{eq:function} reduces to
\[
\Phi_{\textup{IV}}:=(\phi-1)\mathds{E}_{Y|\mX}(\pi-Y)\mathds{1}(Y\leq \pi)\geq (\phi-1)(f-\pi) = |\phi-\bayespif||f-\pi|.
\]
\end{itemize}
Therefore, the integral is evaluated as
\begin{align}\label{eq:integral}
\riskF(\phi)-\riskF(\bayespif) &= \int_{\textup{I}}\Phi_{\textup{I}} d\mathbb{P}_{\mX}+\int_{\textup{II}}\Phi_{\textup{II}} d\mathbb{P}_{\mX}+\int_{\textup{III}}\Phi_{\textup{III}} d\mathbb{P}_{\mX}+\int_{\textup{IV}}\Phi_{\textup{IV}} d\mathbb{P}_{\mX}\notag\\
&\geq \mathbb{E}|\phi-\bayespif||f-\pi|.
\end{align}
Note that the function $|f-\pi|$ is $\alpha$-smooth. Using the same techniques as in Theorem~\ref{thm:identifiability} to the last line of~\eqref{eq:integral}, we conclude
\begin{equation}
\mathbb{E}|\phi-\bayespif|
\lesssim
\left[\riskF(f)-\riskF(\bayespif)\right]^{\alpha\over 1+\alpha}+
 {1\over \rho(\pi, \tN)}\left[\riskF(f)-\riskF(\bayespif)\right].
\end{equation}
\end{proof}

\begin{defn}[Bracketing number]\label{pro:inftynorm}
Consider a function set $\Phi$, and let $\varepsilon>0$. We call $\{(f^l_m,f^u_m)\}_{m=1}^M$ an $L_2$-metric, $\varepsilon$-bracketing function set of $\Phi$, if for every $f\in \Phi$, there exists an $m\in[M]$ such that 
\[
f^l_m(\mX)\leq f(\mX)\leq f^u_m(\mX),\quad \text{for all }\mX\in\mathbb{R}^{d\times d},
\]
and
\[
\vnormSize{}{f^l_m-f^u_m}\stackrel{\text{def}}{=}\sqrt{\mathbb{E}_{\mX}|f^l_m(\mX)-f^u_m(\mX)|^2} \leq \varepsilon, \ \text{for all } m=1,\ldots,M. 
\]
The bracketing number with $L_2$-metric, $\tH_{[\ ]}(\varepsilon,\ \Phi,\ \vnormSize{}{\cdot})$, is defined as the logarithm of the smallest cardinality of the $\varepsilon$-bracketing function set of $\Phi$.  
\end{defn}

\begin{lem}[Bracketing number for bounded functions in $\Phi(r,s_1,s_2)$ and $\Phi(r)$]\label{lem:entropy}
Let $\Phi(r,s_1,s_2)$ denote the  trace function family 
\begin{align*} 
\Phi(r,s_1,s_2)=\{\phi\colon \mX\mapsto \langle \mX, \mB \rangle +b \ \big| \rank(\mB)\leq r,  \supp(\mB)\leq (s_1,s_2), |b|\leq \FnormSize{}{\mB}+1\},
\end{align*}
We use $\FnormSize{}{\phi}\stackrel{\textup{def}}{=}\FnormSize{}{\mB}$ to denote the coefficient magnitude.
Assume, for simplicity, $\mathbb{P}\left(\FnormSize{}{\mX}\leq 1\right)=1$. For any given $k\geq 1$, let $\Phi^k=\{f\in \Phi(r,s_1,s_2)\colon \FnormSize{}{\phi}^2 \leq k\}$ denote the sub-class of functions with coefficient magnitudes bounded by $k$. Then, 
\[
\tH_{[\ ]}(\varepsilon,\ \Phi^k,\ \vnormSize{}{\cdot}) \lesssim r(s_1+s_2) \log {kd_{\max}\over \varepsilon }.
\]
Furthermore, when we consider $\Phi^k = \{\phi\in\Phi(r)\colon \FnormSize{}{\phi}^2\leq k \}$ with $(s_1,s_2) = (d_1,d_2)$, then
\[
\tH_{[\ ]}(\varepsilon,\ \Phi^k,\ \vnormSize{}{\cdot}) \lesssim rd_{\max} \log {k\over \varepsilon }.
\]
\end{lem}
\begin{proof}[Proof of Lemma~\ref{lem:entropy}]
For any given $k\geq 1$, define a matrix family
\begin{align}
\tB=\left\{
\begin{bmatrix}
\mB& 0\\
0 & b\\
 \end{bmatrix}
 \in\mathbb{R}^{(d_1+1) \times (d_2+1)} \colon \rank(\mB)\leq r,\ \supp(\mB)\leq (s_1,s_2),\ |b|\leq \sqrt{k}+1, \ \FnormSize{}{\mB}\leq  \sqrt{k}
 \right\}
 \end{align}
By definition of trace functions, there is an onto mapping from matrices in $\tB$ to functions in $\Phi^k$; i.e. 
\[
\Phi^{k} \subset \left\{\phi\colon\mX\mapsto\langle \mX,\mB\rangle+b\ \bigg| \ \begin{bmatrix}
\mB& 0\\
0 & b\\
 \end{bmatrix}
 \in \tB\right\}.
\]
Furthermore, every pair of functions $\phi_1=\langle\mX, \mB_1\rangle+b_1,\ \phi_2=\langle \mX,\mB_2\rangle+b_2\in\Phi^k$ satisfies the norm relationship
\[
\vnormSize{}{\phi_1-\phi_2}\leq \mnormSize{}{\phi_1-\phi_2} =\sup_{\FnormSize{}{\mX}\leq 1}|\langle \mX,\mB_1\rangle+b_1 -\langle \mX,\mB_2\rangle  -b_2|\leq \sqrt{\FnormSize{}{\mB_1-\mB_2}^2+|b_1-b_2|^2}.
\]
Based on~\citet[Theorem 9.23]{kosorok2007introduction}, the $L_2$-metric, $(2\varepsilon)$-bracketing number in $\Phi^k$ is bounded by
\[
\tH_{[\ ]}(2\varepsilon,\ \Phi^k,\ \vnormSize{}{\cdot}) \leq \tH \left(\varepsilon,\ \tB,\ \FnormSize{}{\cdot}\right)
\]
where $\tH$ denotes the log covering number for the (non-bracketing) set. Therefore, it suffices to bound $\tH(\varepsilon,\ \tB,\ \FnormSize{}{\cdot})$ where $\tB$ is included in a $(\sqrt{5k})$-ball by definition of $\tB$. Now fix two subsets $S_1,S_2\subset [d]$ with $|S_1|=s_1$ and $|S_2|=s_2$, where $|\cdot|$ denotes the cardinality of the sets. Let $\tB_{S_1,S_2}\subset \tB$ denote the subset of matrices satisfying $\mB(i,j)=0$ whenever $(i,j)\notin S_1\times S_2$. Based on~\citet[Lemma 3.1]{candes2011tight}, the log covering number for $\tB_{S_1,S_2}$ is
\begin{align}\label{eq:bracketsparse}
\tH \left(\varepsilon,\ \tB_{S_1,S_2},\ \Fnorm{\cdot}\right)\lesssim r(s_1+s_2+1)\log\left({k\over \varepsilon}\right).
\end{align}
In view of the construction $\tB\subset\bigcup\{\tB_{S_1,S_2}\colon S_1\times S_2\subset [d_1]\times[d_2], |S_1|=s_1, |S_2|=s_2\}$, an $\varepsilon$-covering set $\tB$ is then given by the union of $\varepsilon$-covering set of $\tB_{S_1,S_2}$. Using Stirling's bound, we derive that 
\begin{align}
\tH(\varepsilon,\ \tB,\ \FnormSize{}{\cdot})&\leq \log \left\{{d_{\max} \choose s_1}{d_{\max} \choose s_2}\exp\left[\tH(\varepsilon,\tB_{S_1,S_2},\FnormSize{}{\cdot})\right]\right\}
\\
&\leq s_1 \log {d_{\max}\over s_1}+s_2\log {d_{\max}\over s_2}+C'r(s_1+s_2+1)\log{k\over \varepsilon}\\
& \leq Cr(s_1+s_2)\log{kd_{\max}\over \varepsilon},
\end{align}
where $C,C'>0$ are constants. 

The result for the case of $(s_1,s_2) = (d_1,d_2)$ directly follows from \eqref{eq:bracketsparse}.
\end{proof}

\begin{lem}[Local complexity of $\Phi(r,s_1,s_2)$ and $\Phi(r)$] \label{lem:metric}
Define $\Phi^{k}=\{f\in\Phi(r,s_1,s_2)\colon \FnormSize{}{f}^2 \leq k\}$ for all $k\geq 1$; i.e., $\Phi^k$ is the subset of functions in $\Phi(r,s_1,s_2)$ with coefficient magnitudes bounded by $k$. Set 
\begin{equation}\label{eq:specification}
L_n\gtrsim \left({r(s_1+s_2)\log d_{\max} \over n } \right)^{\alpha+1\over \alpha+2} + {1\over \rho (\pi, \tN)}\left({r(s_1+s_2)\log d_{\max} \over n } \right),\quad \text{and}\ 
\lambda={L_n\over 2 J^2}.
\end{equation}
Then, the following inequality is satisfied for all $k\geq 1$ and $s\geq 1$.
\begin{equation}
{1\over x}\int^{\sqrt{x^{\alpha/(\alpha+1)}+{x\over \rho (\pi, \tN)}}}_{x} \sqrt{\tH_{[\ ]}(\varepsilon,\ \Phi^{k},\ \vnormSize{}{\cdot}) }d\varepsilon \lesssim n^{1/2}, \quad \text{where}\ x:=sL_n+\lambda (k-2)J^2.
\end{equation}
The result for $\Phi(r)$ is the same  except $\log d_{\max}$ being removed from $L_n$ in~\eqref{eq:specification}.
\end{lem}

\begin{proof}[Proof of Lemma~\ref{lem:metric}]
To simplify the notation, we write $\rho=\rho(\pi, \tN)$,  $d = d_{\max}$, and define
\begin{equation}\label{eq:complexity}
g(x, k)={1\over x}\int^{\sqrt{x^{\alpha/(\alpha+1)}+x/\rho}}_{x} \sqrt{r(s_1+s_2)\log\left({kd\over \varepsilon}\right)}d\varepsilon,\quad \text{for all }k\geq 1,
\end{equation}
where we have inserted the bracketing number based on Lemma~\ref{lem:entropy}.  Notice that
\begin{align}\label{eq:g}
g(x,k)&\leq {\sqrt{r(s_1+s_2)}\over L}\int_{x}^{\sqrt{x^{\alpha/(\alpha+1)}+x/\rho}}\sqrt{\log \left(kd \over x \right)}d\varepsilon\notag \\
&\leq \sqrt{r(s_1+s_2)(\log k+\log d - \log x)}\left({\sqrt{x^{\alpha/(2\alpha+2)}}+\sqrt{x/\rho} \over x }-1\right)\notag \\
&\leq  \sqrt{r(s_1+s_2)(\log k+\log d)}\left( {1\over x^{(\alpha+2)/(2\alpha+2)}}+{1\over \sqrt{\rho x}}\right) =: \bar g(x,k),
\end{align}
where the second line follows from $\sqrt{a+b} \leq \sqrt{a}+\sqrt{b}$ for $a,b>0$.  Since the upper bound $\bar g(x,k)$ is decreasing function with respect to $x>0$, it suffices to show that $\bar g( x, k) \leq n^{1/2}$ for all $k\geq 1$ and $s=1$; that is, to show $\bar g (\bar x, k)\lesssim n^{1/2}$ for all $k\geq 1$ under the choice
\[
\bar x := L_n+\lambda (k-2)J^2 \geq {k \over 2} \left\{\left({r(s_1+s_2)\log d \over n } \right)^{\alpha+1\over \alpha+2} + {1\over \rho }\left({r(s_1+s_2)\log d \over n } \right)\right\}.
\]
Plugging the above expression into the last line of~\eqref{eq:g} gives
\[
\bar g(\bar x,k)\leq n^{1/2}\sqrt{\log k+\log d\over (k/2)^{(\alpha+2)/(\alpha+1)} \log d}+n^{1/2}\sqrt{\log k+\log d \over (k/2) \log d }\leq C'n^{1/2},\quad \text{for all }k\geq 1,
\]
where $C'>0$ is a constant independent of $k$ and $d$. The proof is therefore complete. 
\end{proof}

\begin{lem}[sub-Gaussian maximum]\label{lem:subg}
Let $X_1,\ldots,X_n$ be independent sub-Gaussian zero-mean random variables with variance proxy $\sigma^2$. Then, for any $s>0,$
\[\mathbb{P}\left\{\max_{1\leq i\leq n}|X_i|\geq\sqrt{2\sigma^2(\log n +s)}\right\}\leq2 e^{-s}.\]
\end{lem}
\begin{proof}[Proof of Lemma~\ref{lem:subg}]
The conclusion follows from
\begin{align}
\mathbb{P}[\max_{1\leq i\leq n}|X_i|\geq u] \leq \sum_{i=1}^n\mathbb{P}[|X_i|\geq u]\leq 2n e^{-{u^2\over 2\sigma^2}} = 2e^{-s},
\end{align}
where we set $u = \sqrt{2\sigma^2(\log n+s)}.$
\end{proof}

We state the results from \citet[Theorem 1]{scott2011surrogate} in our contexts. 
\begin{thm}[Theorem 1 in~\cite{scott2011surrogate}]\label{thm:scott} Let $\riskF(\cdot)$ be weighted $F$-risk defined in Section~\ref{sec:large-margin} of the main paper with $\pi\in[-1,1]$. Define the conditional risk 
\[
C_{\pi,F}(\mX,t):=F(t)\mathbb{E}_{Y|\mX}(Y-\pi)^{+}+F(-t)\mathbb{E}_{Y|\mX}(Y-\pi)^{-},
\]
and associated function $H_{\pi,F}$:
\[
H_{\pi,F}(\mX)=\inf_{t\in\mathbb{R}\colon t(f(\mX)-\pi)\leq 0}C_{\pi,F}(\mX, t)-\inf_{t\in\mathbb{R}} C_{\pi,F}(\mX, t).
\]
Let $f(\mX)=\mathbb{E}(Y|\mX)$. For any $\varepsilon\geq 0$, define
\[
g(\varepsilon)=\begin{cases}
\inf_{\mX\in\tX:|f(\mX)-\pi|\geq \varepsilon} H_{\pi,F}(\mX),& \varepsilon>0,\\
0& \varepsilon=0.
\end{cases}
\]
Now set $\psi=g^{**}$ where $g^{**}$ denotes the Fenchel-Legendre biconjugate of $g$. Then, for any decision function $\phi\colon \tX\to \mathbb{R}$ and any distribution of $(\mX,Y)$, we have
\[
\psi\left( \risk(\phi)-\inf_{\text{all }\phi}\risk(\phi)\right)\leq \riskF(\phi)-\inf_{\text{all }\psi}\riskF(\phi).
\]
\end{thm}
\begin{thm}[Theorem 3 in~\cite{shen1994convergence}]~\label{thm:refer}Let $\tF$ be a class of functions defined on $\tX$ with $\sup_{f\in\tF}\norm{f}_{\infty}\leq T$. Let $(\mX_i)_{i=1}^n$ be i.i.d.\ random variables with distribution $\mathbb{P}_{\mX}$ over $\tX$. Set $\sup_{f\in\tF}\textup{Var}f(\mX)=V<\infty$. 
Define the empirical process $\mathbb{\hat E}f={1\over n}\sum_{i=1}^n f(\mX_i)$. 
Define $x_n^*$ to be the solution to the following inequality
\[
{1\over x}\int_x^{\sqrt{V}}\sqrt{\tH_{[\ ]}(\varepsilon,\tF,\vnormSize{}{\cdot})}d\varepsilon \lesssim \sqrt{n}.
\]
Suppose $\sqrt{V}\leq T$ and 
\[
x_n^*\lesssim {V\over T},\quad \text{and}\quad \tH_{[\ ]}(\sqrt{V},\tF,\vnormSize{}{\cdot})\lesssim {n(x_n^*)^2 \over V}.
\]
Then, we have
\begin{equation}\label{eq:oneside}
\mathbb{P}\left(\sup_{f\in\tF}\mathbb{\hat E}f -\mathbb{E}f\geq x^*_n\right)\lesssim  \exp\left(-{n (x^*_n)^2\over V+Tx^*_n}\right). 
\end{equation}
\end{thm}

\end{document}